\documentclass{article}

\usepackage{rotating}
\usepackage{footnote}
\usepackage{tablefootnote}
\makesavenoteenv{tabular}
\makesavenoteenv{table}
\usepackage{multirow}
\usepackage{makecell} 
\usepackage[final]{neurips_2021}

\usepackage[utf8]{inputenc} % allow utf-8 input
\usepackage[T1]{fontenc}    % use 8-bit T1 fonts
\usepackage{hyperref}       % hyperlinks
\usepackage{url}            % simple URL typesetting
\usepackage{booktabs}       % professional-quality tables
\usepackage{amsfonts}       % blackboard math symbols
\usepackage{nicefrac}       % compact symbols for 1/2, etc.
\usepackage{microtype}      % microtypography
\usepackage{xcolor}         % colors

\usepackage{amssymb,amsmath,amsfonts,bm,epsfig,graphicx,theorem,epstopdf,bbm}
\usepackage{setspace,latexsym,epsf,color,dsfont,caption,mathtools}
\usepackage{subfigure,tabularx}

\usepackage{algorithm}
\usepackage{algpseudocode}

\usepackage{enumitem}

\DeclareMathOperator{\tr}{trace}
\definecolor{cared}{rgb}{1.0, 0.03, 0.0}

\newcommand{\trace}{\text{trace}}

\newcommand{\diag}{\text{diag}}

\newcommand{\Rb}{\mathbb{R}}
\newcommand{\Eb}{\mathbb{E}}
\newcommand{\Pb}{\mathbb{P}}

\newcommand{\lv}{\mathbbm{1}}

\DeclareMathOperator{\rank}{rank}

\DeclareMathOperator{\adj}{adj}

\newcommand{\myparagraph}[1]{{\noindent \bf #1}\hspace{0.0in}}

\newcommand{\Ac}{\mathcal{A}}
\newcommand{\Tc}{\mathcal{T}}
\newcommand{\Nc}{\mathcal{N}}
\newcommand{\Sc}{\mathcal{S}}
\newcommand{\Cc}{\mathcal{C}}
\newcommand{\Xc}{\mathcal{X}}

\newcommand{\Rc}{\mathcal{R}}
\newcommand{\Fc}{\mathcal{F}}

\newcommand{\Bc}{\mathcal{B}}
\newcommand{\Ec}{\mathcal{E}}
\newcommand{\Hc}{\mathcal{H}}

\newcommand{\Mc}{\mathcal{M}}

\newcommand{\Gc}{\mathcal{G}}

\newcommand{\Rd}{{\mathbb{R}^d}}

\newcommand{\TT}{\intercal}

\DeclareMathOperator{\Det}{Det}
\DeclareMathOperator{\range}{range}
\DeclareMathOperator{\Span}{span}

\newcommand{\alg}{\texttt{Fed-PE} }
\newcommand{\bca}{\texttt{BCA} }

\newtheorem{theorem}{Theorem}
\newtheorem{lemma}{Lemma}

\newtheorem{Proposition}{Proposition}
\newtheorem{assump}{Assumption}
\newtheorem{remark}{Remark}
\newtheorem{definition}{Definition}
\newtheorem{corollary}{Corollary}

\newenvironment{proof}[1][Proof]{{\it #1. } }{\ \rule{0.5em}{0.5em}}

\title{Federated Linear Contextual Bandits}

\author{
   Ruiquan Huang \\
   The Pennsylvania State University \\
   \texttt{rzh5514@psu.edu} \\
   \And
 Weiqiang Wu\\
  Facebook\\
  \texttt{weiqiang.wwu@gmail.com} \\
   \And
   Jing Yang \\
   The Pennsylvania State University \\
   \texttt{yangjing@psu.edu} \\
   \And
  Cong Shen \\
  University of Virginia \\
   \texttt{cong@virginia.edu} \\
}

\begin{document}

\maketitle

\begin{abstract}
This paper presents a novel federated linear contextual bandits model, where individual clients face different $K$-armed stochastic bandits coupled through common global parameters. By leveraging the geometric structure of the linear rewards, a collaborative algorithm called \alg is proposed to cope with the heterogeneity across clients without exchanging local feature vectors or raw data. \alg relies on a novel multi-client G-optimal design, and achieves near-optimal regrets for both disjoint and shared parameter cases with logarithmic communication costs. In addition, a new concept called collinearly-dependent policies is introduced, based on which a tight minimax regret lower bound for the disjoint parameter case is derived. Experiments demonstrate the effectiveness of the proposed algorithms on both synthetic and real-world datasets.

\vspace{-0.08in}
\end{abstract}

\section{Introduction}

Federated learning (FL) \citep{mcmahan2017communicationefficient} is an emerging distributed machine learning (ML) paradigm where massive number of clients collaboratively learn a shared prediction model while keeping all the training data on local devices. Compared with standard centralized machine learning, FL has the following characteristics~\citep{kairouz2021advances}:

\vspace{-0.03in}
\begin{itemize}[leftmargin=*]\itemsep=0pt
\setlength{\leftmargin}{-0.5in}
    \item \textbf{Heterogeneous local datasets.} The local datasets, which are often generated at edge devices, are likely drawn from non-independent and identically distributed ({non-}IID) distributions.
\item \textbf{Communication efficiency.} The communication cost scales with the number of clients, which is one of the primary bottlenecks of FL. It is critical to minimize the communication cost while maintaining the learning accuracy.
\item \textbf{Privacy.} FL protects local data privacy by only sharing model updates instead of the raw data.
\end{itemize}
\vspace{-0.03in}

While the main focus of the state-of-the-art FL is on the supervised learning setting, recently, a few researchers begin to extend FL to the multi-armed bandits (MAB) framework~\citep{lai1985asymptotically,auer2002finite,bubeck2012regret,agrawal2012analysis,agrawal2013further}. In the canonical setting of MAB, a player chooses to play one arm from a set of arms at each time slot. An arm, if played, will offer a reward that is drawn from its distribution which is unknown to the player. With all previous observations, the player needs to decide which arm to pull each time in order to maximize the cumulative reward. MAB thus represents an online learning model that naturally captures the intrinsic exploration-exploitation tradeoff in many sequential decision-making problems.

Extending FL to the MAB framework is naturally motivated by a corpus of applications, such as recommender systems, clinical trials, and cognitive radio. In those applications, the sequential decision making involves multiple clients and is distributed by nature. 
While classical MAB models assume immediate access to the sequentially generated data at the learning agent, under the new realm of FL, local datasets can be stored and analyzed at the clients, thus reducing the communication load and potentially protecting the data privacy. 

Despite the potential benefits of FL, the sequential decision making and bandit feedback bring new challenges to the design of FL algorithms in the MAB setting. Different from the supervised learning setting where static datasets are collected beforehand, under the MAB setting, data is generated sequentially as decisions are made, actions are taken, and observations are collected. In order to maximize the cumulative reward and minimize the corresponding learning regret, it thus requires sophisticated coordination of the actions of the clients.
The heterogeneous reward distributions across clients make the coordination process even more convoluted and challenging. Besides, the data privacy and communication efficiency requirements result in significant challenges for efficient information exchange and aggregation between local clients and the central server.

In this work, we attempt to address those challenges in a federated linear contextual bandits framework. This particular problem is motivated by the following exemplary applications.

\begin{itemize}[leftmargin=*]\itemsep=0pt
\setlength{\leftmargin}{-0.5in}
    \item \textbf{Personalized content recommendation.} For content (arm) recommendation in web-services, user engagement (reward) depends on the profile of a user (context). The central server may deploy a recommender system on each user' local device (client) in order to personalize recommendations without knowing the personal profile or behavior of the user.

    \item \textbf{Personalized online education.} In order to maximize students performances (reward) in online learning, the education platform (central server) needs to personalize teaching methods (arms) based on the characteristics of individual students (context). With the online learning software installed at local devices (client), it is desirable to personalize the learning experiences without allowing the platform to access students' characteristics or scores.
\end{itemize}
In those examples, the reward of pulling the same arm at different clients follows different distributions dependent on the context as in contextual bandits~\citep{Auer:2003:UCB,Langford:2008}. We note that conventional contextual bandits is defined with respect to a single player, where the time-varying context can be interpreted as different incoming user profiles. In contrast, we consider a multi-client model, where each client is associated with a fixed user profile. The variation of contexts is captured over clients as opposed to over time.  Although the set of clients remains fixed through the learning process, the reward of pulling the same arm still varies across clients.
Such a model naturally takes data heterogeneity into consideration. Besides, we adopt a linear reward model, which has been widely studied in contextual bandits~\citep{Li:2010:LinUCB,Agrawal:2013:linear}.
 
\myparagraph{Main contributions.}  
Our main contributions are summarized as follows. 

First, we propose a new federated linear contextual bandits model that takes the diverse user preferences and data heterogeneity into consideration. Such a model naturally bridges local stochastic bandits with linear contextual bandits, and is well poised to capture the tradeoffs between communication efficiency and learning performances in the federated bandits setting. 

Second, we design a novel algorithm named \alg and further develop its variants to solve the federated linear contextual bandits problem. Under \texttt{Fed-PE}, clients only upload their local estimates of the global parameters without sharing their local feature vectors or raw observations. It not only keeps the personal information private, but also reduces the upload cost. We explicitly show that \alg and its variants achieve {near-optimal} regret performances for both disjoint and shared parameter cases with logarithmic communication costs.

Third, we generalize the G-optimal design from the single-player setting \citep{LS19bandit-book} to the multi-client setting. We develop a block coordinate ascent algorithm to solve the generalized G-optimal design efficiently with convergence guarantees. Such a multi-client G-optimal design plays a vital role in \texttt{Fed-PE}, and may find broad applications in related multi-agent setups.

Finally, we introduce a novel concept called \emph{collinearly-dependent policy} and show that the celebrated LinUCB type of policies~\citep{Li:2010:LinUCB}, Thompson sampling based policies with Gaussian priors~\citep{Agrawal:2013:linear}, and least squared estimation based policies, such as \texttt{Fed-PE}, are all in this category. By utilizing the property of  collinearly-dependent policies, we are able to characterize a tight minimax regret lower bound in the disjoint parameter setting. We believe that this concept may be of independent interest for the study of bandits with linear rewards.

\myparagraph{Notations.} Throughout this paper, we use  $\|x\|_V$ to denote $\sqrt{x^\TT V x}$. The \textit{range} of a matrix $A$, denoted by $\range(A)$, is the subspace spanned by the column vectors of $A$. We use $A^\dagger$ and {$\Det(A)$} to denote the pseudo-inverse and pseudo-determinant of square matrix $A$, respectively. The specific definitions can be found in Appendix~\ref{appx:prelim}. 

\section{Related Works}
\myparagraph{Collaborative and distributed bandits.}
Our model is closely related to the collaborative and distributed bandits when action collision is not considered. \citet{landgren2016distributed,landgren2018social} and \citet{martinez2019decentralized} study distributed bandits in which multiple agents face the same MAB instance, and the agents collaboratively share their estimates over a fixed communication graph in order to design consensus-based distributed estimation algorithms to estimate the mean of rewards at each arm. \citet{szorenyi13} considers a similar setup where in each round an agent is able to communicate with a few random peers.  \citet{pmlr-v48-korda16-distributed-cluster} considers the case where clients in different unknown clusters face independent bandit problems, and every agent can communicate with only one other agent per round. The communication and coordination among the clients in those works are fundamentally different from our work. 

\citet{wang2019distributed} investigates communication-efficient distributed linear bandits, where the agents can communicate with a server by sending and receiving packets. It proposes two algorithms, namely, DELB and DisLinUCB, for fixed and time-varying action sets, respectively. The fixed action set setting is similar to our setup, except that it assumes that all agents face the same bandits model, which does not take data heterogeneity into consideration.

\myparagraph{Federated bandits.} A few recent works have touched upon the concept of federated bandits. With heterogeneous reward distributions at local clients, \cite{shi2021aaai} and \cite{shi2021aistats} investigate efficient client-server communication and coordination protocols for federated MAB without and with personalization, respectively. \cite{agarwal2020federated} studies regression-based contextual bandits as an example of the federated residual learning framework, where the reward of a client depends on both a global model and a local model. \citet{li2020federated2} and \cite{Zhu_2021} focus on differential privacy based local data privacy protection in federated bandits. While the linear contextual bandit model considered in \citet{dubey2020differentiallyprivate} is similar to this work, it focuses on federated differential privacy and proposes a LinUCB-based FedUCB algorithm, which incurs a higher regret compared with our result for the shared parameter case. A regret and communication cost comparison between \alg and other baseline algorithms is provided in Table~\ref{table:compare}. 

\begin{table}
\small
  \caption{Performance comparison}
  \label{table:compare}
  \centering
  \resizebox{\textwidth}{!}{
  \begin{tabular}{cccc} 
    \Xhline{3\arrayrulewidth}
Model & Algorithm   & Regret & Communication cost \\
    
    \Xhline{2\arrayrulewidth}

Linear &  DELB   & $O(d\sqrt{MT\log(T)})$ & $O((Md+d\log\log d)\log T)$      \\
    \midrule
\multirow{3}{*}{\begin{tabular}{c} Linear contextual\\ (shared parameter) \end{tabular} }   &  FedUCB\tablefootnote{Results adapted from \citet{dubey2020differentiallyprivate} by letting the privacy budget $1/\epsilon$ go to 0.} &$O\left(\sqrt{dMT}\log T \right)$ & $O(Md^2\log T)$  \\
& { \alg (this work) } & $O(\sqrt{dMT\log(KMT)})$ & $O(M(d^2 + dK)\log T)$\\
 & Lower bound & $\Omega(\sqrt{dMT})$ & N/A\\
    \midrule
\multirow{3}{*}{\begin{tabular}{c} Linear contextual\\ (disjoint parameter) \end{tabular} }   & Centralized\tablefootnote{Results adapted from the single-player linear contextual bandits studied in \cite{dimakopoulou2018estimation} by assuming instantaneous information exchange and sequential decision-making at the central server. }    & $O(\sqrt{dKMT\log^3(KMT)})$ & $O(M{d^2}K T)$\\
 &\alg   (this work)    & $O(\sqrt{dKMT\log(KMT)})$ & $O(Md^2K\log T)$\\
 & Lower bound (this work) & $\Omega(\sqrt{dKMT})$ & N/A\\
    
    \Xhline{3\arrayrulewidth}
  \end{tabular}
  }
  
  \begin{center}
  \scriptsize
  $M$: number of clients; $K$: number of arms; $T$: time horizon; $d$: ambient dimension of the feature vectors.
  \end{center}
\end{table}

\section{Problem Formulation}\label{sec:formulation}

\myparagraph{Clients and local bandits model.} We consider a federated linear contextual bandits setting where there are $M$ clients pulling the same set of $K$ items (arms) denoted as $[K] := \{1,2,\ldots,K\}$. At each time $t$, each client $i\in[M]$ pulls an arm $a_{i,t} \in [K]$ based on locally available information. The incurred reward $y_{i,t}$ is given by
$ y_{i,t} = r_{i,a_{i,t}} + \eta_{i,t}$,
%\end{align}
where $\eta_{i,t}$ is a random noise, and $r_{i,a_{i,t}}$ is the unknown expected reward by pulling arm $a_{i,t}$. 
We note that without additional assumptions or interaction among the clients, each local model is a standard single-player stochastic MAB, where classic algorithms such as UCB~\citep{Auer:2010} and Thompson sampling~\citep{agrawal2012analysis} are known to achieve order-optimal regret.

\myparagraph{Linear reward structure with global parameters.} In order to capture the inherent correlation between rewards of pulling the same arm by different clients, we assume $r_{i,a}$  has a linear structure, i.e., 
  $  r_{i,a}=x_{i,a}^\TT\theta_{a},$ 
where $x_{i,a}\in \Rb^d$ is the feature vector associated with client $i$ and arm $a$, and $\theta_{a} \in \Rd$ is a fixed but unknown {parameter} vector for each $a \in [K]$.
Here we use $x^\TT$ to denote the transpose of vector $x$. The same arm $a$ may have different reward distributions for different clients, due to potentially varying $x_{i,a}$ across clients. Such a linear model naturally captures the heterogeneous data distributions at the clients, yet admits possible collaborations among clients due to the common parameters $\{\theta_a\}_{a\in[K]}$. When $\theta_a$ varies for different arm $a$, it is called the {\it disjoint parameter} case; when $\theta_a$ is known to be a constant across the arms, it is the {\it shared parameter} case. We investigate both cases in Sections~\ref{sec:disjoint} and \ref{sec:shared}, respectively. 

\myparagraph{Communication model.} We assume there exists a central server in the system, and similar to FL, the clients can communicate with the server periodically with zero latency. Specifically, the clients can send ``local model updates'' to the central server, which then aggregates and broadcasts the updated ``global model'' to the clients. (We will specify these components later.) Note that just as in FL, communication is one of the major bottlenecks and the algorithm has to be conscious about its usage. Similar to \cite{wang2019distributed}, we define the communication cost of an algorithm as the {number of scalars (integers or real numbers) communicated between server and clients}. We also make the assumption that clients and server are fully synchronized \citep{mcmahan2017communicationefficient}.

\myparagraph{Data privacy concerns.} Similar to \cite{dubey2020differentiallyprivate}, our contextual bandit problem involves two sets of information that are desirable to be kept private to client $i$: the feature vectors $\{x_{i,a}\}_{a\in[K]}$ and the observed rewards $\{y_{i,t}\}_{t\in[T]}$. Different from the differential privacy mechanism adopted in \cite{dubey2020differentiallyprivate}, in this work, we aim to communicate estimated global model parameters $\{\theta_a\}_a$ between the clients and the server. This is consistent with the FL framework, where only model updates are communicated instead of the raw data. 

\begin{assump}\label{assump:bounded}
We make the following assumptions throughout the paper:
\begin{enumerate}[leftmargin=12pt,topsep=0pt, itemsep=0pt,parsep=0pt]
\item[1)] \textbf{Bounded parameters:} For any $i\in [M]$, $a\in[K]$, we have $\|\theta_a\|_2\leq s$, $0<\ell\leq \|x_{i,a}\|_2\leq L$.
\item[2)]  \textbf{{Independent} 1-subgaussian noise:} 
$\eta_{i,t}$ is a 1-subgaussian noise parameter sampled independently at each time for each client with $\Eb[\eta_{i,t}]=0$, $\Eb[ \exp( \lambda \eta_{i,t} ) ] \le \exp(\frac{\lambda^2}{2})$ for any $\lambda>0$.
\end{enumerate}
\end{assump} 
Assumption~1.1 is {a standard assumption in the bandit literature, which ensures that the maximum regret at any step is bounded}. We emphasize that our work does not make any assumption on the knowledge of suboptimality gaps, nor do we assume the existence of a unique optimal arm at each client.

Our objective is to minimize the expected cumulative regret among all clients, defined as:
\begin{align}
\mathbb{E}[R(T)] = \mathbb{E}\left[\sum_{i=1}^M\sum_{t=1}^T \left(x_{i,a_i^*}^\TT\theta_{a_i^*} - x_{i,a_{i,t}}^\TT\theta_{a_{i,t}}\right)\right],
\end{align}
where $a_i^*\in[K]$ is an optimal arm for client $i$: $\forall b \neq a_i^*$, $x_{i,a_i^*}^\TT\theta_{a_i^*}  -x_{i,b}^\TT\theta_{b}  \geq 0$.

\section{Federated Linear Contextual Bandits: Disjoint Parameter Case}\label{sec:disjoint}
\subsection{Challenges}\label{sec:challenge}
Solving the federated linear contextual bandits model faces several new challenges. The first challenge is due to the constraint that only locally estimated parameters $\{\theta_a\}_{a\in[K]}$ are uploaded to the central server. While this is not an issue for stochastic MAB where the $\{\theta_a\}_{a\in[K]}$ are scalars~\citep{shi2021aaai,shi2021aistats}, this brings significant challenges for the aggregation of the local estimates into a ``global model'' in our setup. This is because under the linear reward structure, the locally received rewards $\{y_{i,t}\}_{t\in[T]}$ only contain the projection of $\theta_a$ along the direction of $x_{i,a}$, while the portion of information lying outside $\range(x_{i,a})$ is not captured in $\{y_{i,t}\}_{t\in[T]}$. Thus, by utilizing $\{y_{i,t}\}_{t\in[T]}$, the locally estimated $\theta_a$, denoted as $\hat{\theta}_{i,a}$, cannot provide any information of $\theta_a$ beyond $\range(x_{i,a})$. Since $\{x_{i,a}\}_{i\in[M]}$ are different for the same arm $a$, the locally estimated $\{\hat{\theta}_{i,a}\}_{i\in[M]}$ are essentially lying in different subspaces. The central server thus needs to take such geometric structure into account when aggregating $\{\hat{\theta}_{i,a}\}$ to construct the global estimate of $\theta_{a}$.  

The geometric structure of the local rewards also brings another challenge for the coordination of actions of local clients. Intuitively, in order to help client $i$ accurately estimate the expected reward by pulling arm $a$, it suffices to obtain an accurate projection of $\theta_a$ on $\range(x_{i,a})$; any part of $\theta_a$ lying outside this subspace is irrelevant. Therefore, if two clients $i$ and $j$ have $x_{i,a}$ and $x_{j,a}$ orthogonal to each other, exchanging the local estimates $\hat{\theta}_{i,a}$ and $\hat{\theta}_{j,a}$ does not help the other client improve her own local estimation. On the other hand, if $x_{i,a}$ and $x_{j,a}$ are completely aligned with each other, $\hat{\theta}_{i,a}$ and $\hat{\theta}_{j,a}$ can be aggregated directly to improve the local estimation accuracy of both. With $M$ possible subspaces spanned by $\{x_{i,a}\}_i$, it is highly likely that different clients receive different amounts of {\it relevant} information (i.e., information lying in $\range(x_{i,a})$) through information exchange facilitated by the central server. Therefore, in order to reduce the overall regret, it is necessary to coordinate the actions of clients in a sophisticated fashion. 

Third, since the exact knowledge of the local feature vectors $\{x_{i,a}\}$ are kept from the central server, the server may not have an accurate estimation of the uncertainty level of the local estimates at each client, or how much the coordination would help individual clients. This would make efficient and effective coordination even more challenging.

\subsection{Federated Phased Elimination (\texttt{Fed-PE}) Algorithm}
To address the aforementioned challenges, we propose the Federated Phased Elimination (\texttt{Fed-PE}) algorithm. The \alg algorithm works in phases, where the length of phase $p$ is $f^p+K$. It contains a client side subroutine (Algorithm~\ref{alg:client}) and a server side subroutine (Algorithm~\ref{alg:server}). Throughout the paper we use superscript $p$ to indicate phase $p$ barring explicit explanation. We use $\Ac_{i}^p\subset [K]$ to denote the subset of active arms at client $i$ in phase $p$, $\Ac^p:=\cup_{i=1}^M \Ac_{i}^p$,  $\mathcal{R}_a^p := \{i: a\in \Ac_i^p\}$, and define $\Tc_{i,a}^p$ as the time indices at which client $i$ pulls arm $a$ during the collaborative exploration step in phase $p$.  Then, the algorithm works as follows.

\begin{algorithm}[H]
\small
\caption{\alg: client $i$ }
\begin{algorithmic}[1]
	\Require $T$, $M$, $K$, $\alpha$, $f^p$
		\State \textbf{Initialization:}  Pull each arm $a\in [K]$ and receive reward $y_{i,a}$; 
		$\hat{\theta}_{i,a}^0\gets \frac{y_{i,a}x_{i,a}}{\|x_{i,a}\|^2}$; Send $\{\hat{\theta}_{i,a}^0\}_{a}$ to the server; $\Ac_{i}^0\gets [K]$; $p\gets 1$.
			\While{not reaching the time horizon $T$} 
\State Receive $ \{(\hat{\theta}_a^p,V_a^p)\}_{a\in \Ac^{p-1}}$ from the server. 
\Comment{\texttt{Arm elimination}}
\For{$a\in\Ac_i^{p-1}$}
\begin{equation}\label{eqn: confidence_interval}
\hat{r}_{i,a}^p \leftarrow x_{i,a}^\TT\hat{\theta}_{a}^p,\qquad
u_{i,a}^p\leftarrow\alpha \left\|x_{i,a}\right\|_{V_a^p}/\ell.
\end{equation}
\EndFor
\State $\hat{a}_i^p \leftarrow \arg\max_{a\in \Ac_i^{p-1}} \hat{r}_{i,a}^p$,\quad  $\Ac_i^p \leftarrow \left\{a\in \Ac_i^{p-1}\ |\ \hat{r}_{i,a}^p + u_{i,{a}}^p\geq \hat{r}_{i,\hat{a}_i^p}^p -u_{i,\hat{a}_i^p}^p\right\}.$
\State Send $\Ac_i^p$ to the central server. \Comment{\texttt{Active arm set updating}}
\State Receive $f_{i,a}^p$ for all $a\in \Ac_i^p$.

\For{$a \in \Ac_i^p $} \Comment{\texttt{Collaborative exploration}}
\State Pull arm $a$ for {$f_{i,a}^p$} times and receive rewards $\{y_{i,t}\}_{t\in \Tc_{i,a}^p}$. 
\begin{align}\label{eqn:local_theta}
    \hat{\theta}_{i,a}^p\gets\bigg(\frac{1}{f_{i,a}^p}\sum_{t\in \Tc_{i,a}^p}y_{i,t}\bigg) \frac{x_{i,a}}{\|x_{i,a}\|^2}.
\end{align}

\EndFor
\State Send $\{\hat{\theta}_{i,a}^p\}_{a\in \Ac_i^p}$ to the server; Pull $\hat{a}_i^p$ until phase length equals $f^p+K.$ 
\State  $p\gets p+1$.
\EndWhile
\end{algorithmic}\label{alg:client}
\end{algorithm}
\vspace{-0.1in} 

At the initialization phase, each client $i$ pulls every arm $a \in [K]$ once and receives a reward $y_{i,a}$, based on which it obtains an estimate of the projection of $\theta_a$. These estimates are sent to the server to construct a preliminary global estimate of $\theta_a$ for each $a$. The global estimates $\{\hat{\theta}_a^1\}_a$ and the potential matrices $\{V_a^1\}_a$ are then broadcast to all clients, after which phase $p=1$ begins.  Note that after receiving $\{\hat{\theta}_{i,a}^0\}_{i,a}$, the central server will keep a unit vector $\bar{e}_{i,a}= \nicefrac{\hat{\theta}_{i,a}^0}{\|\hat{\theta}_{i,a}^0\|}$ for all $i\in[M],a\in[K]$. Since $\hat{\theta}_{i,a}^0$ is a scaled version of $x_{i,a}$, $\bar{e}_{i,a}$ lies in $\range(x_{i,a})$, and will be utilized to coordinate the arm pulling process (coined as {\it collaborative exploration}), as elaborated in Section~\ref{sec:G_optimal}. 

At the beginning of phase $p$, after receiving the broadcast $\{\hat{\theta}_a^p\}_a$ and $\{V_a^p\}_a$ from the server, each client $i$ will utilize the $(\hat{\theta}_a^p, V_a^p)$ pair to estimate the expected rewards $r_{i,a}$ and obtain the confidence level according to Eqn.~(\ref{eqn: confidence_interval}) for each $a\in\Ac_i^{p-1}$. Based on the constructed confidence interval, client $i$ then eliminates some arms in $\Ac_i^{p-1}$ and obtains $\Ac_{i}^p$.

Next, each client $i$ sends the newly constructed active arm set $\Ac_{i}^p$ to the server. The server then decides $f_{i,a}^p$, the number of times client $i$ pulling arm $a$ during the collaborative exploration step in phase $p$ for each $i\in[M]$ and $a\in\Ac_i^p$. The specific mechanism to decide $f_{i,a}^p$ is elaborated in Section~\ref{sec:G_optimal}. 

After the collaborative exploration step, client $i$ performs least-square estimation (LSE) for each arm $a\in \Ac_i^p$ based on \textit{local} observations collected in the current phase according to Eqn. (\ref{eqn:local_theta}) and then sends it to the server for global aggregation. 
Note that although $\hat{\theta}_{i,a}^p$ lies in $\range(x_{i,a})$, the exact value of $x_{i,a}$ is not revealed to the server, thus preserving the privacy to certain extent.

\begin{algorithm}[H]
\small
\caption{\alg: Central server}
\begin{algorithmic}[1]
	\Require $T$, $M$, $K$, $\alpha$, $f^p$
		\State \textbf{Initialization:} Receive $\{\hat{\theta}_{i,a}^0\}_{i,a}$;  $\bar{e}_{i,a}\gets \frac{\hat{\theta}_{i,a}^0}{\|\hat{\theta}_{i,a}^0\|}$ for all $i\in[M],a\in[K]$;
$ V_a^1\gets\left(\sum_{i\in[M]}\frac{\hat{\theta}_{i,a}^0(\hat{\theta}^0_{i,a})^\TT}{\|\hat{\theta}_{i,a}^0\|}\right)^\dagger$, $\hat{\theta}_a^1\gets V_a^1\left(\sum_{i\in[M]} \hat{\theta}_{i,a}^0\right)$ for all $a\in[K]$; 
Broadcast $\{\hat{\theta}_a^1, V_a^1\}_{a\in[K]}$; $p\gets 1$.
	\While{not reaching the time horizon $T$} 
\State Receive $\{\Ac_i^p\}_{i\in[M]}$; Set $\Ac^p \gets \cup_{i=1}^M \Ac_{i}^p$; Set $\mathcal{R}_a^p \gets \{i: a\in \Ac_i^p\}$.
\State Solve the multi-client G-optimal design in (\ref{eqn: G_opt}), and obtain solution $\pi^p = \{\pi_{i,a}^p\}_{i\in[M],a\in\Ac_i^p}$.

\State For every client $i$, send $ \{f_{i,a}^p := \lceil \pi_{i,a}^pf^p\rceil  \}_{a\in \Ac_i^p}$.

\State Receive $ \{(a,\hat{\theta}_{i,a}^p) \}_{a\in \Ac_i^p}$  from each client $i$.
\For{{$a\in \Ac^p$}} \Comment{\texttt{Global aggregation}}
\vspace{-0.1in}
\begin{align}
V_a^{p+1}&\gets\left(\sum_{i\in \Rc_a^p}f_{i,a}^p\frac{\hat{\theta}_{i,a}^p(\hat{\theta}^p_{i,a})^\TT}{\|\hat{\theta}_{i,a}^p\|^2}\right)^{\dagger},  \quad\hat{\theta}_a^{p+1}\gets V_a^{p+1}\left(\sum_{i\in \Rc_a^{p}}f_{i,a}^p\hat{\theta}_{i,a}^p\right).\label{eqn: theta_hat}
\end{align}

\EndFor 
\State Broadcast $\{(\hat{\theta}_a^{p+1},V_a^{p+1})\}_{a\in\Ac^p} $ to all clients.
\State $p\gets p+1$.
\EndWhile
\end{algorithmic}\label{alg:server}
\end{algorithm}
\vspace{-0.1in}

\subsection{Multi-client G-optimal Design}\label{sec:G_optimal}
In this subsection, we elaborate the core design of \texttt{Fed-PE}, the collaborative exploration step. There are three main design objectives we aim to achieve: 1)
As explained in Section~\ref{sec:challenge}, one of the main challenges in our federated linear contextual bandits setting is that, each client may benefit differently from the information exchange through the central server. To minimize the overall regret, for each arm $a\in \Ac^p$, it is desirable to ensure that after the global aggregation following the collaboration exploration in phase $p$, the uncertainty in $\hat{r}^{p+1}_{i,a}$ {\it across the clients} is balanced. 
2) For each client $i$, in order to eliminate the sub-optimal arms efficiently, it is also important to guarantee that the uncertainty in $\hat{r}^{p+1}_{i,a}$ {\it across the arms} in $\Ac_i^p$ is balanced. 3) Finally, in order to ensure synchronized model updating, we aim to have each client perform the same number of arm pulling in each phase. 

Motivated by those objectives, we propose a multi-client G-optimal design to coordinate the exploration of all clients. Specifically, we define
$\pi_i^p: \Ac_i^p\rightarrow [0,1]$ as a distribution on the active arm set $\Ac_i^p$ for each $i$, and denote $\pi^p:=(\pi_1^p,\ldots,\pi_M^p)$ as a vector in $\Rb^{\sum_{i\in[M]}|\Ac_{i}^p|}$. {Let $e_{i,a}:=x_{i,a}/\|x_{i,a}\|$.} {We note that $e_{i,a}$ equals to either $\bar{e}_{i,a}$ or $-\bar{e}_{i,a}$.} Then, we define a feasible set $\Cc^p\subset\Rb^{\sum_{i\in[M]}|\Ac_{i}^p|}$ as follows:
\begin{align}\label{eqn: feasible set}
\Cc^p =\left\{ \pi^p \middle | \begin{array}{l}
 \pi_{i,a}^p\geq 0,\forall i\in [M], a\in \Ac_i^p,\\ \sum_{a\in \Ac_i^p}\pi_{i,a}^p = 1, \forall i\in[M],  \\
{\rank}(\{\pi_{i,a}^pe_{i,a}\}_{i\in\Rc_a^p}) = \rank(\{e_{i,a}\}_{i\in\Rc_a^p}),\forall a\in \Ac^p
\end{array}\right\}.\end{align}
We can verify that $\Cc^p$ is a convex set. 
The first two conditions ensure that $\{\pi_{i,a}^p\}_{a}$ form a valid distribution for each client $i$. We name the last condition as the ``{\it rank-preserving}'' condition. We note that the subspace spanned by the LHS of the rank-preserving condition is always a subset of that spanned by the RHS. Thus, once the rank is preserved, the subspaces spanned by the LHS and the RHS are the same. Thus, this condition ensures that every dimension of $\theta_a$ lying in $\range(\{x_{i,a}\}_{i\in \Rc_a^p})$ will be explored under the collaborative exploration. Any violation of the ``rank-preserving'' condition will lead to information missing along the unexplored dimensions, which shall be prevented in order to reduce the uncertainty level regarding arm $a$ at every client $i\in \Rc_a^p$. 

Then, we formulate the so called multi-client G-optimal design problem as follows: 
\begin{align}\label{eqn: G_opt}
&\text{minimize\ \  } G(\pi) = \sum_{i = 1}^M\max_{a\in \Ac_i^p}  e_{i,a}^\TT\bigg(\sum_{j\in \Rc_a^p} \pi_{j,a}^pe_{j,a}e_{j,a}^\TT\bigg)^{\dagger}e_{i,a} \quad\text{s.t. } \pi^p\in\Cc^p.
\end{align}
We note that $e_{i,a}^\TT\big(\sum_{j\in \Rc_a^p} \pi_{j,a}^p e_{j,a}e_{j,a}^\TT\big)^{\dagger}e_{i,a}$ can be interpreted as an {\it approximate} measure of the uncertainty level along dimension $x_{i,a}$ if the arms are explored locally according to distributions $\{\pi_i^p\}_i$. Thus, the objective function is an approximate measure of the total uncertainties in the least explored arms at each of the clients. By solving (\ref{eqn: G_opt}), the aforementioned three design objectives can be met. {We point out that although the server does not known ${e}_{i,a}$, the objective function remains the same when $e_{i,a}$ is replaced by $\bar{e}_{i,a}$. Thus, the server can simply use $\bar{e}_{i,a}$ to solve (\ref{eqn: G_opt}).}  

After solving (\ref{eqn: G_opt}) and obtaining $\{\pi_{i,a}^p\}$, the server would set $\{f_{i,a}^p := \lceil \pi_{i,a}^pf^p\rceil\}_{a\in \Ac_i^p}$ and send it to client $i$. Note that after taking the ceiling function, $\sum_{a\in \Ac_i^p}f_{i,a}^p$ may be greater than $f^p$. To ensure synchronized updating, each client $i$ would keep pulling the estimated best arm $\hat{a}_i^p$ until the phase length equals $f^p+K$.

We note that the multi-client G-optimal design formulated in (\ref{eqn: G_opt}) is related to the G-optimal design for the single-player linear bandits problem discussed in \cite{LS19bandit-book}, and the DELB algorithm for the distributed linear bandits in \cite{wang2019distributed}. However, for such cases, the player(s) faces a single bandit problem, thus the objective is to simply obtain a distribution over a so-called core set of arms in order to minimize the maximum uncertainty {\it across the arms}. In contrast, due to the multiple clients involved in the federated bandits setting and the heterogeneous reward distributions, we are essentially solving $M$ \emph{coupled} G-design problems, one associated with each client. Such coupling effect fundamentally changes the nature of the problem, leading to very different characterization of the problem and numerical approaches. 

In Appendix~\ref{appx:G-optimal}, we analyze an equivalent problem of the multi-client G-optimal design. We note that such equivalence essentially generalizes the equivalence between the original G-optimal design and D-optimal design in \cite{LS19bandit-book} to the coupled design case.  While the original G-design problem can be approximately solved through the Frank-Wolfe algorithm under an appropriate initialization \citep{MinimumVolumnEllipsoid}, solving the multi-client G-optimal design problem is numerically non-trivial. In Appendix~\ref{appx:CoordinateAscent}, we propose a block coordinate ascent algorithm to solve the equivalent problem of (\ref{eqn: G_opt}) efficiently with guaranteed convergence.

\subsection{Theoretical Analysis of \alg}
We now characterize the performance of the \alg algorithm.
\begin{theorem}\label{thm: BoundAlgorithm1}
Under Assumption~\ref{assump:bounded}, with probability at least $1-\delta$, the cumulative regret under \alg scales in
$O\left(\sqrt{dKMT(\log(K(\log T)/\delta)+\min\{d,\log M\})}\right)$ 
and the communication cost scales in $O(Md^2K\log T)$.
\end{theorem}

The complete version of Theorem \ref{thm: BoundAlgorithm1} and its proof can be found in Appendix \ref{sec: ProofAlgorithm1}. 
For the communication cost, at each phase $p$, client $i$ uploads at most $K$ local estimates with dimension $d$ and downloads at most $K$ global estimates and potential matrices with dimension $d$ and $d^2$, respectively. Thus, the upload cost is $O(MdK\log T)$ and the download cost is $O(Md^2K\log T)$. 

\vspace{-0.1in}
\begin{remark}
When we set {$\delta=O(\sqrt{\nicefrac{dK}{MT}})$}, the overall regret scales in $O(\sqrt{dKMT\log(MKT)})$, and the per-client regret scales in $O(\sqrt{dKT\log(MKT)/M})$. While the minimax lower bound for standard stochastic MAB scales in $\Omega(\sqrt{KT})$, the collaborative learning induced by \alg leads to $\sqrt{d/M}$-fold reduction of the per-client regret. We also note that for single-player linear contextual bandits with disjoint parameters, the best known upper bound scales in $\tilde{O}(\sqrt{dKMT})$ \citep{dimakopoulou2018estimation} over $MT$ arm pulls, which indicates that the regret of \alg is close to the state-of-the-art centralized algorithms at a communication cost in $O(\log T)$. 
\end{remark}
\subsection{Enhanced Fed-PE}\label{sec:enhanced}
The original \alg algorithm requires exponentially increasing $f^p$ in order to achieve the regret upper bound in Theorem~\ref{thm: BoundAlgorithm1}. This is because in each phase $p$, we only utilize the rewards collected in phase $p-1$ to estimate $\theta_a$. While this simplifies the analysis, the measurements collected in earlier phases cannot be utilized. In order to overcome this limitation, we propose an \texttt{Enhanced} \alg algorithm by leveraging all historical information. \texttt{Enhanced} \alg achieves different tradeoffs between communication cost and regret performance by adjusting $f^p$. 
The detailed description and analysis of \texttt{Enhanced} \alg for different selection of $f^p$ can be found in Appendix~\ref{appx:enhanced}. 

\subsection{Lower Bound}
To derive a tight lower bound, we focus on a set of representative policies defined as follows.
\begin{definition}[Collinearly-dependent policy]\label{defn:col_dependent_policy}
Two clients $i$ and $j$ are called \emph{collinear} if there exist an arm $a\in[K]$ and a subset $\Sc\subset [M]$ such that the following conditions are satisfied: 1) $x_{i,a}\notin \Span(\{x_{m,a}|m\in\Sc\})$; and 2) $x_{i,a}\in \Span(\{x_{m,a}|m\in\Sc\}\cup \{x_{j,a}\})$.
For any two clients $i$ and $j$ that are not collinear, if the action of client $i$ is independent of the action of $j$ under a policy $\pi$, then, the policy is called a collinearly-dependent policy. 
\end{definition}
We note that the definition of collinearly-dependent policies is actually quite natural. Intuitively, for two clients that are not collinear, their local observations on any arm $a$ cannot be utilized to improve each other's knowledge of their own local models. As a result, they should not affect each other's decision-making process. As shown in Appendix~\ref{appx:lower}, we can verify that the most celebrated ridge regression based LinUCB type of policies~\citep{Li:2010:LinUCB}, Thompson sampling based polices with Gaussian priors~\citep{Agrawal:2013:linear}, and least-square estimation based policies, including \texttt{Fed-PE}, all fall in this category. 
\begin{theorem}\label{thm:lower}
For any collinearly-dependent policy, there exists an instance of the federated linear contextual bandits such that the regret is lower bounded as $R(T)=\Omega(\sqrt{dKMT})$.
\end{theorem}

\begin{remark}
Theorem~\ref{thm:lower} essentially shows that, even if raw data transmission and instantaneous communication are allowed and other collinearly-dependent policies are adopted, we cannot improve the order of the regret summarized in Theorem~\ref{thm: BoundAlgorithm1} much, i.e., \alg is order-optimal up to $\sqrt{\log (KMT)}$. 
\end{remark}

The proof of Theorem~\ref{thm:lower} relies on the construction of a special instance of the federated linear contextual bandits where the clients can be divided into $d$ groups. Clients in each group face the same $K$-armed stochastic bandits model locally, while clients from two distinct groups are not collinear. Analyzing the regret bound in each individual group, we can show that it is lower bounded by $\Omega(\sqrt{KMT/d})$. Then, by utilizing the property of collinearly-dependent policies, we can show that the overall regret is lower bounded by $\Omega(\sqrt{dKMT})$ for this scenario. More discussions on the collinearly-dependent policies and the complete proof of Theorem~\ref{thm:lower} can be found in Appendix~\ref{appx:lower}.

\section{Federated Linear Contextual Bandits: Shared Parameter Case}\label{sec:shared}
%\vspace{-0.05in}
The \alg algorithm can be slightly modified for the shared parameter case where $\theta_a = \theta, \forall a\in[K]$. While the client side operation stays the same, the global aggregation step at the server side in (\ref{eqn: theta_hat}) can be changed by letting the ``potential matrix'' $V^p $ be $ ( \sum_{a\in[K]}(V_a^p)^\dagger)^\dagger$, and the global estimator
$\hat{\theta}^p$ be $V^p\left(\sum_{i\in[M]}\sum_{a\in\Ac_i^{p-1}}f_{i,a}^{p-1}\hat{\theta}_{i,a}^{p-1}\right)$.
Below, we present the main result for this case and leave the detailed algorithm description and regret analysis in Appendix~\ref{appx:enhanced}.

\begin{theorem}\label{thm: BoundSharedParameter}
Under Assumption~\ref{assump:bounded}, with probability at least $1-\delta$, the regret of the adapted \alg for the shared parameter case is upper bounded by 
$O\left(\sqrt{dMT(\log({(\log T)}/{\delta})+\min\{d,\log MK\})}\right)$, and the communication cost scales in {$O((KdM + d^2M)\log T)$}.
\end{theorem}
\begin{remark} 
By setting $\delta=O(\sqrt{\nicefrac{1}{MT}})$, we can show that the overall regret scales in $O(\sqrt{dMT\log(MTK)})$. We note that this bound improves the regret bound for the no differential privacy guarantee case in \cite{dubey2020differentiallyprivate} by a factor of $\sqrt{\log T}$, although our settings are slightly different. By assuming all clients face the same linear bandits in this shared parameter setting, the regret in the federated setting over horizon $T$ must be worse than the linear bandits over horizon $MT$. Since the latter is lowered bounded by $\Omega(\sqrt{dMT})$~\citep{pmlr-v15-chu11a}, the minimax regret for the shared parameter setting is bounded by $\Omega(\sqrt{dMT})$ as well. Thus, the modified \alg is near-optimal for this case. 

In terms of the communication cost, the uploading cost stays the same as in the disjoint parameter case, while the broadcast cost is reduced by a factor of $K$, since only one potential matrix needs to be broadcast for the shared parameter case. 
\end{remark}
\vspace{-0.05in}

\vspace{-0.1in}
\section{Experiments}\label{sec:simulation}
Experiment results using both synthetic and real-world datasets are reported in this section to evaluate \alg and the proposed enhancement. Additional experimental details and more experimental results can be found in Appendix~\ref{appx:experiment}.  
We consider four different algorithms, namely, \texttt{Fed-PE}, \texttt{Enhanced Fed-PE}, local UCB without communication, and a modified \alg algorithm with full information exchange after collaborative exploration in each phase (coined as `Collaborative' in Figure~\ref{fig:syn1}). 
For all experiments, we set $T = 2^{17}, f^p = 2^p, p\in\{1,2,\ldots,16\}$, and run 10 trials. For \alg and its variants, we choose $\delta=0.1$. Note that other values of $\delta$  may further improve the regret.
We evaluate the algorithms on both synthetic and MovieLens-100K datasets.

\textbf{Synthetic Dataset}: We first set $M=100, K=10,$ and $d=3$. We set $\{\theta_a\}$ as the canonical basis of $\Rb^3$. The feature vectors $x_{i,a}$ are generated randomly ensuring that the suboptimality reward gaps lie in $[0.2,0.4]$ and $\ell=0.5, L=1$. The per-client cumulative regret as a function of $T$ is plotted in Figure~\ref{fig:syn1}(a). We see that \texttt{Enhanced Fed-PE} outperforms \texttt{Fed-PE} while being slightly worse than `Collaborative'. This indicates that keeping feature vectors $x_{i,a}$ private to clients does not impact the learning performance significantly. All \alg related algorithms outperform local UCB when $T$ is sufficiently large, demonstrating the effectiveness of communication in improving learning locally. We also set $K=10$, $d=4$, and vary the number of clients $M$. The performance of \texttt{Enhanced Fed-PE} is plotted in Figure~\ref{fig:syn1}(b). We note that the per-client regret  monotonically decreases as $M$ increases, corroborating the theoretical results.     

\textbf{Movielens Dataset}: We then use the MovieLens-100K dataset~\citep{harper2016movielens} to evaluate the performances. Motivated by \cite{bogunovic2020stochastic-RobustAttack}, we first complete the rating matrix $R=[r_{i,a}]\in\Rb^{943\times 1682}$ through collaborative filtering \citep{collaborative-filtering}, and then use non-negative matrix factorization with 3 latent factors to get $R=WH$, where $W\in\Rb^{943\times 3}$, $H\in\Rb^{3\times 1682}$. Let $x_{i,a}$ be the $i$th row vector of $W$. We apply the $k$-means algorithm to the row vectors of $H$ to produce $K = 30$ groups (arms), and let $\theta_a$ be the center of the $a$-th group. Finally, we randomly choose $M=100$ users' feature vectors. 
We observe that $0.4\leq \|x_{i,a}\|^2\leq 0.8, $ and the suboptimality gaps lie in $[0.01, 0.8]$. The regret performances of the algorithms are plotted in Figure~\ref{fig:syn1}(c). The curves show similar characteristics as in Figure~\ref{fig:syn1}(a). These results demonstrate the effectiveness of collaborative learning in the federated bandits setting.

\begin{figure*}[t]	
\vspace{-0.3in}
	\centering  
	\subfigure[\small Synthetic regret.]{\includegraphics[width=0.32\textwidth]{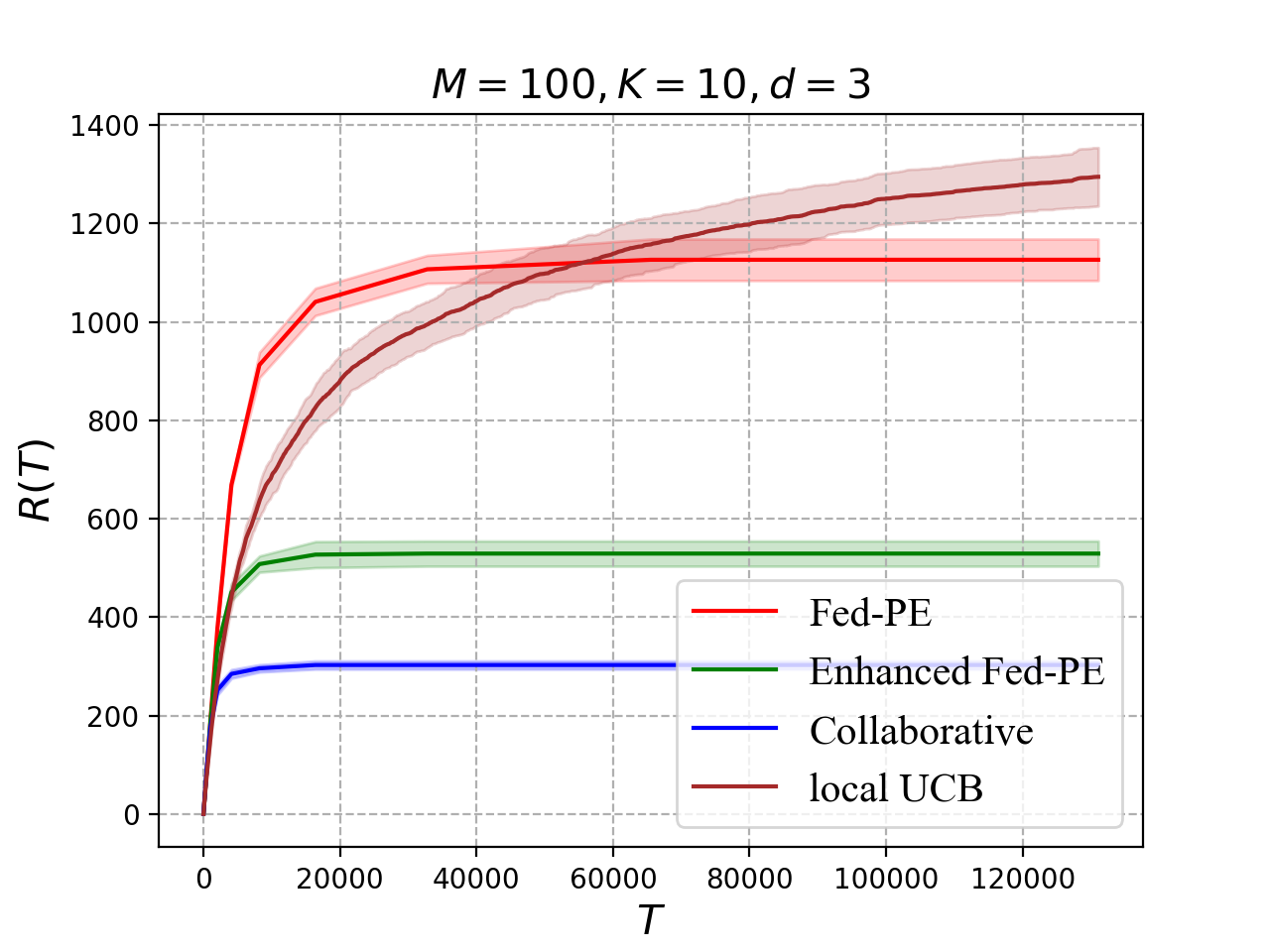}}
	\subfigure[\small Regret with varying $M$.]{\includegraphics[width=0.32\textwidth]{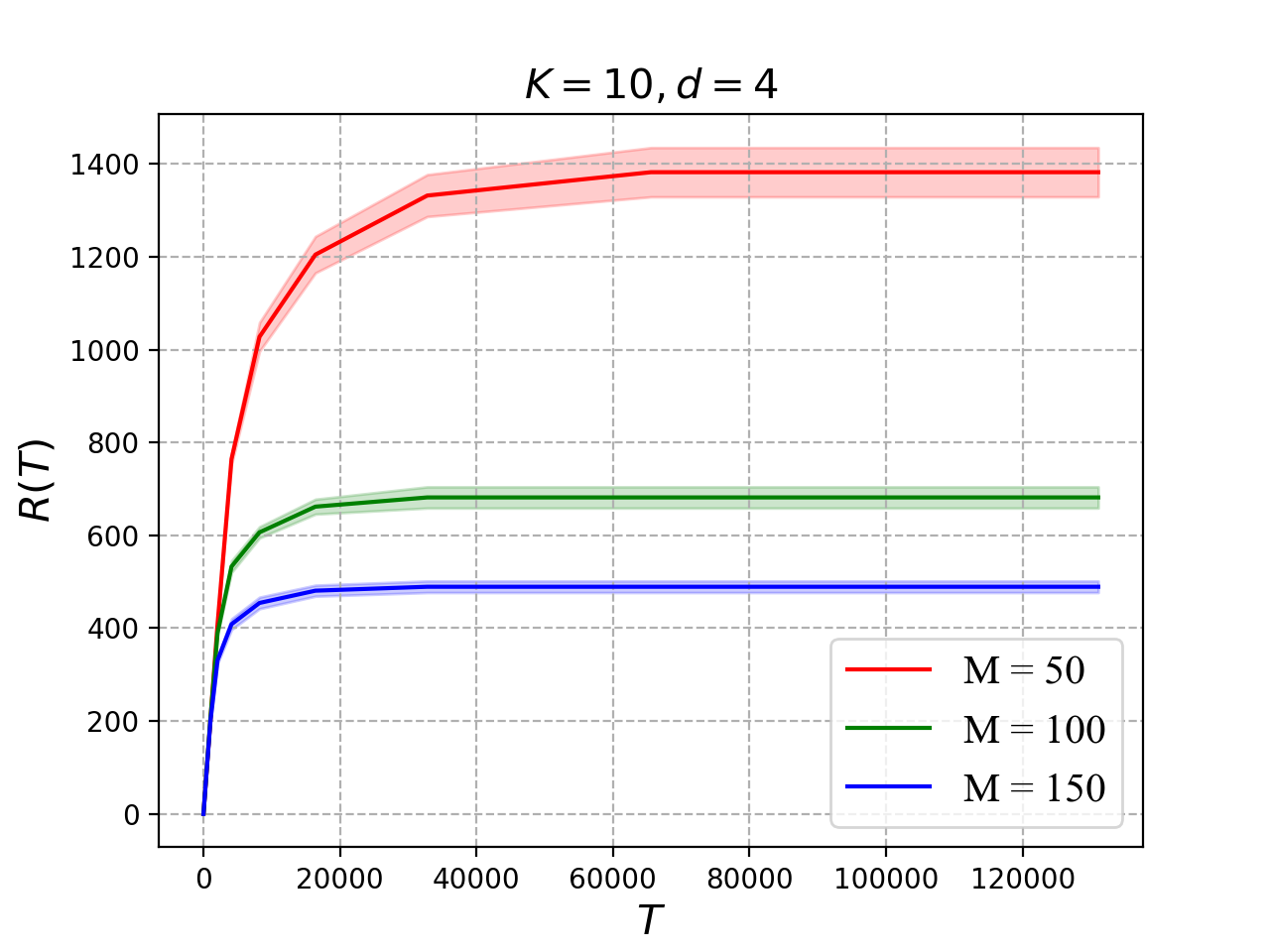}}
		\subfigure[\small MovieLens regret.]{\includegraphics[width=0.32\textwidth]{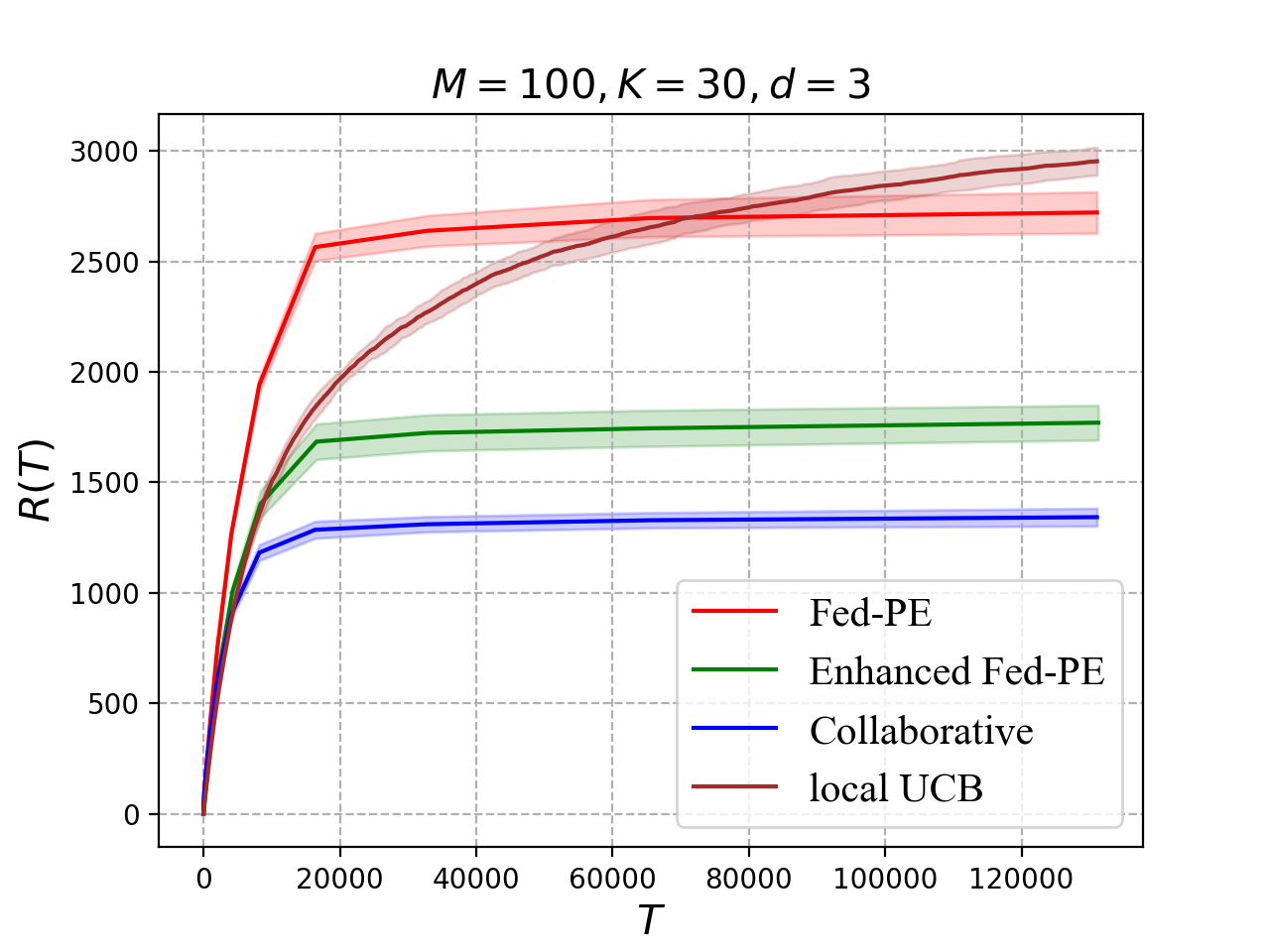}}
		\vspace{-0.1in}
	\caption{\small Pseudo-regret over $T$. Shaded area indicates the standard deviation.}\label{fig:syn1}
	\vspace{-0.1in}
\end{figure*}

\section{Discussion and Conclusion}\label{sec:discussion}
In this work, we have considered a novel federated linear contextual bandits model, which naturally connects local stochastic MAB models with linear contextual bandits through common global parameters. While each client can only observe a projection of each global parameter in its own subspace, \alg utilizes the geometric structure of the local estimates to reconstruct the global parameters and guide efficient collaborative exploration. Theoretical analysis indicates that \alg achieves near-optimal regret for both disjoint and shared parameter cases with a communication cost in the order of $O(\log T)$.  

An interesting open question is whether we can further reduce the communication cost without downgrading the regret performance. 
In particular, we note the the original single-player G-optimal design allows for a sparse solution whose support is of size $d(d+1)/2$. Our numerical results indicate that such sparse solutions exist for the multi-client G-optimal design as well. Utilizing the sparsity of the solution may reduce the communication cost significantly. Theoretical characterization of the existence of such sparse solutions is our next step. 

Another possible direction to explore is to incorporate the differential privacy mechanism to the \alg framework. Although local feature vectors $\{x_{i,a}\}$ are kept private under \texttt{Fed-PE}, local estimate $\hat{\theta}_{i,a}$ lies in $\range(x_{i,a})$, thus revealing the direction of $x_{i,a}$ to the central server. We aim to add certain perturbation on $\hat{\theta}_{i,a}$ in order to obfuscate the direction information without significantly affecting the regret performance.

\begin{ack}
 The work of RH and JY was supported by the US National Science Foundation under Grants CNS-1956276, CNS-2003131, CNS-2114542, and ECCS-2030026. CS acknowledges the funding support by the US National Science Foundation under Grants ECCS-2029978, ECCS-2033671, and CNS-2002902. WW's work was done before he joined Facebook.
\end{ack}

\appendix

\allowdisplaybreaks

\section{Related Works}
\myparagraph{Linear contextual bandits.} The reward model considered in this work is similar to that of linear contextual bandits in the literature. The single-agent setting is first introduced in~\cite{Auer:2003:UCB} through the LinRel algorithm and is subsequently improved through the OFUL algorithm in \cite{DaniHK08} and the LinUCB algorithm in \cite{Li:2010:LinUCB}. \cite{Tsitsiklis:2010:LinearBandits} extend the work of \cite{DaniHK08} by considering both optimistic and explore-then-commit strategies. 
A modified version of LinUCB, named SupLinUCB, is considered in \cite{pmlr-v15-chu11a}, which is further improved by \cite{Valko:2013:Kernel}.
This line of literature typically allows for sequential stochastic {\citep{goldenshluger2013,Bastani:2015}} or adversarial {\citep{Raghavan2018colt,KannanMRWW2018nips}} context arrivals, which is different from the fixed context setting and distributed nature of the federated linear contextual bandits considered in this work. 

\myparagraph{Batched and parallel bandits.} Batched bandits is the setting where the time axis is partitioned into batches, and the decision at each time $t$ depends only on observations from batches strictly prior to the current one~\citep{Perchet:2016}. While \cite{Perchet:2016} study the two-armed stochastic bandits, \citet{gao2019batched} and \citet{han2020sequential} extend it to the $K$-armed MAB and the linear contextual bandits, respectively. The parallel linear contextual bandits studied in \cite{chan2021parallelizing} is essentially similar to the batched bandits setting in which $P$ distinct processors perform simultaneous queries in batches.
We note that the phased decision processes in our setting is similar to batched bandits, as the decision in a phase only depends on observations collected in previous phases. However, the distributed network structure and the associated communication and privacy concerns are not considered in the batched bandits setting.

\myparagraph{Collaborative and distributed bandits.} 
The collaborative and distributed bandits with multiple agents has gained growing interest recently. One research direction is the multi-player multi-armed bandits (MP-MAB) problem \citep{liu2010distributed,anandkumar2011distributed}, where collision occurs when two players simultaneously play the same arm. Without explicit communication among players, the main focus of MP-MAP is to avoid~\citep{rosenski2016multi,besson2017multi,avner2014concurrent} or exploit~\citep{boursier2018sic,Shi2020aistats} such collisions in order to maximize the collective cumulative rewards.

When action collision is not considered, \citet{landgren2016distributed,landgren2018social} and \citet{martinez2019decentralized} study distributed bandits in which multiple agents face the same MAB instance, and the agents collaboratively share their estimates over a fixed communication graph in order to design consensus-based distributed estimation algorithms to estimate the mean of rewards at each arm. \citet{szorenyi13} considers a similar setup where in each round an agent is able to communicate with a few random peers.  \citet{pmlr-v48-korda16-distributed-cluster} considers the case where clients in different unknown clusters face independent bandit problems, and every agent can communicate with only one other agent per round. {Similar approaches have been extended to the contextual bandits and recommender systems, where user/context similarities are exploited to improve sample efficiency of online recommendation \citep{cesa-bianchi2013nips,wu2016contextual,pmlr-v32-gentile14-online-clustering,gentile2017context}.}
The communication and coordination among the clients in those works are fundamentally different from our work. 

\citet{wang2019distributed} investigates communication-efficient distributed linear bandits, where the agents can communicate with a server by sending and receiving packets. It proposes two algorithms, namely, DELB and DisLinUCB, for fixed and time-varying action sets, respectively. The fixed action set setting is quite similar to our setup, except that it assumes that all agents face the same bandits problem, which does not take data heterogeneity into consideration.

\myparagraph{Federated bandits.} A few recent works have touched upon the concept of federated bandits. With heterogeneous reward distributions at local clients, \cite{shi2021aaai} and \cite{shi2021aistats} investigate efficient client-server communication and coordination protocols for federated MAB without and with personalization, respectively. \cite{agarwal2020federated} studies regression-based contextual bandits as an example of the federated residual learning framework, where the reward of a client depends on both a global model and a local model. \citet{li2020federated2} and \cite{Zhu_2021} focus on differential privacy based local data privacy protection in federated bandits. While the linear contextual bandit model considered in \citet{dubey2020differentiallyprivate} is similar to the shared parameter case studied this work, it focuses on federated differential privacy and proposes a LinUCB-based algorithm, which incurs a much higher regret compared with our result.

\section{Preliminaries}\label{appx:prelim}
\subsection{Notations}
Throughout this paper, we use  $\|x\|_V$ to denote $\sqrt{x^\TT V x}$. The \textit{range} of a matrix $A$, denoted by $\range(A)$, is the subspace spanned by the column vectors of $A$. Occasionally we use $\Span(A)$ to denote $\range(A)$ as well. We use $\rank(\{x_i\}_i)$ to denote
the maximum number of linearly independent vectors in $\{x_i\}_i$, and $\Span(\{x_i\}_i)$ to denote the subspace spanned by them. 
For any positive semi-definite matrices $X$ and $Y$ of the same size, $X\succeq Y$ implies  that $X-Y$ is positive semi-definite. $\lv\{\cdot\}$ is the indicator function while $I_d$ is a $d\times d$ identify matrix.

\subsection{Matrix Analysis}
\begin{definition}[Pseudo-inverse of matrices]\label{def: pinv}
Given a matrix $A$, the pseudo-inverse of $A$ is a {unique} matrix, denoted by $A^\dagger$, that satisfies the following properties:
\[AA^\dagger A = A,\quad A^\dagger A A^\dagger = A^\dagger,
\quad (AA^\dagger)^\TT = AA^\dagger,\quad (A^\dagger A)^\TT = A^\dagger A.\]
\end{definition}

\begin{definition}[Pseudo-determinant of matrices]\label{defn:peudo_det}
For a square matrix $A\in\mathbb{R}^{n \times n}$, the pseudo-determinant is defined as 
\[\Det(A) = \lim_{\epsilon\rightarrow 0}\frac{\det(A + \epsilon I)}{\epsilon^{n-\rank(A)}},\]
and the generalized pseudo-determinant with respect to a degree $s$ is defined as
\[\Det_s(A) = \lim_{\epsilon\rightarrow 0} \frac{\det(A+\epsilon I)}{\epsilon^{n - s}}.\]
\end{definition}
\begin{remark}\label{remark:peudo}
Note that $\Det_s(A) = 0$ when $s>\rank(A)$, and $\Det_s(A) = \infty$ if $s<\rank(A)$. When $s=\rank(A)$, $\Det_{\rank(A)}(A)$ becomes $\Det(A)$, the standard pseudo-determinant of matrix $A$. We introduce the definition of generalized pseudo-determinant in order to handle cases when the rank of $A$ is uncertain.
\end{remark}

We first characterize some important properties of the pseudo-inverse of a matrix.

\begin{Proposition}\label{prop: limit_relations_pseudo}
Suppose $A\in\Rb^{n\times n}$ is a symmetric matrix, and $v\in\range(A)$. Then, the following limit relations hold:
\begin{itemize}[leftmargin=18pt,topsep=0pt, itemsep=0pt,parsep=0pt] 
    \item [1)] $A^\dagger = \lim_{\epsilon\rightarrow0}(A+\epsilon I)^{-1}A(A+\epsilon I)^{-1}$.
    \item [2)] $A = \lim_{\epsilon\rightarrow0}A(A+\epsilon I)^{-1} A$.
    \item [3)] $A^\dagger v = \lim_{\epsilon\rightarrow0}(A+\epsilon I)^{-1} v$.
\end{itemize}
\end{Proposition}

\begin{proof}
Since $A$ is a symmetric matrix, we can decompose it as $A = U\Lambda U^\TT$, where $UU^\TT = I$ and $\Lambda = \diag(\lambda_1,\ldots,\lambda_n)$. We assume $|\lambda_1|\geq|\lambda_2|\geq\ldots\geq |\lambda_n|$ are the eigenvalues of $A$. If $\rank(A)=d\leq n$, then $|\lambda_d|>0 = |\lambda_{d+1}|$. There must exist $d$ constants $c_1,\ldots,c_d$ such that $v = \sum_{s=1}^d c_su_s$, where $u_s$ is the $s$th column vector of matrix $U$.

Note that when $A$ is symmetric, 
\[A^\dagger = U\diag\left(\frac{1}{\lambda_1},\ldots,\frac{1}{\lambda_d}, 0 ,\ldots, 0\right)U^\TT.\]

For the first two relations, note that $A+\epsilon I = U\Lambda U^\TT + \epsilon UU^\TT = U\diag(\lambda_1+\epsilon,\ldots,\lambda_n+\epsilon)U^\TT$. Thus, for sufficiently small $\epsilon$, we have
\begin{align}\label{eqn:inverse}
(A+\epsilon I)^{-1} = U\diag\left(\frac{1}{\lambda_1 + \epsilon}, \ldots, \frac{1}{\lambda_n + \epsilon}\right) U^\TT.
\end{align}

Substituting (\ref{eqn:inverse}) into 1) and 2), the first two relations can be readily obtained.

For 3), we express $v = Uc$, where $c = (c_1,\ldots,c_d,0,\ldots,0)^\TT$. Then we have
\begin{align}
\lim_{\epsilon\rightarrow0}(A+\epsilon I)^{-1} v& = \lim_{\epsilon\rightarrow 0} U\diag\left(\frac{1}{\lambda_1+\epsilon},\ldots,\frac{1}{\lambda_n+\epsilon}\right) U^\TT Uc\\
& = \lim_{\epsilon\rightarrow 0 } U\left(\frac{c_1}{\lambda_1+\epsilon},\ldots,\frac{c_d}{\lambda_d + \epsilon},0,\ldots,0\right)^\TT\\
& = U\diag\left(\frac{1}{\lambda_1},\ldots,\frac{1}{\lambda_d},0,\ldots,0\right)U^\TT Uc\\
& = A^\dagger v.
\end{align}
\end{proof}

Next, we present several useful lemmas that are needed in the analysis of the generalized G-optimal design discussed in Appendix \ref{appx:G}.

\begin{lemma}[Jacobi's formula \citep{Jacobi1999Formula}]\label{lemma: Jacobi}
If $A:\mathbb{R}\rightarrow\mathbb{R}^{n\times n}$ is differentiable in its domain, and $A(t)$ is invertible, then \[\frac{d\det(A(t))}{dt} = \tr\left(\adj(A(t))\cdot\frac{dA(t)}{dt}\right),\]
where $\adj(A(t))$ is the adjugate of $A(t)$, i.e. $\adj(A(t))A(t) = \det(A(t))\cdot I_n$. %Here $I$ is an $n\times n$ identify matrix.
\end{lemma}

Lemma~\ref{lemma: logdet} below explicitly identifies the derivative of $\log\Det(\cdot)$. We note that the general derivative of pseudo-determinant is studied in \cite{differentiatingPseudo2018holbrook}. 

\begin{lemma}\label{lemma: logdet}
For a given index set $\mathcal{I}$ and a set of vectors $v_{s}\in\mathbb{R}^{n}$ where $s\in\mathcal{I}$, 
consider the following function in terms of $\pi:=\{\pi_{s}\}_{s\in\mathcal{I}}$:
\[F(\pi) = \log\Det\left(\sum_{s\in\mathcal{I}}\pi_{s}v_{s}v_{s}^\TT\right)\]
defined over $\{\pi|\pi_s\geq 0, \forall s\in \mathcal{I}, \rank\left(\sum_{s\in\mathcal{I}}\pi_{s}v_{s}v_{s}^\TT\right)=\rank(\{v_{s}\}_{s\in\mathcal{I}}):= D\}$.

Then, the $\iota$-th coordinate of the gradient of $F(\pi)$ satisfies
\begin{equation}\label{eqn: gradient}
    (\nabla F)_{\iota} = v_{\iota}^\TT\left(\sum_{s\in\mathcal{I}} \pi_{s}v_{s}v_{s}^\TT\right)^{\dagger}v_{\iota}.
\end{equation}
\end{lemma}

\begin{proof}
Denote $A:=\sum_{s\in\mathcal{I}} \pi_{s}v_{s}v_{s}^\TT$. 
Then, according to the definition of pseudo-determinant in Definition~\ref{defn:peudo_det}, we have
\begin{align}
     \frac{\partial F(\pi)}{\partial \pi_{\iota}}&
     =\frac{1}{\Det(A)}\frac{\partial\Det(A)}{\partial\pi_{\iota}}\\
     &=\lim_{\epsilon\rightarrow0}\frac{1}{\Det(A)}\frac{\partial\det(A+\epsilon I)}{\epsilon^{n-D}\partial\pi_{\iota}}\label{eqn:F_derivative3}\\
     & = \lim_{\epsilon\rightarrow 0}\frac{1}{\Det(A)\epsilon^{n-D}}\trace(\adj(A+\epsilon I) v_{\iota}v_{\iota}^\TT)\label{eqn:F_derivative4}\\
     & = \lim_{\epsilon\rightarrow 0}\frac{1}{\Det(A)\epsilon^{n-D}}\trace(\det(A+\epsilon I) (A+\epsilon I)^{-1} v_{\iota}v_{\iota}^\TT)\label{eqn:F_derivative5}\\
     & = \lim_{\epsilon\rightarrow0} v_{\iota}^\TT(A+\epsilon I)^{-1}v_{\iota}\label{eqn:F_derivative6}\\
     & = v_{\iota}^\TT\left(\sum_{s\in\mathcal{I}} \pi_{s}v_{s}v_{s}^\TT\right)^{\dagger}v_{\iota},\label{eqn:F_derivative7}
\end{align}
where Eqn. (\ref{eqn:F_derivative4}) follows from Lemma~\ref{lemma: Jacobi}, (\ref{eqn:F_derivative5}) follows from the definition of adjugate, and (\ref{eqn:F_derivative6}) follows from the definition of $\Det(A)$ in Definition~\ref{defn:peudo_det} and the fact that $\trace(AB)=\trace(BA)$ when their dimensions match. Eqn. (\ref{eqn:F_derivative7}) is due to Proposition \ref{prop: limit_relations_pseudo}-3) since $v_{\iota}\in \range(A)$.
\end{proof}

The following lemma establishes the concavity of $\log\Det(\cdot)$, which generalizes the result for $\log\det(\cdot)$ \citep{boyd_vandenberghe_2004} to 
the pseudo-determinant case. 

\begin{lemma}\label{lemma: concave}
Let $S$ be a subspace of $\mathbb{R}^n$. Then, $\log\Det(A)$ is a concave function in any convex subset of $\Mc_S=\{A| A= A^\TT,  \range(A) = S, A\succeq 0\}$, i.e., the collection of symmetric positive semi-definite matrices with range $S$. 
\end{lemma}

\begin{proof} We show the concavity by
considering an arbitrary line, defined by $X=Y+tZ$, 
where $X,Y,Z$ are matrices lying in  
a convex subset of $\Mc_S$. Assume $\rank(X) = k$.

Let $g(t) = \log\Det(Y+tZ)$. Then,
\begin{align*}
    g(t)& = \log\lim_{\epsilon\rightarrow0}\frac{\det(Y+tZ+\epsilon I)}{\epsilon^{n-k}}\\
    &=\lim_{\epsilon\rightarrow0}\log\det(Y+\epsilon I+tZ) - (n-k)\log\epsilon\\
    &=\lim_{\epsilon\rightarrow0}\log \det \left((Y+\epsilon I)^{\frac{1}{2}}\left(I+t(Y+\epsilon I)^{-\frac{1}{2}}Z(Y+\epsilon I)^{-\frac{1}{2}}\right)(Y+\epsilon I)^{\frac{1}{2}}\right) - (n-k)\log\epsilon\\
    &= \lim_{\epsilon\rightarrow0} \log \det (Y+\epsilon I) + \log\det\left(I+t(Y+\epsilon I)^{-\frac{1}{2}}Z(Y+\epsilon I)^{-\frac{1}{2}}\right) - (n-k)\log\epsilon\\
    &= \lim_{\epsilon\rightarrow0} \log \frac{\det (Y+\epsilon I)}{\epsilon^{n-k}} + \log\det\left(I+t(Y+\epsilon I)^{-\frac{1}{2}}Z(Y+\epsilon I)^{-\frac{1}{2}}\right)\\
  &  =\log\Det Y + \lim_{\epsilon\rightarrow0} \log\det\left(I+t(Y+\epsilon I)^{-\frac{1}{2}}Z(Y+\epsilon I)^{-\frac{1}{2}}\right)\\
    &=\log\Det Y + \sum_{i=1}^k\log\det(1 + t\lambda_i),
    \end{align*}
where $\lambda_1,\lambda_2,\ldots,\lambda_k$ are eigenvalues of $\lim_{\epsilon\rightarrow 0}(Y+\epsilon I)^{-\frac{1}{2}}Z(Y+\epsilon I)^{-\frac{1}{2}}$. 

Taking the second derivative of $g(t)$, we have 
\begin{align}
\frac{d^2g(t)}{dt^2} = -\sum_{i=1}^k\frac{\lambda_i^2}{(1+t\lambda_i)^2}<0.
\end{align}
Therefore, $g(t)$ is concave, and the proof is complete.
\end{proof}

\subsection{Important Inequalities}
In this subsection, we present two important inequalities that will be utilized in Appendix~\ref{sec: ProofAlgorithm1} and Appendix~\ref{appx:enhanced}, respectively.

\begin{lemma}\label{lemma: subG-L2}
Let $\zeta\in\mathbb{R}^{n}$ be a 1-sub-Gaussian random vector conditioned on $\mathcal{F}_p$ and $A\in\mathbb{R}^{n\times n}$ be a $\mathcal{F}_p$-measurable matrix. Let $\lambda>0$ and $ \det(I_n-2\lambda AA^\TT)>0$. Then, we have
\begin{equation}
    \mathbb{E}\left[e^{\lambda\|A\zeta\|^2}\bigg|\mathcal{F}_p\right]
    \leq \sqrt{\frac{1}{\det(I_n-2\lambda AA^\TT)}}.
\end{equation}
\end{lemma}

\begin{proof}
Assume $x\thicksim \Nc(0,2\lambda I_n)$ is an independent Gaussian random vector.
Then, we have
\begin{align}
\mathbb{E}_{x,\zeta}\left[e^{x^\TT A\zeta}\bigg|\Fc_p\right]& = \mathbb{E}_{\zeta}\left[\mathbb{E}_x\left[e^{x^\TT A\zeta}|A, \zeta\right]\bigg|\Fc_p\right] = \mathbb{E}\left[e^{\lambda\|A\zeta\|^2}\bigg|\Fc_p\right].\label{eqn: E(zeta^2)}
\end{align}
On the other hand,
\begin{align}
\mathbb{E}_{x,\zeta}\left[e^{x^\TT A\zeta}\bigg|\Fc_p\right]& = \mathbb{E}_{x}\left[\mathbb{E}_{\zeta}\left[e^{x^\TT A\zeta}|x\right]\bigg|\Fc_p\right]\\
& \leq \mathbb{E}\left[e^{\|A^\TT x\|^2/2}\bigg|\Fc_p\right]\label{eqn:subG}\\
& = \int \sqrt{\frac{1}{(4\pi\lambda)^n}} e^{\frac{1}{2}x^\TT AA^\TT x - \frac{1}{4\lambda}x^\TT x} dx\\
&=\sqrt{\frac{1 }{(2\lambda)^n\det(\frac{1}{2\lambda}I_n - AA^\TT)}}\\
& = \sqrt{\frac{1}{\det(I_n-2\lambda AA^\TT)}},\label{eqn: E(x^2)}
\end{align}
where (\ref{eqn:subG}) is due to the property of sub-Gaussian random vectors. 
The result then follows by combining Eqn.~(\ref{eqn: E(zeta^2)}) and Eqn.~(\ref{eqn: E(x^2)}).
\end{proof}

We note that similar but slightly more complicated versions of the inequalities for standard Gaussian vectors and sub-Gaussian vectors have been shown in \cite{Laurent-Massart} and \cite{hsu2011tail}, respectively. 

The following result is a simplified version of the classical self-normalized bound (Theorem 1 in \cite{Abbasi:2011:IAL}) with one-dimensional random variables. We recover the proof for completeness.
\begin{lemma}\label{lemma: Laplace}
Suppose $\sigma_q$ is $\Fc_q$-measurable, and $\{X_{q}\}_{q\geq 1}$ is an $\Fc$-adapted $\sigma_q$-sub-Gaussian random variable, i.e., $\Eb[\exp(\lambda X_q)|\Fc_q]\leq \exp(\lambda^2\sigma^2_q/2)$. Then, for any $\sigma,\delta>0$:
$$\Pb\left[\exists p \text{ s.t. } \left|\sum_{q=1}^p X_q\right|\geq \sqrt{\left(\sigma^2+\sum_{q=1}^p\sigma_q^2\right)\log \frac{\sigma^2+\sum_{q=1}^p\sigma_q^2}{\sigma^2\delta^2}}\right] \leq \delta.$$
\end{lemma}

\begin{proof}
Let $M_p(x) := \exp\left(x\sum_{q=1}^p X_q - \frac{1}{2}x^2\sum_{q=1}^p\sigma_q^2\right)$. We can verify that $\Eb[M_p(x)|\Fc_p] \leq M_{p-1}(x)$, and $\bar{M}_{p} := \int \frac{\exp(-\sigma^2x^2/2)}{\sqrt{2\pi/\sigma^2}}M_p(x)dx$ is a super-martingale. Besides,
\begin{align*}
    \bar{M}_p &= \int \exp\left(\frac{(\sum_{q=1}^pX_q)^2}{2\sigma^2+2\sum_{q=1}^p\sigma_q^2}\right)
    \exp\left(-\frac{\sigma^2+\sum_{q=1}^p\sigma_q^2}{2}\left(x - \frac{\sum_{q=1}^p X_q}{\sigma^2+\sum_{q=1}^p\sigma_q^2}\right)^2\right)\frac{\sigma dx}{\sqrt{2\pi}}\\
    &=\sqrt{\frac{\sigma^2}{\sigma^2+\sum_{q=1}^p\sigma_q^2}}\exp\left(\frac{(\sum_{q=1}^pX_q)^2}{2\sigma^2+2\sum_{q=1}^p\sigma_q^2}\right).
\end{align*}
By the maximal inequality (Theorem 3.9 in \cite{LS19bandit-book}), and the fact that $\Eb[\bar{M}_1] \leq 1$, we have
\[\Pb\left[\exists p \text{ s.t. } \sqrt{\frac{\sigma^2}{\sigma^2+\sum_{q=1}^p\sigma_q^2}}\exp\left(\frac{(\sum_{q=1}^pX_q)^2}{2\sigma^2+2\sum_{q=1}^p\sigma_q^2}\right)\geq \frac{1}{\delta} \right] = \Pb\left[\max_p \bar{M}_p\geq \frac{1}{\delta}\right]\leq \delta\Eb[\bar{M}_1]\leq \delta.\]
\end{proof}

\section{Generalized G-optimal Design}\label{appx:G}
\subsection{Analysis of the Multi-client G-optimal Design}\label{appx:G-optimal}
After extending the techniques in \citep{LS19bandit-book} to the multi-constraint and pseudo-determinant case, we establish an important property of the optimization problems (\ref{eqn: G_opt}) and (\ref{eqn: F_opt}), as stated in the following lemma. 
\begin{lemma}\label{lemma: OptimizationEquivalence}
Given sets $\Ac_i^p\subset[K]$ and $\Rc_a^p\subset[M]$ under \texttt{Fed-PE} in phase $p$, consider the following optimization problem:
\begin{align}\label{eqn: F_opt}
\text{maximize} \quad&F(\pi) = \sum_{a\in[K]}\log \Det\left(\sum_{j\in \Rc_a^p} \pi_{j,a}^p e_{j,a}e_{j,a}^\TT\right) \quad\text{s.t. } \pi^p\in\Cc^p.
\end{align}
Denote $d_a^p := \rank(\{e_{i,a}\}_{i\in\Rc_a^p}),\forall a\in[K]$. Then, we have the following equivalent statements:
\begin{itemize}[leftmargin=18pt,topsep=0pt, itemsep=0pt,parsep=0pt] 
\item[1)] $\pi^*$ is a maximizer of $F(\pi)$.
\item[2)] $\pi^*$ is a minimizer of $G(\pi)$ defined in Eqn.~\eqref{eqn: G_opt}.
\item[3)] $G(\pi^*) = \sum_{a\in[K]} d_a^p$.
\end{itemize}
\end{lemma}

\begin{remark}
Lemma~\ref{lemma: OptimizationEquivalence} is surprisingly similar to the original equivalence between G-optimal design and D-optimal design introduced in Theorem 21.1 of \cite{LS19bandit-book}, which was first studied in \cite{kiefer_wolfowitz_1960}. 
It provides an alternative method to solve (\ref{eqn: G_opt}) by maximizing $F(\pi)$. We note that a determinant maximization problem similar to (\ref{eqn: F_opt}) with two linear constraints has been studied in \cite{D_optimal_design_additional_cost}, while a generalized experimental design similar to (\ref{eqn: G_opt}) has been investigated in \cite{harman2014multi_constraints_d_design}. However, those works do not establish the equivalence between the two problems. Although an algorithm is proposed in \cite{harman2014multi_constraints_d_design} to solve the generalized experimental design problem, it does not provide any convergence guarantees. 

The generalized multi-client G-optimal design problem studied in this work can be treated as a particularized version of the generalized experimental design in \cite{harman2014multi_constraints_d_design}. As we will show below, the pseudo-determinant in the objective function in (\ref{eqn: F_opt}) can be viewed as the determinant of a block diagonal matrix, and the client-wise constraints essentially are imposed on individual blocks. Such special structure enables us to establish the equivalence between the generalized multi-agent G-optimal design in (\ref{eqn: G_opt}) and the optimization problem in (\ref{eqn: F_opt}), as well as to design a novel block coordinate ascent method to solve (\ref{eqn: F_opt}) efficiently without violating the constraints.
\end{remark}

In order to prove Lemma~\ref{lemma: OptimizationEquivalence}, we prove a generalized version of it instead. Before we proceed, we introduce the following notations.
First, we omit the phase index $p$ in this subsection without causing any ambiguity. Next, we embed each normalized feature vector $e_{i,a}:=\nicefrac{x_{i,a}}{\|x_{i,a}\|}$ into a higher dimensional space as follows: we construct a vector $v_{i,a}\in\mathbb{R}^{dK}$ as:
\[v_{i,a} = (0,\ldots,0,\underbrace{e_{i,a}^\TT}_{a\text{-th block}},0,\ldots,0)^\TT\in\mathbb{R}^{dK},\]
i.e., the $[(a-1)d+k]$-th coordinate of $v_{i,a}$ is the $k$-th coordinate of $e_{i,a}$ for all $k\in[d]$ and all other coordinates of $v_{i,a}$ are 0.

We can verify that 
\[e_{i,a}^\TT\left(\sum_{j\in\Rc_a}\pi_{j,a} e_{j,a}e_{j,a}^\TT\right)^\dagger e_{i,a} = v_{i,a}^\TT\left(\sum_{a\in[K]}\sum_{j\in\Rc_a}\pi_{j,a}v_{j,a}v_{j,a}^\TT\right)^\dagger v_{i,a},\]
and
\[\prod_{a\in[K]} \Det\left(\sum_{j\in\Rc_a}\pi_{j,a} e_{j,a}e_{j,a}^\TT\right) = \Det\left(\sum_{a\in[K]}\sum_{j\in\Rc_a}\pi_{j,a}v_{j,a}v_{j,a}^\TT\right),\]
due to the fact that $\sum_{a\in[K]}\sum_{j\in\Rc_a}\pi_{j,a}v_{j,a}v_{j,a}^\TT$ is a block diagonal matrix, and the $a$-th diagonal block is exactly the matrix $\sum_{j\in\Rc_a}\pi_{j,a} e_{j,a}e_{j,a}^\TT$. Meanwhile, the following relation also holds:
\[\sum_{a\in[K]}\rank\left(\sum_{j\in\Rc_a}\pi_{j,a} e_{j,a}e_{j,a}^\TT\right) = \rank\left(\sum_{a\in[K]}\sum_{j\in\Rc_a}\pi_{j,a}v_{j,a}v_{j,a}^\TT\right).\]

With those notations, we present a generalized version of Lemma \ref{lemma: OptimizationEquivalence} below. It is a ``generalized version'' in the sense that, {the vectors $\{v_{i,a}\}_{i,a}$ involved in Lemma~\ref{lemma: General_equivalence} can be arbitrary real vectors, which allows more versatility in the formulation. Besides,} in Lemma \ref{lemma: General_equivalence}, $\sum_{a\in\Ac_i}\pi_{i,a}$ may be different for different client $i\in [M]$, which implies that the clients can work ``asynchronously'' instead of ``synchronously'' as under \texttt{Fed-PE}. Note that we change the order of the double summation, which removes $\Rc_a$ from the expressions.

\begin{lemma}[Generalized version of Lemma \ref{lemma: OptimizationEquivalence}]\label{lemma: General_equivalence} 
Given any  $v_{i,a}\in\mathbb{R}^{n}$, where $i\in[M], a\in\Ac_i$, and $\{\Ac_i\}$ are index sets. Consider the following optimization problems \eqref{eqn: general_F_opt} and \eqref{eqn: general_G_opt}:
\begin{align}
&\left\{\begin{array}{lc}
\text{maximize\ \  } \bar{F}(\pi) = \log \Det\left(\sum_{j\in [M]} \sum_{a\in\Ac_j}\pi_{j,a}v_{j,a}v_{j,a}^\TT\right),\\
\text{subject to }  \pi\in\bar{\Cc} .
\end{array}\right.\label{eqn: general_F_opt}\\
&\left\{\begin{array}{lc}
\text{minimize\ \ } \bar{G}(\pi) = \sum_{i = 1}^Mf_i\max_{a\in \Ac_i}  v_{i,a}^\TT\left(\sum_{j\in [M]}\sum_{a\in\Ac_j} \pi_{j,a}v_{j,a}v_{j,a}^\TT\right)^{\dagger}v_{i,a},\\
\text{subject to } \pi\in\bar{\Cc}. \end{array}\right.\label{eqn: general_G_opt}
\end{align}
where the feasible set $\bar{\Cc}\subset\Rb^{\sum_{i\in[M]}|\Ac_i|}$ contains all $\pi$ that satisfy: 
\begin{align}\label{eqn: general_feasible set}
\bar{\Cc} =\left\{ \pi \middle | \begin{array}{l}
 \pi_{i,a}\geq 0,\forall i\in [M], a\in \Ac_i\\ \sum_{a\in \Ac_i}\pi_{i,a} = f_i, \forall i\in[M],  \\
\rank(\{\pi_{i,a}v_{i,a}\}_{i\in[M],a\in\Ac_i}) = \rank(\{v_{i,a}\}_{i\in[M],a\in\Ac_i})
\end{array}\right\}.\end{align}
Then, we have the following equivalent statements:
\begin{itemize}[leftmargin=18pt,topsep=0pt, itemsep=0pt,parsep=0pt] 
    \item [1)] $\pi^*$ is a maximizer of (\ref{eqn: general_F_opt}).
    \item [2)] $\pi^*$ is a minimizer of (\ref{eqn: general_G_opt}).
    \item [3)] $\bar{G}(\pi^*) = \rank(\{v_{i,a}\}_{i\in[M],a\in\Ac_i}):= D$.
\end{itemize}
\end{lemma}

\begin{proof}
$1)\Longrightarrow2)$: Since $\bar{\Cc}$ is a convex set, and its image under a linear transformation is still convex, $\bar{F}(\pi)$ is a concave function over $\bar{\Cc}$ according to Lemma \ref{lemma: concave}. Using the first-order optimality, we have
\[
    \left<\nabla \bar{F},\pi-\pi^*\right>\leq 0
\]
holds for {any $\pi\in\bar{\Cc}$}. 
Therefore, based on Lemma \ref{lemma: logdet}, we have
\begin{align}
    0&\geq \sum_{i\in[M]}\sum_{a\in\Ac_i}(\pi_{i,a}-\pi^*_{i,a})v_{i,a}^\TT\bigg(\sum_{j\in [M]}\sum_{a\in\Ac_j} \pi^*_{j,a}v_{j,a}v_{j,a}^\TT\bigg)^{\dagger}v_{i,a}\\
    &=\sum_{i\in[M]}\sum_{a\in\Ac_i}\pi_{i,a}v_{i,a}^\TT\bigg(\sum_{j\in [M]}\sum_{a\in\Ac_j} \pi^*_{j,a}v_{j,a}v_{j,a}^\TT\bigg)^{\dagger}v_{i,a}\nonumber\\
    & \quad- 
     \trace\bigg(\sum_{i\in[M]}\sum_{a\in\Ac_i}\pi^*_{i,a}v_{i,a}v_{i,a}^\TT \bigg(\sum_{j\in [M]}\sum_{a\in\Ac_j} \pi^*_{j,a}v_{j,a}v_{j,a}^\TT\bigg)^{\dagger}\bigg)\\
    & = \sum_{i\in[M]}\sum_{a\in\Ac_i}\pi_{i,a}v_{i,a}^\TT\bigg(\sum_{j\in [M]}\sum_{a\in\Ac_j} \pi^*_{j,a}v_{j,a}v_{j,a}^\TT\bigg)^{\dagger}v_{i,a} -  D.\label{eqn: preFirstOrderOptimality1}
\end{align}

Let $a_i = \arg\max_{b}v_{i,b}^\TT\left(\sum_{j\in \Rc_b} \pi^*_{j,b}v_{j,b}v_{j,b}^\TT\right)^{\dagger}v_{i,b}$, and set $\hat{\pi}_{i,a_i} = f_i$ and $\hat{\pi}_{i,b}=0$ for any $b\neq a_i$.  Then, define a sequence $\{\pi^m\}_{m=1}^{\infty}\in\bar{\Cc}$ with $\lim_{m\rightarrow\infty}\pi_m=\hat{\pi}$. Substituting $\pi^m$ into Eqn.~(\ref{eqn: preFirstOrderOptimality1}) and taking the limit of $m$, we have
\begin{align}
& \lim_{m\rightarrow\infty}\sum_{i\in[M]}\sum_{a\in\Ac_i}\pi^m_{i,a}v_{i,a}^\TT\bigg(\sum_{j\in [M]}\sum_{a\in\Ac_j} \pi^*_{j,a}v_{j,a}v_{j,a}^\TT\bigg)^{\dagger}v_{i,a} -  D\\
&=\sum_{i\in[M]}\sum_{a\in\Ac_i}\hat{\pi}_{i,a}v_{i,a}^\TT\bigg(\sum_{j\in [M]}\sum_{a\in\Ac_j} \pi^*_{j,a}v_{j,a}v_{j,a}^\TT\bigg)^{\dagger}v_{i,a} -  D\\
&=\sum_{i\in[M]}\hat{\pi}_{i,a_i} \max_{a\in\Ac_i}v_{i,a}^\TT\bigg(\sum_{j\in [M]}\sum_{a\in\Ac_j} \pi^*_{j,a}v_{j,a}v_{j,a}^\TT\bigg)^{\dagger}v_{i,a} -  D\leq 0,
\end{align}
i.e.,
\begin{equation}\label{eqn: FirstOrderOptimality1}
\sum_{i\in[M]}f_i\max_{a\in\Ac_i}v_{i,a}^\TT\bigg(\sum_{j\in [M]}\sum_{a\in\Ac_j} \pi^*_{j,a}v_{j,a}v_{j,a}^\TT\bigg)^{\dagger}v_{i,a}\leq D.
\end{equation}

On the other hand, for any feasible point $\pi$, we have
\begin{align}
    D &=\trace\bigg(\sum_{i\in[M]}\sum_{a\in\Ac_i}\pi_{i,a}v_{i,a}v_{i,a}^\TT \bigg(\sum_{j\in [M]}\sum_{a\in\Ac_j} \pi_{j,a}v_{j,a}v_{j,a}^\TT\bigg)^{\dagger}\bigg)\\
    &=\sum_{i\in[M]}\sum_{a\in\Ac_i}\pi_{i,a}v_{i,a}^\TT\bigg(\sum_{j\in [M]}\sum_{a\in\Ac_j} \pi_{j,a}v_{j,a}v_{j,a}^\TT\bigg)^{\dagger}v_{i,a}\\
    &\leq\sum_{i\in[M]}f_i\max_{a\in\Ac_i}v_{i,a}^\TT\bigg(\sum_{j\in [M]}\sum_{a\in\Ac_j} \pi_{j,a}v_{j,a}v_{j,a}^\TT\bigg)^{\dagger}v_{i,a}.\label{eqn: FirstOrderOptimality2}
\end{align}

Combining Eqns. (\ref{eqn: FirstOrderOptimality1}) and (\ref{eqn: FirstOrderOptimality2}), we conclude that  $\pi^*$ is also a minimizer of $\bar{G}(\pi)$, and $\bar{G}(\pi^*) = D$.

$2)\Longrightarrow3)$ This can be seen from the above argument.

$3)\Longrightarrow1)$ For any feasible point $\pi\in \bar{\Cc}$, we have:
\begin{align}
    &\left<\nabla \bar{F},\pi-\pi^*\right>\\
    &=\sum_{i\in[M]}\sum_{a\in\Ac_i}\pi_{i,a}v_{i,a}^\TT\bigg(\sum_{j\in [M]}\sum_{a\in\Ac_j} \pi^*_{j,a}v_{j,a}v_{j,a}^\TT\bigg)^{+}v_{i,a} - D\\
    &\leq \bar{G}(\pi^*) - D = 0.
\end{align}
Thus, based on the concavity of $\bar{F}$, we conclude that $\pi^*$ is a maximizer of $\bar{F}(\pi)$.
\end{proof}

\subsection{Block Coordinate Ascent for Generalized G-optimal Design}\label{appx:CoordinateAscent}
In this subsection, we provide an efficient block coordinate ascent algorithm to solve the optimization problem in (\ref{eqn: F_opt}). We keep the same notations as in Appendix~\ref{appx:G-optimal}, e.g., omitting the phase index $p$.

The block coordinate ascent algorithm inherits the idea from \citep{MinimumVolumnEllipsoid}, which leverages the low-rank updating formula (Corollary A.10 in \cite{MinimumVolumnEllipsoid}). While only nonsingular matrices are considered in \cite{MinimumVolumnEllipsoid}, we extend it to the case where pseudo-inverse and pseudo-determinant are involved. Such extension is necessary, because under \texttt{Fed-PE}, $\rank(\{x_{i,a}\}_{i\in \Rc^p_a})$ is decreasing as phase $p$ progresses in general, making singular matrices unavoidable. As elaborated in this subsection, such extension is technically non-trivial.

\begin{lemma}\label{lemma:low_rank_update}
Assume $A\in\mathbb{R}^{n\times n}$ is a positive semi-definite matrix and vector $u\in \range(A)$. Let $\lambda\in\mathbb{R}$ such that $\range(A+\lambda uu^T) = \range(A)$. Then, $\Det_s(A+\lambda uu^\TT) = (1+\lambda u^\TT A^\dagger u)\Det_s(A)$ holds for any $s\in\{0,1,\ldots,n\}$. Moreover, 
\[(A+\lambda uu^\TT)^\dagger = A^\dagger - \frac{\lambda A^\dagger uu^\TT A^\dagger}{1+\lambda u^\TT A^\dagger u}. \]
\end{lemma}

\begin{proof}
Corollary A.10 in \citep{MinimumVolumnEllipsoid} states that for nonsingular matrix $B$ and any vector $u$ and $\lambda$ such that $B+\lambda uu^\TT$ is also nonsingular, the following updating rules hold:
\begin{align}\label{eqn:update1}
    \det(B+\lambda uu^\TT) &= (1+\lambda u^\TT B^{-1} u)\det(B),\\
(B+\lambda uu^\TT)^{-1} & = B^{-1} - \frac{\lambda B^{-1} uu^\TT B^{-1}}{1+\lambda u^\TT B^{-1} u}.\label{eqn:update2}
\end{align}

For a general positive semi-definite matrix $A\in\mathbb{R}^{n\times n}$, we let $B:=A+\epsilon I_n$, where $\epsilon >0$ is an arbitrarily small number, and insert it in Eqn. (\ref{eqn:update1}). We have
\begin{align}
     \det(A+\epsilon I_n+\lambda uu^\TT) =(1+\lambda u^\TT (A+\epsilon I_n)^{-1} u)\det(A+\epsilon I_n).
\end{align}
 Dividing both sides by $\epsilon^{n-s}$ and letting $\epsilon$ go to 0, we have
\begin{align}
   \lim_{\epsilon\rightarrow0} \frac{ \det(A+\epsilon I_n+\lambda uu^\TT)}{\epsilon^{n-s}} =   \lim_{\epsilon\rightarrow0} \frac{(1+\lambda u^\TT (A+\epsilon I_n)^{-1} u)\det(A+\epsilon I_n)}{\epsilon^{n-s}}.
\end{align}
Given the fact that $u\in \range(A+\lambda uu^\TT) = \range(A)$ and Proposition~\ref{prop: limit_relations_pseudo}, we then have
\begin{align}\label{eqn:low_rank_update1}
    \Det_s(A+\lambda uu^\TT) = (1+\lambda u^\TT A^\dagger u)\Det_s(A).
\end{align}

In order to obtain the second part of Lemma~\ref{lemma:low_rank_update}, we first have 
\begin{align}\label{eqn:update3}
(A+\lambda uu^T)^\dagger = \lim_{\epsilon\rightarrow0}(A+\epsilon I_n + \lambda uu^\TT)^{-1}(A+\lambda uu^\TT)(A+\epsilon I_n+\lambda uu^\TT)^{-1}
\end{align}
based on Proposition~\ref{prop: limit_relations_pseudo}. 
Then, we use (\ref{eqn:update2}) to obtain
\begin{align}\label{eqn:update4}
    (A+\epsilon I_n + \lambda uu^\TT)^{-1}=(A+\epsilon I_n)^{-1}-\frac{\lambda(A+\epsilon I_n)^{-1}uu^\TT(A+\epsilon I_n)^{-1}}{1+\lambda u^\TT (A+\epsilon I_n)^{-1} u}.
\end{align}

Inserting (\ref{eqn:update2}) in (\ref{eqn:update3}), using Proposition \ref{prop: limit_relations_pseudo} to remove the limit, and manipulating the expression with the associative property of matrix multiplication, we finally have
\begin{align}
    (A+\lambda uu^\TT)^\dagger 
    &=A^\dagger -\frac{\lambda A^\dagger uu^\TT A^\dagger}{1+\lambda u^\TT A^\dagger u}.\label{eqn:low_rank_update2}
\end{align}
\end{proof}

The next two lemmas will be utilized to demonstrate that the coordinate ascent algorithm (Algorithm \ref{alg: BCM}) described below does not violate the ``rank-preserving'' condition defined in Eqn.~(\ref{eqn: feasible set}).

\begin{lemma}\label{lemma:rank}
Let  $A :=\sum_{s=1}^d \lambda_s u_su_s^\TT$ with $u_s\in\Rb^n$, $\lambda_s>0$ for $s\in[d]$. Then, $\range(A)=\Span(\{u_s\}_{s\in[d]})$, and $\rank(A)=\rank(A^\dagger)=\rank(\{\lambda_s u_s\}_{s\in[d]})=\rank(\{u_s\}_{s\in[d]})$.
\end{lemma}
\begin{proof}
It suffices to show the column space of $A$ is the same as the space spanned by $\{u_s,s\in[d]\}$. 
Note that 
\[\range(A) = \{v\in\Rb^n | v = Ac, c\in\Rb^n\},\] 
and 
\[\Span(\{u_s\}_{s\in[d]}) = \bigg\{v\in\Rb^n \bigg| v = \sum_{s\in[d]}u_s c_s, c_s\in\Rb\bigg\}.\]
Then, for any $v\in\range(A)$, there must exist a $c\in\Rb^n$ such that $v = Ac = \sum_{s\in[d]}\lambda_s u_s u_s^\TT c = \sum_{s\in[d]}(\lambda_s u_s^\TT c) u_s\in\Span(\{u_s\}_{s\in[d]}).$
Thus, we have $\range(A)\subseteq \Span(\{u_s\}_{s\in[d]}).$

On the other hand, if $\range(A)\neq \Span(\{u_s\}_{s\in[d]})$, there must exist a $v\in \Span(\{u_s\}_{s\in[d]})$, $v\neq 0$, such that $Av=0$. Thus, we have $v^\TT Av=0$, which implies that
\begin{align}
    v^\TT \left(\sum_{s=1}^d \lambda_s u_su_s^\TT\right)v=0\Leftrightarrow \sum_{s=1}^d \lambda_s (v^\TT u_s)^2=0.
\end{align}
Since $\lambda_s>0$ for $s\in[d]$, we must have $v^\TT u_s=0$ for all $s\in[d]$. Meanwhile, since $v\in\Span(\{u_s\}_{s\in[d]})$, it indicates that $v=0$, which contradicts with the assumption that $v\neq 0$.
\end{proof}

\begin{lemma}\label{lemma: linIndpdnt}
Assume $A \in\mathbb{R}^{n\times n}$ is a singular positive semi-definite matrix with $d=\rank({A}) < n$. Let $u_0\notin \range(A)$. Then, for any $\lambda_0>0$, $u_0^\TT(\lambda_0 u_0u_0^\TT + A)^\dagger u_0 = 1/\lambda_0$.
\end{lemma}

\begin{proof}
First, we decompose $A$ into $A = \sum_{s=1}^d \lambda_s u_su_s^\TT$, where $\{u_s\}_s$ are orthonormal eigenvectors of $A$, and $\{\lambda_s>0\}_s$ are non-zero eigenvalues of $A$. Thus, $u_0\notin \range(A)$ indicates that $u_0, u_1,\ldots, u_d$ are linearly independent.

Since $X\succeq Y$ is equivalent to $ X^\dagger \preceq Y^\dagger$, we have
\begin{equation}\label{eqn: uVu<1}
u_t^\TT\bigg(\sum_{s=0}^d \lambda_su_su_s^\TT\bigg)^\dagger u_t\leq \frac{u_t^\TT(u_tu_t^\TT)^\dagger u_t}{\lambda_t} = \frac{1}{\lambda_t}
\end{equation}
holds for all $t\in\{0,1,\ldots,d\}$.

Multiplying both sides of (\ref{eqn: uVu<1}) by $\lambda_t$ and summing over $t$, we have
\[d+1\geq \sum_{t=0}^d \lambda_tu_t^\TT\bigg(\sum_{s=0}^d \lambda_s u_su_s^\TT\bigg)^\dagger u_t = \trace\bigg(\bigg(\sum_{s=0}^d \lambda_s u_su_s^\TT\bigg)^\dagger\bigg(\sum_{t=0}^d \lambda_t u_tu_t^\TT\bigg)\bigg) = d+1.\]
Thus, Eqn.~$(\ref{eqn: uVu<1})$ must hold with equality for all $t\in\{0,1\ldots,d\}$. 
In particular, $u_0^\TT(\lambda_0 u_0u_0^\TT + A)^\dagger u_0 = 1/\lambda_0.$ 
\end{proof}

\begin{algorithm}
\caption{Block Coordinate Ascent (BCA)}
\begin{algorithmic}[1]
\State \textbf{Input:} $\{\Ac_i\}_i,\{\Rc_a\}_a, \{e_{i,a}\}_{i,a}, {\epsilon}$. %\jingc{what is $N$?}
\State For each $i {\in[M]},a {\in \Ac_i}$: $\pi_{i,a}\leftarrow \frac{1}{|\Ac_i|}$.
\State For each $a\in[K]$: {$\tilde{V}_a\leftarrow \left(\sum_{i\in\Rc_a} \pi_{i,a}e_{i,a}e_{i,a}^\TT\right)^\dagger$}, $d_a \leftarrow \rank(\tilde{V}_a)$.
\While{$G(\pi)> \sum_{a\in[K]}d_a+\epsilon$ } 
    \For{$i\in[M]$}
    solve the following optimization problem: 
    \begin{align}\label{eqn:bca}
    %(\omega_a)_{a\in\Ac_i} &= \arg
    \max_{\{\omega_{a}\}_{a\in\Ac_i}} & \sum_{a\in\Ac_i}\log(1+\omega_a e_{i,a}^\TT \tilde{V}_a e_{i,a}),\quad \text{s.t.} \sum_{a\in\Ac_i}\omega_a = 0 \text{ and } -\pi_{i,a}\leq \omega_a \leq 1-\pi_{i,a}.
    \end{align}
    \For{$a\in\Ac_i$} 
    \[\pi_{i,a}\leftarrow\pi_{i,a} + \omega_a,\quad \tilde{V}_a \leftarrow \tilde{V}_a - \frac{\omega_a \tilde{V}_a e_{i,a}e_{i,a}^\TT \tilde{V}_a}{1+\omega_a e_{i,a}^\TT \tilde{V}_a e_{i,a}}.\]
    \EndFor
    \EndFor
\EndWhile
\end{algorithmic}\label{alg: BCM}
\end{algorithm}

We are now ready to introduce the Block Coordinate Ascent (\texttt{BCA}) algorithm. As depicted in Algorithm~\ref{alg: BCM}, at each phase $p$, \bca aims to obtain a distribution $\pi$ by changing $\pi_i:=\{\pi_{i,a}\}_{a\in\Ac_i}$ for a client $i$ while fixing the distributions for all the other clients $j\neq i$ in each step.

Consider that $\pi$ is changed to $\pi^{+}$ by redistributing $\pi_i$ to $\pi_i^+$. Thus, $\pi_i^+ = \pi_i + \omega$, where $\sum_{a\in\Ac_i}\omega_{a} = 0$, and $-\pi_{i,a}\leq\omega
_a\leq 1-\pi_{i,a}$.  

Recall that
\[F(\pi)  = \sum_{a\in[K]}\log\Det_{d_a}\bigg(\sum_{j\in\Rc_a}\pi_{j,a}e_{j,a}e_{j,a}^\TT\bigg).\]
To simplify the notation, let $U_a := \sum_{j\in\Rc_a}\pi_{j,a}e_{j,a}e_{j,a}^\TT$.
Then, to maximize $F(\pi^{+})$ is equivalent to maximizing the following:
\begin{align}
F(\pi^{+}) - F(\pi) &= \sum_{a\in\Ac_i}\log\frac{\Det_{d_a}\left(U_a + \omega_ae_{i,a}e_{i,a}^\TT\right)}{\Det_{d_a}\left(U_a\right)}\\
& = \sum_{a\in\Ac_i}\log\left(1+\omega_a e_{i,a}^\TT U_a^\dagger e_{i,a}\right),\label{eqn:update5}
\end{align}
where Eqn.~(\ref{eqn:update5}) follows from the low-rank updating formula (\ref{eqn:low_rank_update1}) in Lemma~\ref{lemma:low_rank_update} when $\rank(U_a + \omega_ae_{i,a}e_{i,a}^\TT)=\rank(U_a)$.

After obtaining $\{\omega_a\}_{a\in\Ac_i}$, the distribution $\pi_i$ will be updated; Correspondingly, all involved $\{U_a^\dagger\}_{a\in\Ac_i}$ in the objective function (\ref{eqn:update5}) shall be updated as well. The low-rank updating formula (\ref{eqn:low_rank_update2}) from Lemma~\ref{lemma:low_rank_update} is then invoked to perform the updating efficiently. After that, \bca proceeds to the next block and repeats the procedure.
 
\begin{Proposition}
 The distribution vector $\pi$ after each block updating under \bca is always in $\Cc$.
\end{Proposition}
\begin{proof}
 We prove this through induction. First, it is obvious that the initialization  $\pi_{i,a}=\frac{1}{|\Ac_i|}$, for $i\in[M]$, $a\in\Ac_i$ ensures that $\pi_i$ is a valid distribution for every $i\in[M]$. Meanwhile, since  $\pi_{i,a}=\frac{1}{|\Ac_i|}$ for $i\in[M]$, $a\in\Ac_i$, according to Lemma~\ref{lemma:rank}, we have $\rank(U_a)=\rank(\{e_{j,a}\}_{j\in\Rc_a})$. Thus, the rank-preserving condition is satisfied as well. 
 
 Next, we use $\pi$ and $\pi^+$ to denote the distribution before and after one block updating on $\pi_i$, respectively. Assume $\pi\in\Cc$, and we aim to show that $\pi^+\in \Cc$ as well. It is straightforward to see that the conditions $\sum_{a\in\Ac_i}\omega_a = 0$ and $-\pi_{i,a}\leq \omega_a \leq 1-\pi_{i,a}$ ensure that $\pi_i^+$ is still a valid distribution. It thus suffices to show that the rank-preserving condition is satisfied for $\pi^+$. We prove this through contradiction. Assume after the updating, $\rank(U_a^+)<\rank(U_a)$. 
 We note that 
 \begin{align}
 U_a^+&=U_a + \omega_a e_{i,a}e_{i,a}^\TT =(\pi_{i,a}+\omega_a)e_{i,a}e_{i,a}^\TT+ \sum_{j\in\Rc_a\backslash i}\pi_{j,a}e_{j,a}e_{j,a}^\TT.
 \end{align}
 According to Lemma~\ref{lemma:rank}, if $\pi_{i,a}+\omega_a>0$, we must have $\rank(U_a^+)=\rank(U_a)$. Thus, if $\rank(U_a^+)<\rank(U_a)$, we must have $\pi_{i,a}+\omega_a=0$, and $e_{i,a}\notin\Span(\{e_{j,b}\}_{j\in\Rc_a\backslash\{i\}})$. Then, according to Lemma \ref{lemma: linIndpdnt}, we must have $e_{i,a}^\TT U_a^{\dagger} e_{i,a} = 1/\pi_{i,a}$. Correspondingly, Eqn.~(\ref{eqn:update5}) becomes $-\infty$. Since we can always perturb $\omega$ to make $\pi^+_{i,a}\neq 0$ for every $a\in\Ac_i$, there exist feasible solutions making (\ref{eqn:update5}) greater than $-\infty$. Therefore, the distribution that makes $\rank(U_a^+)<\rank(U_a)$ cannot be the solution to (\ref{eqn:bca}) in the \bca algorithm. This indicates that the solution of \bca in each phase satisfies the rank-preserving condition.
\end{proof}

The convergence of \bca can be obtained by leveraging the result for the block coordinate minimization (BCM) algorithm introduced in \citep{Optimizationmodels}. Given a function $f_0(x_1,\ldots,x_v)$ where $x_i\in\mathcal{X}_i$, $i=1,\ldots,v$, the BCM algorithm generates $x^{(k)} := (x_1^{(k)},\ldots,x_v^{(k)})$ as follows:
\[x_{i}^{(k+1)} = \arg\min_{y\in\mathcal{X}_i}
f_0\left (x_1^{(k+1)},\ldots,x_{i-1}^{(k+1)}, y, x_{i+1}^{(k)},\ldots,x_v^{(k)} \right ).\] Then, the following theorem guarantees the convergence of the BCM algorithm.

\begin{theorem}[Theorem 12.4 in \cite{Optimizationmodels}]\label{thm:convergence}
Assume $f_0(x_1,\ldots,x_v)$ is convex and continuously differentiable on the feasible set $\mathcal{X} := \mathcal{X}_1\otimes\ldots\otimes\mathcal{X}_v$. Moreover, let $f_0$ be strictly convex in $x_i$ when the other variable blocks $x_j, j\neq i$, are held constant. If the sequence $\{x^{(k)}\}$ generated by the BCM algorithm is {well-defined}, then every limit point of $\{x^{(k)}\}$ converges to an optimal solution of the optimization problem $\min_{x\in\mathcal{X} } f_0(x)$. 
\end{theorem}
Theorem~\ref{thm:convergence} assures that every sequence $\{\pi^{(k)}\}$ generated by \bca converges to an optimal solution of (\ref{eqn: F_opt}), since  $F(\pi^+) - F(\pi)$ is concave, continuously differentiable, and strictly concave in each $\pi_i$, $i\in[M]$ when the other $\pi_j$'s, $j\neq i$, are fixed.

\section{Proof of Theorem~\ref{thm: BoundAlgorithm1}}\label{sec: ProofAlgorithm1}

We first present the complete version of Theorem~\ref{thm: BoundAlgorithm1}. 
\begin{theorem}[Complete version of Theorem~\ref{thm: BoundAlgorithm1}]\label{thm: complete_thm_Fed-PE}
Consider time horizon $T$ that consists of $H$ phases with $f^p= cn^p$, where $c$ and $n>1$ are fixed integers, and $n^p$ denotes the $p$th power of $n$.
Let 
\begin{equation}\label{eqn: alpha}
\alpha=\min\left\{\sqrt{2\log (KH/\delta)+d\log(ke)},\sqrt{2\log(2MKH/\delta)}\right\},
\end{equation}
where $k>1$ is a number satisfying $kd\geq 2\log(KH/\delta)+d\log(ke)$.  
Then, with probability at least $1-\delta$, the regret under \alg is upper bounded as 
$$R(T)\leq 4\alpha\frac{L}{\ell}\sqrt{dKM}\left(\frac{\sqrt{n^2-n}}{\sqrt{n}-1}\sqrt{T} + \frac{K}{\sqrt{cn} - \sqrt{c}}\right).$$
Furthermore, by assuming $K=O(\sqrt{T})$, the cumulative regret scales as
$$O\left(\frac{L}{\ell}\sqrt{dKMT(\log(K(\log T)/\delta)+\min\{d,\log M\})}\right)$$ 
and the communication cost scales as $O(Md^2K\log T)$.
\end{theorem}

The proof of Theorem~\ref{thm: complete_thm_Fed-PE} relies on a ``good'' event that happens with high probability. We define the ``bad'' event as follows:
\begin{equation}\label{eqn: errorevent}
\Ec(\alpha) = \{\exists p\in[H], i\in[M], a\in\Ac_i^{p-1}, |\hat{r}_{i,a}^p - r_{i,a}|\geq u_{i,a}^p = \alpha\sigma_{i,a}^p \},
\end{equation}
where $\alpha$ is defined in (\ref{eqn: alpha}) barring explicit explanations, and $\sigma_{i,a}^p := \|x_{i,a}\|_{V_a^p}/\ell$. 
We call $\Ec^c(\alpha)$ the ``good'' event.

\subsection{Bound the Probability of the Bad Event}
Before we bound the probability of the bad event $\Ec(\alpha)$, let us first introduce some necessary notations and characterize the sub-Gaussianity of the locally estimated rewards.

Let $\mathcal{F}_p$ be the $\sigma$-algebra of the events happened {in and} before the arm elimination stage at phase $p$, i.e.
$\mathcal{F}_p = \sigma\{a_{i,t}\in[K],y_{i,t}\in\Rb|t\in \cup_{q=0}^{p-1}\cup_{a\in \Ac_i^q}\mathcal{T}_{i,a}^q, i\in[M]\}$, with $\Fc_0 = \emptyset$. Note that $\Ac_i^{p}$ and $\Rc_a^{p}$ are $\mathcal{F}_p$-measurable.

Note that $\{a\in\Ac_i^{p}\}$ is equivalent to $\{i\in\Rc_a^{p}\}$. Throughout the analysis, we frequently exchange the order of double summations $\sum_{i\in[M]}\sum_{a\in\Ac_i^{p}}$ and $\sum_{a\in[K]}\sum_{i\in\Rc_a^p}$. 

\begin{lemma}\label{lemma: BndOneErr}
At phase $p\in[H]$, for any client $i\in[M]$, arm $a\in \Ac_i^{p-1}$, $\hat{r}_{i,a}^p - r_{i,a}$ is a conditionally sub-Gaussian random variable, i.e. $\mathbb{E}[\exp(\lambda(\hat{r}_{i,a}^p - r_{i,a}))|\mathcal{F}_{p-1}]\leq\exp\big(\frac{\lambda^2(\sigma_{i,a}^p)^2}{2}\big)$.
\end{lemma}

\begin{proof}
Let $\xi_{i,a}^{p-1}$ be the sum of the independent sub-Gaussian noise incurred during the collaborative exploration step in phase $p-1$, i.e. $\xi_{i,a}^{p-1}:=\sum_{t\in \Tc^{p-1}_{i,a}} \eta_{i,t}$. Then, given $f_{i,a}^{p-1}$, $\xi_{i,a}^{p-1}$ is a conditionally $\sqrt{f_{i,a}^{p-1}}$-sub-Gaussian random variable. 

Recall the definition of local estimators, and we have
\begin{align}
    \hat{\theta}_{i,a}^{p-1} &= \bigg(\frac{1}{f_{i,a}^{p-1}}\sum_{t\in\Tc_{i,a}^{p-1}}y_{i,t}\bigg)\frac{x_{i,a}}{\|x_{i,a}\|^2}\\
    &=\bigg(x_{i,a}^\TT\theta_a + \frac{\xi_{i,a}^{p-1}}{f_{i,a}^{p-1}}\bigg)\frac{x_{i,a}}{\|x_{i,a}\|^2}\\
    & = \frac{x_{i,a}x_{i,a}^\TT}{\|x_{i,a}\|^2}\theta_a + \frac{x_{i,a}\xi_{i,a}^{p-1}}{f_{i,a}^{p-1}\|x_{i,a}\|^2}.
\end{align}

 Since $\hat{r}_{i,a}^p= x_{i,a}^\TT\hat{\theta}_a^p$, and $V_a^p = \left(\sum_{j\in\Rc_a^{p-1}}f_{j,a}^{p-1}\frac{x_{j,a}x_{j,a}^\TT}{\|x_{j,a}\|^2}\right)^\dagger$, we have 
\begin{align}
\hat{r}_{i,a}^p-r_{i,a}&=x_{i,a}^\TT\hat{\theta}_{a}^p - x_{i,a}^\TT\theta_a\\
& = x_{i,a}^\TT V_a^p\bigg(\sum_{j\in\Rc_a^{p-1}}f_{j,a}^{p-1}\hat{\theta}_{j,a}^{p-1}\bigg) - x_{i,a}^\TT\theta_a\\
& = x_{i,a}^\TT V_a^p\bigg(\sum_{j\in\Rc_a^{p-1}}f_{j,a}^{p-1}\frac{x_{j,a}x_{j,a}^\TT}{\|x_{j,a}\|^2}\theta_a + \sum_{j\in\Rc_a^{p-1}}\frac{x_{j,a}\xi_{j,a}^{p-1}}{\|x_{j,a}\|^2} \bigg) - x_{i,a}^\TT\theta_a\\
&=x_{i,a}V_a^p (V_a^p)^\dagger\theta_a + x_{i,a}^\TT V_a^p\bigg(\sum_{j\in\Rc_a^{p-1}}\frac{e_{j,a}}{\|x_{j,a}\|}\xi_{j,a}^{p-1}\bigg) - x_{i,a}^\TT\theta_a\label{eqn: preErrorTerm}\\
&= x_{i,a}^\TT V_a^p\bigg(\sum_{j\in\Rc_a^{p-1}}\frac{e_{j,a}}{\|x_{j,a}\|}\xi_{j,a}^{p-1}\bigg),\label{eqn: ErrorTerm}
\end{align}
where (\ref{eqn: preErrorTerm}) is due to the definition of $V_a^p$ in Eqn. (\ref{eqn: theta_hat}), and (\ref{eqn: ErrorTerm}) due to the fact that $x_{i,a}\in \range(V_a^p)$ and the property of pseudo-inverse specified in Definition \ref{def: pinv}.

Noe that Eqn. (\ref{eqn: ErrorTerm}) is a linear combination of $\{\xi_{j,a}^{p-1}\}_{  j\in\Rc_a^{p-1}}$. Thus, given $\Fc_p$, $\hat{r}_{i,a}^p-r_{i,a}$ is a conditionally sub-Gaussian random variable, whose parameter can be bounded as
\begin{align}
    \sum_{j\in\Rc_a^{p-1}}\left(x_{i,a}^\TT V_a^p\frac{e_{j,a}}{\|x_{j,a}\|}\right)^2f_{j,a}^{p-1}
    &\leq\frac{1}{\ell^2}\sum_{j\in\Rc_a^{p-1}}x_{i,a}^\TT V_a^p f_{j,a}^{p-1}e_{j,a}e_{j,a}^\TT V_a^p x_{i,a}\label{eqn: x>l}\\
    &=\frac{1}{\ell^2}x_{i,a}^\TT V_a^p x_{i,a}=(\sigma_{i,a}^p)^2,
\end{align}
where Eqn. (\ref{eqn: x>l}) comes from the bounded parameter assumption in Assumption~\ref{assump:bounded} that $\|x_{i,a}\|\geq \ell$ for all $i,a$.
\end{proof}

Lemma \ref{lemma: BndOneErr} enables us to use concentration inequalities to bound the probability of the bad event $\Ec(\alpha)$ defined in (\ref{eqn: errorevent}) as follows.

\begin{lemma}\label{lemma: BndAllErr}
 Under \texttt{Fed-PE}, we have $\mathbb{P}[\Ec(\alpha)]\leq\delta$.
\end{lemma}

\begin{proof}
Let $\alpha=\min\{\alpha_1,\alpha_2\}$, where 
\begin{align}
\alpha_1 &= \sqrt{2\log(2MKH/\delta)},\\ 
\alpha_2 &= \sqrt{2\log(KH/\delta) + d\log(ke)}.
\end{align}
As specified in Theorem~\ref{thm: complete_thm_Fed-PE}, $k$ is a number that satisfies $k\geq\max\{\alpha_2^2/d,1\}$. Note that this choice of $k$ ensures that $\alpha_2^2\geq d$.

Based on the definition of $\Ec(\alpha)$ in (\ref{eqn: errorevent}), we have 
\begin{align}
\Ec(\alpha) = \Ec(\min\{\alpha_1,\alpha_2\})\supset\Ec(\max\{\alpha_1,\alpha_2\}),
\end{align}
which implies that $\mathbb{P}[\Ec(\alpha)]=\max\left(\mathbb{P}[\Ec(\alpha_1)],\mathbb{P}[\Ec(\alpha_2)]\right)$.
Thus, it suffices to prove that $\mathbb{P}[\Ec(\alpha_i)]\leq \delta, \forall i\in\{1,2\}$. In the following, we bound $\Pb[\Ec(\alpha_1)]$ and $\Pb[\Ec(\alpha_2)]$ separately.

{\it (i) Bound $\Pb[\Ec(\alpha_1)]$.} First, based on Lemma \ref{lemma: BndOneErr} and Hoeffding's inequality, we have
\begin{align}
    \Pb\left [|\hat{r}_{i,a}^p-r_{i,a}|\geq \alpha_1\sigma_{i,a}^p|\Fc_{p-1} \right ]\leq 2\exp(-\alpha_1^2/2)=\frac{\delta}{MKH}.
\end{align}
Then, by applying the union bound, 
\begin{align}
\mathbb{P}[\Ec(\alpha_1)] &= \mathbb{P}[\exists p\in[H], i\in[M], a\in\Ac_i^{p-1}, |\hat{r}_{i,a}^p - r_{i,a}|\geq {\alpha_1}\sigma_{i,a}^p ]\\
&\leq\sum_{p\in[H]}\sum_{i\in[M]}\sum_{a\in\Ac_i^{p-1}}\mathbb{P}[|\hat{r}_{i,a}^p - r_{i,a}|\geq \alpha_1\sigma_{i,a}^p|\Fc_{p-1}]\\
&\leq HMK\frac{\delta}{MKH}=\delta.
\end{align}

{\it (ii) Bound $\Pb[\Ec(\alpha_2)]$.}
Next, we aim to show that $\mathbb{P}[\Ec(\alpha_2)]\leq \delta$ is true. Our first observation is that
\begin{align}
|\hat{r}_{i,a}^p-r_{i,a}|&=\bigg|x_{i,a}^\TT V_a^p\bigg(\sum_{j\in\Rc_a^{p-1}}\frac{e_{j,a}}{\|x_{j,a}\|}\xi_{j,a}^{p-1}\bigg)\bigg|\\
&\leq\|x_{i,a}\|_{V_a^p}\cdot\bigg\|\sum_{j\in\Rc_a^{p-1}}\frac{e_{j,a}}{\|x_{j,a}\|}\xi_{j,a}^{p-1}\bigg\|_{V_a^p}\label{eqn: xy<|x|y|}\\
&= \sigma_{i,a}^p\cdot\bigg\|\sum_{j\in\Rc_a^{p-1}}\frac{e_{j,a}}{\|x_{j,a}\|}\ell\xi_{j,a}^{p-1}\bigg\|_{V_a^p},
\end{align}
where (\ref{eqn: xy<|x|y|}) is due to Cauchy-Schwarz inequality. 
Thus, if $\Ec(\alpha_2)$ happens, we must have that  $$\left\|\sum_{j\in\Rc_a^{p-1}}\frac{e_{j,a}}{\|x_{j,a}\|}\ell\xi_{j,a}^{p-1}\right\|_{V_a^p}\geq \alpha_2$$ hold for some phase $p$ and arm $a$.

We now analyze the term $$X_{a,p}:=\left\|\sum_{j\in\Rc_a^{p-1}}\frac{e_{j,a}}{\|x_{j,a}\|}\ell\xi_{j,a}^{p-1}\right\|^2_{V_a^p}$$ for a given phase $p$ and arm $a$. Note that

\begin{equation}
\begin{aligned}
    X_{a,p}&=
    \sum_{i,j\in\Rc_a^{p-1}}\frac{\ell\xi_{i,a}^{p-1}e_{i,a}^\TT}{\|x_{i,a}\|}\left(\sum_{k\in\Rc_a^{p-1}}f_{k,a}^{p-1}e_{k,a}e_{k,a}^\TT\right)^\dagger\frac{\ell\xi_{j,a}^{p-1}e_{j,a}}{\|x_{j,a}\|}.\\
\end{aligned}
\end{equation}

In the following, we aim to write $X_{a,p}$ in a matrix form.
We note that $f_{i,a}^{p-1}$ may equal $0$ under the \bca algorithm, and if this happens, client $i$ does not pull arm $a$ during the collaborative exploration step in phase $p$, even though $a$ is in the active arm set. For this case, $\xi_{i,a}^{p-1} = 0$. Thus, we define a new random variable $\zeta_{i,a}$ as follows:
\begin{align}
\zeta_{i,a} =
\left\{\begin{array}{cc}
\frac{\ell\xi_{i,a}^{p-1}}{\sqrt{f_{i,a}^{p-1}}\|x_{i,a}\|}, &\text{ if } f_{i,a}^{p-1}\neq 0,\\
0, &\text{ if } f_{i,a}^{p-1} = 0.
\end{array}\right.
\end{align}
Then, $\{\zeta_{i,a}\}_{i\in\Rc_a^{p-1}}$ are conditionally independent 1-sub-Gaussian random variables. 

Define vector $\zeta := (\zeta_{i,a})_{i\in\Rc_a^{p-1}}$ and matrix $A:=(a_{i,j})_{i,j\in\Rc_a^{p-1}}$ where
\[a_{i,j} = \sqrt{f_{i,a}^{p-1}}e_{i,a}^\TT\bigg(\sum_{k\in\Rc_a^{p-1}}f_{k,a}^{p-1}e_{k,a}e_{k,a}^\TT\bigg)^\dagger e_{j,a}\sqrt{f_{j,a}^{p-1}}.\]
We can verify that $A$ is a symmetric matrix and 
\[\begin{aligned}
\sum_{k\in\Rc_a^{p-1}} a_{i,k}a_{k,j}& = \sqrt{f_{i,a}^{p-1}}e_{i,a}^\TT V_a^p \bigg(\sum_{k\in\Rc_a^{p-1}}f_{j,a}^{p-1}e_{k,a}e_{k,a}^\TT\bigg) V_a^p e_{j,a}\sqrt{f_{j,a}^{p-1}}\\
& = \sqrt{f_{i,a}^{p-1}}e_{i,a}^\TT V_a^p (V_a^p)^\dagger V_a^p e_{j,a}\sqrt{f_{j,a}^{p-1}}\\
&  = \sqrt{f_{i,a}^{p-1}}e_{i,a}^\TT V_a^p e_{j,a}\sqrt{f_{j,a}^{p-1}}= a_{i,j}.
\end{aligned}\]
Thus, $A^2 = A$, which implies that the eigenvalues of $A$ are either 1 or 0.  

Meanwhile, we have
\[\trace(A) = \sum_{i\in\Rc_a^{p-1}} a_{i,i} = V_a^p\sum_{i\in\Rc_a^{p-1}}f_{i,a}^{p-1}e_{i,a}e_{i,a}^\TT=V_a^p (V_a^p)^\dagger = \rank(V_a^p)=d_a^p.\]
Thus, the sum of the eigenvalues of $A$ is $d_a^p$. Combining with the fact that the eigenvalues of $A$ must be 1 or 0, we conclude that there are exactly $d_a^p$ eigenvalues equal to 1, and the rest eigenvalues are all 0. Therefore, $\rank(A) = d_a^p\leq d$.

By the definition of $\zeta$ and $A^2=A$, we have $X_{a,p} = \zeta^\top A\zeta = \|A\zeta\|^2$, where $\zeta$ is a conditionally 1-sub-Gaussian random vector, and $A$ is a $\Fc_{p-1}$-measurable matrix.
Therefore, for any $\lambda\in(0,1/2)$, we have
\begin{align}
    \mathbb{P}\left[X_{a,p}\geq \alpha_2^2|\Fc_{p-1}\right]&=\Pb\left[\|A\zeta\|^2\geq \alpha_2^2|\Fc_{p-1}\right]\\
    &=\Pb\left[e^{\lambda\|A\zeta\|}\geq e^{\lambda \alpha^2_2}\big|\Fc_{p-1}\right]\\
    &\leq e^{-\lambda \alpha_2^2}\Eb\left[e^{\lambda\|A\zeta\|}\big|\Fc_{p-1}\right] \label{eqn:markov}\\
    &\leq e^{-\lambda \alpha_2^2}\sqrt{\frac{1}{\det(I_d-2\lambda A^2)}}\label{eqn:lemma4}\\
    &=e^{-\lambda \alpha_2^2}(1-2\lambda)^{-d_a^p/2}\label{eqn: A^2=A}\\
    &\leq e^{-\lambda \alpha_2^2} (1-2\lambda)^{-d/2}
\end{align}
where (\ref{eqn:markov}) is due to Markov's inequality, (\ref{eqn:lemma4}) follows from Lemma~\ref{lemma: subG-L2}, and (\ref{eqn: A^2=A}) follows from the fact that the eigenvalues of $A$ are either $1$ or $0$ and there are exactly $d_a^p$ $1$'s. 

By choosing $\lambda = \frac{\alpha_2^2-d}{2\alpha_2^2} {\in (0, \frac{1}{2})}$, we have
\begin{align}
    \mathbb{P}\left[X_{a,p}\geq \alpha_2^2|\Fc_{p-1}\right] & \leq  \left(\frac{\alpha_2^2}{d}\right)^{{d}/{2}}\exp\left(-\frac{\alpha_2^2-d}{2}\right){\leq \frac{\delta}{KH}}.
\end{align}
The last inequality is due to the following analysis:
\begin{align}
    &\left(\frac{\alpha_2^2}{d}\right)^{{d}/{2}}\exp\left(-\frac{\alpha_2^2-d}{2}\right)\leq \frac{\delta}{KH}\\
    &\Leftrightarrow \frac{d}{2}\left(2\log\alpha_2 - \log d\right) + \frac{d}{2}-\frac{1}{2}\left(\sqrt{2\log(KH/\delta)+d\log(ke)}\right)^2\leq \log\left(\frac{\delta}{KH}\right)\\
    &\Leftrightarrow d\log\alpha_2\leq \frac{d\log(dk)}{2}\\
    &\Leftrightarrow \alpha_2^2\leq dk,\label{eqn:alpha2_2}
\end{align}
where (\ref{eqn:alpha2_2}) is assured by the definition of $k$.

Finally, the proof is completed by applying the union bound over phase $p$ and arm $a$.
\end{proof}

\subsection{Bound the Regret under the Good Event}
\begin{lemma}\label{lemma: NotKillOptimal}
If $\Ec^c(\alpha)$ occurs, we must have $a_i^*\in \Ac_i^p $ , i.e. any optimal arm will never be eliminated. 
\end{lemma}

\begin{proof} 
When $\Ec^c(\alpha)$ occurs, we have 
\[\hat{r}_{i,\hat{a}_i}^p - \hat{r}_{i,a_i^*}^p\leq r_{i,\hat{a}_i} - r_{i,a_i^*} + u_{i,\hat{a}_i}^p + u_{i,a_i^*}^p \leq u_{i,\hat{a}_i}^p + u_{i,a_i^*}^p,\]
where $\hat{a}_i$ is the estimated optimal arm for client $i$ in phase $p$. Thus, under the elimination procedure in \texttt{Fed-PE}, $a_i^*$ will not be eliminated.
\end{proof}

\begin{lemma}\label{lemma: PhaseRegret}
If $\Ec^c(\alpha)$ occurs, the regret of \alg in phase $p$ is upper bounded by $4\alpha\frac{L}{\ell}\sqrt{dKM}\frac{f^p+K}{\sqrt{f^{p-1}}}$.
\end{lemma}
\begin{proof}
When $\Ec^c(\alpha)$ happens, $|\hat{r}_{i,a} - r_{i,a}|\leq u_{i,a}^p = \alpha\sigma_{i,a}^p = {\frac{\alpha}{\ell}}\|x_{i,a}\|_{V_a^p}$ holds for all $p\in[H],i\in[M],a\in\Ac_i^{p-1}$. Therefore, by the construction of $\Ac_i^p$, pulling any active arm $a\in\Ac_i^p$ in phase $p$ will incur a regret upper bounded by
\begin{align}
\Delta_{i,a} &= r_{i,a_i^*} - r_{i,a}\\
&\leq \hat{r}_{i,\hat{a}_i} + u_{i,a_i^*}^p - \hat{r}_{i,a} + u_{i,a}^p\label{eqn:regret_gap1}\\
&\leq u_{i,\hat{a}_i}^p + u_{i,a_i^*}^p + 2u_{i,a}^p\label{eqn:regret_gap2}\\
&\leq 4\max_{a\in\Ac_i^p}u_{i,a}^p, \label{eqn:regret_gap3}
\end{align}
where (\ref{eqn:regret_gap1}) follows from the definition of $\Ec(\alpha)$, (\ref{eqn:regret_gap2}) follows from the arm elimination procedure under \texttt{Fed-PE}, and (\ref{eqn:regret_gap3}) is due to 
Lemma \ref{lemma: NotKillOptimal}.

Under \texttt{Fed-PE}, the phase length is fixed as $f^p+ K$ for $p\in[H]$. Then, {given that the good event $\Ec^c(\alpha)$ happens}, the total regret incurred during phase $p$, denoted as $R_p$, can be bounded as follows:
\begin{align}
R_p &\leq \sum_{i\in[M]} 4\max_{a\in\Ac_i^p} u_{i,a}^p(f^p+K)\\
& \leq4(f^p+K)\sqrt{M\sum_{i\in [M]}\max_{a\in \Ac_i^p}(u_{i,a}^p)^2}\label{eqn: 1}\\
& \leq \frac{4\alpha L}{\ell}(f^p + K)\sqrt{M\sum_{i\in[M]}\max_{a\in \Ac_i^p} \frac{x^\TT_{i,a}}{\|x_{i,a}\|_2} V_a^p \frac{x_{i,a}}{\|x_{i,a}\|_2} } \label{eqn:phase_regret1}\\
&{\leq} \frac{4\alpha L}{\ell}(f^p+K)\sqrt{M\sum_{i\in[M]} \max_{a\in \Ac_i^p} e_{i,a}^\TT \bigg(\sum_{j\in \Rc_a^p}\lceil\pi_{j,a}f^{p-1}\rceil e_{j,a}e_{j,a}^\TT\bigg)^\dagger e_{i,a}}\label{eqn:phase_regret2}\\
&\leq \frac{4\alpha L}{\ell}(f^p+K)\sqrt{\frac{M}{f^{p-1}}\sum_{i\in[M]} \max_{a\in \Ac_i^p} e_{i,a}^\TT \bigg(\sum_{j\in \Rc_a^p} \pi_{j,a} e_{j,a}e_{j,a}^\TT\bigg)^\dagger e_{i,a}}\label{eqn: ineq_opt}\\
&\leq \frac{4\alpha L}{\ell}(f^p+K)\sqrt{\frac{dKM}{f^{p-1}}}\label{eqn: 2}\\
&=\frac{4\alpha L}{\ell}\sqrt{dKM}\frac{f^p+K}{\sqrt{f^{p-1}}}.
\end{align}
where (\ref{eqn: 1}) is based on Cauchy-Schwarz inequality, (\ref{eqn:phase_regret1}) follows from the definition of $u_{i,a}^p := \alpha \|x_{i,a}\|_{V_a^p}/\ell$.
We note that summation in (\ref{eqn: ineq_opt}) is exactly the solution to the multi-client G-optimal design in (\ref{eqn: G_opt}), which equals $\sum_{a}d_a^p\leq dK$ according to Lemma \ref{lemma: OptimizationEquivalence}.

Note that the upper bound also holds for $p=1$ when $f^0$ is defined as $1$. This is because  $V_a^0:=(\sum_{j\in[M]}e_{j,a}e_{j,a}^\TT)^\dagger\preceq (\sum_{j\in \Rc_a^p} \lceil\pi_{j,a}f^0\rceil e_{j,a}e_{j,a}^\TT)^\dagger$ for any $\pi$, although the central server does not utilize the G-optimal design to obtain $\pi^0$ during initialization. Thus, we still have (\ref{eqn:phase_regret2}) hold. 
\end{proof}

\begin{corollary}\label{cor: Bound_sum_sigma}
Let $\sigma_{i,a}^p = \|x_{i,a}\|_{V_a^p}/\ell$. Then, under \texttt{Fed-PE}, for any $p\in[H]$, $a\in \Ac_i^p$, we have
 $ \sum_{i\in[M]}\max_{a\in\Ac_i}(\sigma_{i,a}^p)^2\leq
    \frac{dKL^2}{\ell^2 f^{p-1}}.$
\end{corollary}

\begin{proof}
Corollary~\ref{cor: Bound_sum_sigma} can be easily verified based on the fact that
$\sum_{i\in[M]}\max_{a\in\Ac^p_i}(\sigma_{i,a}^p)^2\leq \frac{L^2}{\ell^2} \sum_{i\in[M]}\max_{a\in \Ac_i^p}\frac{x^\TT_{i,a}}{\|x_{i,a}\|_2} V_a^p \frac{x_{i,a}}{\|x_{i,a}\|_2}$
and the rest analysis follows the same argument as in (\ref{eqn:phase_regret1})-(\ref{eqn: 2}).
\end{proof}

\subsection{Put Pieces Together}
Finally, we are ready to prove Theorem \ref{thm: BoundAlgorithm1}. Recall the superscript $p$ of $n$ is the exponent. Then, with probability at least $1-\delta$ and $n>1$, the total regret over the $H$ phases can be bounded as:
\begin{align}
R(T) &= \sum_{p=1}^H R_p \leq \sum_{p=1}^H \frac{4\alpha L}{\ell}\sqrt{dKM}(\sqrt{c}n^{\frac{p+1}{2}} + Kn^{-\frac{p-1}{2}}/\sqrt{c})\\
&\leq \frac{4\alpha L}{\ell}\sqrt{dKM}\left(\sqrt{cn}\frac{n^{\frac{H+1}{2}} - \sqrt{n}}{\sqrt{n}-1} + \frac{K}{\sqrt{
cn} - \sqrt{c}}\right)\\
&\leq \frac{4\alpha L}{\ell}\sqrt{
dKM}\left(\frac{\sqrt{n^2-n}}{\sqrt{n}-1}\sqrt{cn\frac{n^H-1}{n-1}} + \frac{K}{\sqrt{cn} - \sqrt{c}}\right)\label{eqn:total_regret1}\\
&\leq \frac{4\alpha L}{\ell}\sqrt{
dKM}\left(\frac{\sqrt{n^2-n}}{\sqrt{n}-1}\sqrt{T} + \frac{K}{\sqrt{cn} - \sqrt{c}}\right),\label{eqn:total_regret2}
\end{align}
where (\ref{eqn:total_regret1}) follows from that $n^{H/2} - 1\leq \sqrt{n^H -1}$, and (\ref{eqn:total_regret2}) is due to the fact that
$$T = \sum_{p=1}^H f^p + KH\geq\sum_{p=1}^H cn^p=cn\frac{n^{H}-1}{n-1}.$$
Since $$\alpha= O\left (\sqrt{\log(K(\log T)/\delta) + \min(d,\log M)} \right ),$$ the regret scales in
$$O\left(\frac{L}{\ell}\sqrt{dKM(\log (K(\log T)/\delta) + \min(d,\log M)}\left (\sqrt{T}+K \right) \right ).$$ 
When $K=O(\sqrt{T})$, the cumulative regret scales as \[O\left (\frac{L}{\ell}\sqrt{dKMT(\log (K(\log T)/\delta) + \min(d,\log M)} \right).\]

\begin{remark}\label{remark:com_fed_pe}
We note the upload cost in Theorem~\ref{thm: complete_thm_Fed-PE} can be reduced by a factor of $d$ by sending scalars $\{\nicefrac{\hat{\theta}^p_{i,a}}{\bar{e}_{i,a}}\}$ instead of vectors $\{\hat{\theta}^p_{i,a}\}$, as $\hat{\theta}^p_{i,a}$ is always in the same direction as $\bar{e}_{i,a}$. Similarly, for the download cost, if instead of broadcasting, the server calculates the projection of $\hat{\theta}^p_{a}$ along each direction $\bar{e}_{i,a}$, as well as $\|\bar{e}_{i,a}\|_{V_a^p}$, and send them back to client $i$, client $i$ can utilize those quantities instead of $\hat{\theta}_a^p$ and $V_a^p$ to obtain $\hat{r}_{i,a}^p$ and $u_{i,a}^p$. The corresponding download cost can then be reduced by a factor of $d^2$. Then, the overall communication cost would scale in $O(KM\log T)$.
\end{remark}

\section{Analysis of \texttt{Enhanced} \alg}\label{appx:enhanced}
\subsection{Algorithm Details}
As noted in Section~\ref{sec:enhanced}, the motivation of \texttt{Enhanced} \alg is to improve the efficiency of \alg by leveraging all historical information to obtain more accurate estimates of $r_{i,a}$ and $u_{i,a}$ in each phase $p$. To achieve this goal, we keep the other parts of \alg intact while only changing the arm elimination step as follows.

After calculating $\hat{r}_{i,a}^p$ according to Eqn.~(\ref{eqn: confidence_interval}) in \alg and obtaining $\sigma_{i,a}^p = \|x_{i,a}\|_{V_a^p}/\ell$, \texttt{Enhanced} \alg aggregates estimates from previous phases to obtain a refined estimate of ${r}_{i,a}^p$, denoted as $\bar{r}_{i,a}^p$, and the corresponding confidence interval $\bar{u}_{i,a}^p$ as follows:
\begin{align}
\bar{r}_{i,a}^p &= \frac{\sum_{q=1}^p\hat{r}_{i,a}^qf^{q-1}}{\sum_{q = 1}^pf^{q-1}}, \\
\bar{\sigma}_{i,a}^p &= \sqrt{\frac{dK}{M}+\sum_{q=1}^p (\sigma_{i,a}^q)^2 (f^{q-1})^2},\label{eqn: refined_sigma}\\
\alpha_{i,a}^p &= \sqrt{\log\bigg(\frac{M^3K(\bar{\sigma}_{i,a}^p)^2}{d\delta^2}\bigg)}
,\\ 
\bar{u}_{i,a}^p &= \frac{\alpha_{i,a}^p\bar{\sigma}_{i,a}^p}{\sum_{q=1}^pf^{q-1}},&\label{eqn: refined_alpha_u}
\end{align}
with $f^0:= 1$. 

Denote $\bar{a}_i^p := \arg\max_{a\in\Ac_i^{p-1}}\bar{r}_{i,a}^p$. 
Then, the updating rule for the active arm set $\Ac_i^p$ is changed as follows
\begin{align}
\Ac_i^p \leftarrow \left\{a\in \Ac_i^{p-1}\ |\ \bar{r}_{i,a}^p + \bar{u}_{i,{a}}^p\geq \bar{r}_{i,\bar{a}_i^p}^p -\bar{u}_{i,\bar{a}_i^p}^p\right\}.
\end{align}

\subsection{Performance of \texttt{Enhanced} \alg}\label{subsec: different_phase_length}
The performance of \texttt{Enhanced} \alg is summarized in the following theorem.
\begin{theorem}\label{thm: BoundRefinedAlgorithm}
Let $S_{p-1} := \sum_{q=0}^{p-1}f^q$. Consider time horizon $T$ consisting of $H$ phases such that $HK\leq S_H$. Then, with probability at least $1-\delta$, the total regret under \texttt{Enhanced} \alg can be bounded as
\begin{align}
R(T)\leq \frac{4\sqrt{6} L}{\ell}\sum_{p=1}^H\frac{S_p-S_{p-1}+K}{\sqrt{S_{p-1}}}\sqrt{dKM\log \frac{LKMT}{\ell\delta}}.
\end{align}
In particular, if $f^p\geq K$ and there exists a constant $c>1$ such that $S_p\leq cS_{p-1}$ holds for all $p$, then, $\sum_{p=1}^H\frac{S_p-S_{p-1}+K}{\sqrt{S_{p-1}}}\leq 4\sqrt{cT}$, which indicates that
\[R(T)\leq \frac{16\sqrt{6c}L}{\ell}\sqrt{dKMT\log\frac{LKMT}{\ell\delta}}.\]
\end{theorem}

In the following, we consider different selections of $f^p$ and evaluate the corresponding regret performance and communication costs. We fix the time horizon $T$, and assume $K\leq \sqrt{T}$. Note that $T = S_H + HK$.

\myparagraph{Uniform selection.} Let $f^1 = K-1$, $f^p = K$, and $S_p = pK$. Let $H = \min\{p: 2pK\geq T\}$. We have
\begin{align}
    \sum_{p=1}^H\frac{S_p - S_{p-1} + K}{\sqrt{S_{p-1}}}&
    = \sum_{p=1}^{H }\frac{2K}{\sqrt{pK}}\leq 2\sqrt{K}\int_0^{\frac{T}{2K}+1}\frac{1}{\sqrt{x}}dx = 4\sqrt{T/2 + K}\leq 4\sqrt{T}.
\end{align}
The communication cost with this phase length selection is $O(T/K)$.

\myparagraph{Exponential selection.} Choose $f^p = cn^p$, the same selection as in \texttt{Fed-PE}. Since $S_{p-1} \geq f^{p-1}= cn^{p-1}$, we have
\begin{align}
    \sum_{p=1}^H\frac{S_p - S_{p-1} + K}{\sqrt{S_{p-1}}}& \leq \sum_{p=1}^H \frac{f^p + K}{\sqrt{f^{p-1}}}= O(\sqrt{T}),
\end{align}
and the communication cost is the same as that of \texttt{Fed-PE}, i.e. $O(\log T)$. We note that the regret under \texttt{Enhanced} \alg scales in the same order as that under \texttt{Fed-PE}.

\myparagraph{Greedy selection.}  Generate a sequence $\{\tilde{S}_p\}_{p=0}^H$ that satisfies the following equations
\begin{align}
    \tilde{S}_0 = 1, \quad \tilde{S}_p - \tilde{S}_{p-1} + K = 2\sqrt{T\tilde{S}_{p-1}}.
\end{align}
Since $K\leq \sqrt{T}\leq \sqrt{T\tilde{S}_{p-1}}$, we have 
\begin{align}\label{eqn:recursion}
\tilde{S}_p + \sqrt{T\tilde{S}_{p-1}}\geq \tilde{S}_p - \tilde{S}_{p-1} + K = 2\sqrt{T\tilde{S}_{p-1}}.
\end{align}
Thus, $\tilde{S}_p\geq \sqrt{T\tilde{S}_{p-1}}$. Since $\tilde{S_0} = 1$, $\tilde{S}_{p}\geq T^{1-\frac{1}{2^p}}$. 

Let $H = \min\{p: \tilde{S}_p+pK\geq T\}$. 
In order to have $\tilde{S}_{H_0} + H_0K\geq T$, we have
\begin{align}
    & \tilde{S}_{H_0}+H_0K\geq T\Leftrightarrow \sum_{p=2}^{H_0} (\tilde{S}_p- \tilde{S}_{p-1} + K) + \tilde{S}_1 + K\geq T\Leftarrow 2\sum_{p=2}^{H_0} \sqrt{T\tilde{S}_{p-1}} \geq T\nonumber\\
    &\Leftarrow 2 \sqrt{\tilde{S}_{H_0-1}} \geq \sqrt{T}\Leftarrow 2 T^{\frac{1}{2} - \frac{1}{2^{H_0}}} \geq \sqrt{T}\Leftrightarrow \log 2\geq \frac{1}{2^{H_0}}\log T  \Leftarrow H_0\geq\log_2\log_2 T.
\end{align}
Thus, $H\leq \lceil\log_2\log_2 T\rceil\in\{p: \tilde{S}_p + pK\geq T\}$. 

Let $S_p = \lceil\tilde{S}_p\rceil$ for $p\in[H-1]$ and $S_H = T-HK$. Accordingly, $f^p = S_p - S_{p-1}$. Applying the equality in (\ref{eqn:recursion}), we have
\begin{align}
\sum_{p=1}^H \frac{S_{p}-S_{p-1} + K}{\sqrt{S_{p-1}}}
= O(H\sqrt{T}) = O(\sqrt{T} \log\log T),
\end{align}
and communication cost under this choice of phase length is ${O(\log\log T)}$.

We summarize the results in Table~\ref{table:enhanced}. Note that we only list the scaling in $T$ and omit the other common factors in both regrets and communication costs. Among those three selections, exponential selection achieves the same regret performance as uniform selection, with a significantly reduced communication cost. On the other hand, the greedy selection achieves the lowest communication cost, with the price of slightly increased regret bound.

\begin{table}[h]
  \caption{Performance comparison with different phase length selection}
  \label{table:enhanced}
  \centering
  \begin{tabular}{ccc}
    \toprule
    $f^p$ selection    & Regret     & Communication Cost \\
    \midrule
    Uniform & $O(\sqrt{T\log T})$  &   $O(T/K)$   \\
    Exponential   & $O(\sqrt{T\log T})$ &  $O(\log T)$   \\
    Greedy    &   $O(\sqrt{T\log T}\log\log T)$    & $O(\log\log T)$  \\
    \bottomrule
  \end{tabular}
\end{table}

\subsection{Proof of Theorem~\ref{thm: BoundRefinedAlgorithm}}
Similar to the proof of Theorem~\ref{thm: complete_thm_Fed-PE}, we define a
``bad'' event as follows:
\begin{align}
\bar{\Ec} = \left \{\forall p\in[H],i\in[M],a\in\Ac_i^{p-1}, |\bar{r}_{i,a}^p - r_{i,a}|\geq \bar{u}_{i,a}^p \right \}.
\end{align}

Then, we will first show that the probability of $\bar{\Ec}$ is upper bounded by $\delta$, and then analyze the regret when $\bar{\Ec}^c$ occurs.

\begin{lemma}\label{lemma:bad_enhanced}
Under \texttt{Enhanced} \alg, we have
$\mathbb{P}[\bar{\Ec}]\leq \delta$.
\end{lemma}
\begin{proof}
Lemma \ref{lemma: BndOneErr} has shown that each $\hat{r}_{i,a}^q - r_{i,a}$ is a conditionally sub-Gaussian random variable. Thus, $(\hat{r}_{i,a}^q-r_{i,a})f^{q-1}$ is a conditionally $\sigma_{i,a}^q f^{q-1}$-sub-Gaussian random variable. Letting $\sigma^2=dK/M$ in Lemma \ref{lemma: Laplace}, we have
\[\Pb\left [|\bar{r}_{i,a}^p - r_{i,a}| \geq \bar{u}_{i,a}^p|\Fc_{p-1} \right ] \leq \frac{\delta}{MK}.\]
Lemma~\ref{lemma:bad_enhanced} then follows by applying the union bound over $i\in[M]$ and $a\in[K]$.
\end{proof}

\begin{lemma}\label{lemma: phase_regret_refined}
If $\bar{\Ec}^c$ happens, then the regret in phase $p\in[H]$ is upper bounded by  $\frac{4\sqrt{6} L}{\ell}\frac{f^p+K}{\sqrt{S_{p-1}}}\sqrt{dKM\log\frac{LKMT}{\ell\delta}}$.
\end{lemma}

\begin{proof}
First, we bound the terms $\bar{\sigma}_{i,a}^p$ and $\alpha_{i,a}^p$ defined in (\ref{eqn: refined_sigma}) and (\ref{eqn: refined_alpha_u}), respectively. 

We note that Corollary~\ref{cor: Bound_sum_sigma} still holds under $\texttt{Enhanced Fed-PE}$ due to the fact that the local estimates $\{\hat{\theta}_{i,a}^{p}\}$ only rely on the observations collected in phase $p$, the same as under $\texttt{Ped-PE}$.  

Thus,
\begin{align}\bar{\sigma}_{i,a}^p& = \sqrt{\frac{dK}{M}+\sum_{q=1}^p (\sigma_{i,a}^q)^2 (f^{q-1})^2}\\
&\leq\sqrt{dK + \sum_{q=1}^p\frac{dKL^2(f^{q-1})^2}{f^{q-1}\ell^2}}\label{eqn:sigma_bar1}\\
&\leq\frac{L}{\ell}\sqrt{2dKT}.\label{eqn:sigma_bar2}
\end{align}

Using (\ref{eqn:sigma_bar2}), we can bound $\alpha_{i,a}^p$ as follows:
\begin{align}
\alpha_{i,a}^p &= \sqrt{\log\frac{M^3K(\bar{\sigma}_{i,a}^p)^2}{d\delta^2}}\\
&\leq\sqrt{\log\frac{2L^2M^3K^2T}{\ell^2\delta^2}}\\
&\leq\sqrt{3\log\frac{LKMT}{\ell\delta}} {:=} \alpha_0.\label{eqn:alpha_0}
\end{align}

Following similar argument as in Lemma \ref{lemma: PhaseRegret}, we have
\begin{align}
    R_p&\leq \sum_{i\in[M]}4\max_{a\in\Ac_i^{p}} \bar{u}_{i,a}^p(f^p+K)\\
    & = 4(f^p+K)\sum_{i\in[M]}\max_{a\in\Ac_i^p}\alpha_{i,a}^p\frac{\bar{\sigma}_{i,a}^p}{S_{p-1}} \\
    &\leq 4\alpha_0\frac{f^p+K}{S_{p-1}}\sum_{i\in[M]}\max_{a\in\Ac_i^p}\sqrt{\frac{dK}{M}+\sum_{q=1}^p(\sigma_{i,a}^{q})^2 (f^{q-1})^2} \label{eqn:alpha_bound} \\
    &\leq 4\alpha_0\frac{f^p+K}{S_{p-1}}\sum_{i\in[M]}\sqrt{\frac{dK}{M}+\sum_{q=1}^p\max_{a\in\Ac_i^q}(\sigma_{i,a}^{q})^2 (f^{q-1})^2}\label{eqn:phase_regret_enhanced1}\\
    &\leq \frac{4\alpha_0(f^p+K)}{S_{p-1}}\sqrt{M\bigg(\frac{dK}{M}+\sum_{q=1}^p\sum_{i\in[M]}\max_{a\in\Ac_i^q} (\sigma_{i,a}^q)^2(f^{q-1})^2\bigg)},\label{eqn:cauchy}
\end{align}    
where in (\ref{eqn:alpha_bound}) we use $\alpha_0$ from Eqn. (\ref{eqn:alpha_0}) to bound $\alpha^p_{i,a}$, and (\ref{eqn:cauchy}) follows from Cauchy-Schwarz inequality. 

Then, we apply the result from Corollary~\ref{cor: Bound_sum_sigma} on all $q\in\{1,\ldots,p\}$ and $R_p$ can be further bounded as
\begin{align}    
    R_p&\leq\frac{4\alpha_0 L}{\ell}\frac{f^p+K}{S_{p-1}}\sqrt{dK + M\sum_{q=1}^p\frac{dK(f^{q-1})^2}{f^{q-1}} }\\
    &\leq\frac{4\alpha_0 L}{\ell}\frac{f^p+K}{S_{p-1}}\sqrt{dKS_{p-1} + dKMS_{p-1} }\\
    &\leq \frac{4\sqrt{6} L}{\ell}\frac{f^p+K}{\sqrt{S_{p-1}}}\sqrt{dKM\log\frac{LKMT}{\ell\delta}}.
\end{align}
\end{proof}

Finally, we are ready to prove Theorem \ref{thm: BoundRefinedAlgorithm}. The first part of the theorem follows directly from Lemma \ref{lemma: phase_regret_refined}. 
The second part of Theorem \ref{thm: BoundRefinedAlgorithm} can be obtained by noticing that
\begin{align}
    \sum_{p=1}^H \frac{S_p - S_{p-1}}{\sqrt{S_{p-1}}}&
    =\sum_{p=1}^H\left(\sqrt{\frac{S_p}{S_{p-1}}} + 1\right)(\sqrt{S_p} - \sqrt{S_{p-1}})\\
    &\leq 2\sqrt{c}\sum_{p=1}^H(\sqrt{S_p} - \sqrt{S_{p-1}})\\
    &\leq 2\sqrt{cT}.
\end{align}

\section{Lower Bound Analysis}\label{appx:lower}

\subsection{Collinearly-dependent Policies}
First, we state the definition of collinearity between a pair of vectors. 

\begin{definition}[Collinear vectors]\label{defn:collinear}
For a given set of vectors $\Xc$, two vectors $x,y\in \Xc$ are called \emph{collinear} (denoted as $x\sim y$) if there exists a subset $S\subset \Xc$ such that the following conditions are satisfied: 1) $x\notin \Span(S)$; and 2) $x \in \Span(S\cup\{y\})$.
\end{definition}

Definition \ref{defn:collinear} indicates that two clients $i$ and $j$ are collinear if there exists an arm $a\in[K]$ such that the corresponding feacture vectors $x_{i,a}$ and $x_{j,a}$ are collinear given $\{x_{i,a}\}_{i\in[M]}$.

\begin{Proposition}\label{prop:equivalent}
The collinear relation on a given set of vectors $\Xc$ is an equivalence relation, i.e., for any $x,y,z\in \Xc$, we have
\begin{itemize}[leftmargin=18pt,topsep=0pt, itemsep=0pt,parsep=0pt] 
\item[1)] \emph{Reflexivity:} $x\sim x$. % \congc{what is $c$?}
\item[2)] \emph{Symmetry:} If $x\sim y$, then we must have $y\sim z$. 
\item[3)] \emph{Transitivity:} If $x \sim y$ and $y\sim z$, then, we must have $y\sim z$. 
\end{itemize}
\end{Proposition}

\begin{proof} We prove those three properties one by one in the following.
\begin{itemize}
[leftmargin=*]\itemsep=0pt
\setlength{\leftmargin}{-0.5in}
    \item \textit{Proof of Reflexivity.} The reflexivity is obvious when we set $S=\emptyset$. 
    \item \textit{Proof of Symmetry.} Assume for a subset $S\subset \Xc$, we have $x\notin \Span(S)$ and $x\in \Span(S\cup \{y\})$. This implies that
    \begin{align}\label{eqn:prop2-1}
         y\notin \Span(S).
    \end{align}
    
    Meanwhile, we have
    \begin{align}
        \Span(S)&\subset \Span(S\cup \{x\}) \subseteq \Span(S\cup \{y\}).
    \end{align}
    Since $\dim( \Span(S\cup \{y\}))= \dim( \Span(S))+1$, we must have
    \begin{align}
        \Span(S\cup \{x\})= \Span(S\cup \{y\}),
    \end{align}
    which indicates that 
    \begin{align}\label{eqn:prop2-2}
    y\in \Span(S\cup \{x\}).
    \end{align}
    
    Therefore, combining (\ref{eqn:prop2-1}) and (\ref{eqn:prop2-2}), we must have $y\sim x$.
    
   \item \textit{Proof of Transitivity.} We consider the following possible cases.
   \begin{enumerate}
       \item[a)] $\Span(x)$, $\Span(y)$, $\Span(z)$ are not distinct. Following the proof of reflexivity and the symmetry property, we can easily show that $x\sim z$.
       \item[b)] $\Span(x)$, $\Span(y)$, $\Span(z)$ are all distinct, and $\dim(\Span(\{x,y,z\} ))=2$. For this case, we let $S=\{y\}$. Then, since $\Span(x)$, $\Span(y)$, $\Span(z)$ are all distinct, we must have  $x \notin \Span(S)$. 
       On the other hand, since $\dim(\Span(\{x,y,z\} ))=2$, we must have $\Span(\{x,y,z\}) =\Span(\{y,z\})$. This implies that 
       \begin{align}
           x & \in \Span(\{x,y,z\}) =\Span(\{y,z\}) =\Span(S\cup\{z\}). 
       \end{align}
       Therefore, we must have $x\sim z$.
       \item[c)] $x,y,z$ are linearly independent. Let $S,T\subset \Xc$ be the minimal subsets such that $x\notin \Span(S)$, $x\in \Span(S\cup\{y\})$, and $z\notin \Span(T)$, $z\in \Span(T\cup\{y\})$. Then, we have the following subcases:

           \underline{{\it c1}) $y\notin \Span(S\cup T)$}. This implies that
           \begin{align}\label{eqn:prop3-11}
           x\notin \Span(S\cup T).
           \end{align}
           This is because $y\sim x$, thus $y\in \Span(S\cup\{x\})$. If $x\in \Span(S\cup T)$, then we must have $y\in \Span(S\cup (S\cup T))=\Span(S\cup T)$.

           On the other hand, since $y\sim z$, we have $ y\in \Span(T\cup\{z\})$, which indicates that
           \begin{align}\label{eqn:prop3-12}
               x\in \Span(S\cup\{y\})\subset \Span(S\cup T\cup\{z\})=\Span((S\cup T)\cup\{z\}).
           \end{align}
           
           Combining (\ref{eqn:prop3-11}) and (\ref{eqn:prop3-12}), we have $x\sim z$.
           
         \underline{{\it c2)} $y\in \Span(S\cup T)$}. Since we assume both $S$ and $T$ are minimal, then, for any $s\in S$, $x\sim s$ as well, as %. This is due to the assumption that
           \begin{align}\label{eqn:prop3-21}
           x & \notin \Span((S\backslash \{s\}) \cup \{y\}),\quad x \in \Span(S\cup \{y\}).
           \end{align}
           
           Let $u$ be a vector in $S$, and define $S_1:=S\backslash \{u\}$. For the remaining vectors in $S_1$, label them as $s_1$, $s_2$, $\ldots$, $s_m$. Then, we set $\Bc_1:=\{x,y,s_1,\ldots,s_m\}$ as the basis for $\Span(S_1 \cup \{x,y\})$. Note that those vectors in $\Bc_1$ are now unit vectors with respect to $\Bc_1$.
           
           Since $x\sim u$, we have $u\in \Span(S_1 \cup \{x,y\} ) $. Denote the coordinates of $u$ relative to basis $\Bc_1$ as $[u]_{\Bc_1}:=(u_1,u_2,\ldots,u_{m+2})^T$. Then, we must have $u_i\neq 0$ for any $i$, as otherwise $S$ cannot be minimal. This also implies that replacing any vector in $\Bc_1$ by $u$ also form a basis for $\Span(S_1 \cup \{x,y\})$.
           We further consider two possible cases.

              \underline{{\it c2-i)} $z\in \Span(S\cup\{y\})$.} Since $\Span(S\cup\{y\})=\Span(S_1\cup\{u, y\})=\Span({S_1}\cup\{x, y\})$, we can express the coordinates of $z$ relative to the selected basis $\Bc_1$ as $[z]_{
              \Bc_1}:=(z_1,z_2,\ldots,z_{m+2})$.  Let $i^*$ be the first non-zero coordinate with $z_{i^*}\neq 0$. Then, define a vector $w$ as  
           \begin{align}\label{eqn:w}
           w&=\left\{ \begin{array}{cc}
                u, &  i^*=1,\\
            y,    & i^*=2,\\
            s_{i^*-2}, &  \mbox{otherwise}.
           \end{array}\right.
           \end{align}
          and $W=(S\cup\{y\})\backslash\{w\}$.
           
           Then, $x\notin \Span(W)$ according to (\ref{eqn:prop3-21}), and $x\in \Span(S\cup\{y\})=\Span(W\cup\{z\})$ based on the definition of $w$ in (\ref{eqn:w}). Thus, $x\sim z$.

      \underline{ {\it c2-ii)} $z\notin \Span(S\cup\{y\})$}. Assume $\dim(\Span(S\cup T))=n$. Since $z\in \Span(T\cup\{y\})$, $y\in \Span(S\cup T)$, we must have $z\in \Span(S\cup \{y\}\cup T)=\Span(S\cup T)$. While $\dim(\Span(S\cup\{y\}))$ equals $m+2$ as shown before, when $z$ is included, we must have $n>m+2$. Set $\Bc_2=\{x,y,s_1,s_2,\ldots,s_m, z, t_1,\ldots,t_{n-m-3}\}$, where $t_i$s are vectors from $T$ that are linearly independent with the other vectors in $\Bc_2$. Then $\Bc_2$ form a basis for $\Span(S\cup T)$, and all the vectors in $\Bc_2$ are unit vectors with respect to $\Bc_2$. Denote $T_1=\{t_1,\ldots,t_{n-m-3}\}$.

       Since $z\in \Span(\{y\}\cup T)$ and $z\notin \Span(\{y\}\cup T_1)$, there must exist a vector $v\in T\backslash T_1$, whose coordinate vector with respect to $\Bc_2$ is denoted as $[v]_{\Bc_2}:=(v_1,v_2,\ldots,v_n)$ with $v_{m+3}\neq 0$.
       
       If $v_1\neq 0$, let $W= S_1\cup T_1\cup \{y,v\}$. Note that $S_1\cup T_1\cup\{y\}$ contains unit vectors in $\Bc$ except $x$ and $z$, while $v$ contains non-zero components along the dimension spanned by $x$ and $z$. Therefore, $x\notin \Span(W)$, while $x\in \Span(W \cup \{z\})=\Span(S\cup T)$. Thus, $x\sim z$.
       
       If $v_1=0$, then, $v_2,v_3,\ldots,v_{2+m}$ cannot be all zero. Otherwise, $z\in \Span(T_1\cup\{v\})\subset \Span(T)$, which contradicts the assumption that $z\notin \Span(T)$. % and $z\in \Span(T\cup\{y\})$. 
       
       Let $i^*$ be the smallest index with $v_{i^*}\neq 0$, and
        \begin{align}
           w&=\left\{ \begin{array}{cc}
            y,  & i^*=2,\\
            s_{i^*-2}, &  \mbox{otherwise}.
           \end{array}\right.
           \end{align}
          Let $W=((S\cup \{y\})\backslash\{w\})\cup T_1\cup \{v\}$. Compared with $\Bc_2$, $W$ does not contain $x,z$ but $u,v$; Besides, one vector in $S\cup\{y\}$ is removed.   
          If $x\in\Span(W)$, then $\Span(W)=\Span(((S\cup \{x,y\})\backslash\{w\})\cup T_1\cup \{v\})=\Span((S_1\cup \{x,y\})\cup T_1\cup \{v\})=\Span(S\cup T)$. However, $\dim(W)=n-1$, which contradicts with the assumption that $\dim(\Span(S\cup T))=n$. Thus, we must have $x\notin \Span(W)$. 
           .
           
           Next, to show that $x\sim z$, it suffices to show that $x\in \Span(W\cup \{z\})$. Through linear algebra, we can verify that $\Span(W\cup \{z\})=\Span(S\cup T)$. Since $x\in \Span(S\cup T)$, the proof is thus complete.
   \end{enumerate}
\end{itemize}
\end{proof}

\begin{lemma}
For any given arm $a\in[K]$, partition feature vectors $\{x_{{i,a}}\}_{i\in[M]}$ into $L$ equivalence classes denoted as $S_1$, $S_2$, $\ldots$, $S_L$ based on the collinear relation. Then, different classes are linearly independent, i.e., 
\begin{align}
    \Span(S_1 \cup S_2\ldots \cup S_i)\cap \Span(S_{i+1})=\{0\}
\end{align}
for any $i=1,\ldots,L-1$.
\end{lemma}
\begin{proof}
We prove it through contradiction. 
First, we assume $\dim(\Span(S_1 \cup S_2\ldots \cup S_i))=n$, and assume $\{e_1,e_2,\ldots,e_n\}\subseteq \cup_{j=1}^i S_j$ is a basis for it.  

Assume  $\Span(S_1 \cup S_2\ldots \cup S_i)\cap \Span(S_{i+1})\neq\{0\}$. Then, there must exist $u\in \Span(S_{i+1})$, $u\neq 0$, such that $u=\sum_{i}u_i e_i$. Then, we have $u\in \Span(\{e_i|u_i\neq 0\})$. 
However, $u\notin \Span(\{e_i|u_i\neq 0\}-\{e_j\})$ for any $j$ such that $u_j\neq 0$.
Thus, $u\sim e_j$ for any $j$ such that $u_j\neq 0$. Based on Proposition~\ref{prop:equivalent}, we must have $S_{i+1}$ equivalent to at least one of $S_1$, $S_2$, $\ldots$, $S_i$, which contradicts with the assumption on those equivalence classes. 

Therefore, we must have 
$    \Span(S_1 \cup S_2\ldots \cup S_i)\cap \Span(S_{i+1})=\{0\}$. Note that the labeling of the classes can be arbitrary, which indicates the property holds for any groups of the equivalence classes.
\end{proof}

\begin{lemma}\label{lemma:collinear}
Let $X\in \Rb^{d\times m}$ and $Y\in \Rb^{d\times n}$ be two matrices such that $\Span(X) \cap \Span(Y)=\{0\}$. 
Then,
\begin{align}
    u^\TT(X X^\TT + Y Y^\TT)^\dagger v& = \left\{ \begin{array}{cl}
      u^\TT (X X^\TT )^\dagger v,    & \mbox{ if }u,v\in\Span(X), \\
      0, &\mbox{ if } u\in\Span(X), v\in\Span(Y),\\
      u^\TT(Y Y^\TT)^\dagger v ,  &  \mbox{ if } u,v \in \Span(Y).
    \end{array}  \right. 
\end{align}
\end{lemma}

\begin{proof}
Since $\Span(X) \cap \Span(Y)=\{0\}$, there exists an invertible matrix $A\in \Rb^{d\times d}$ such that 
\begin{align}
    \Span(AX)&=\Span(e_1,e_2,\ldots, e_{d_X}),\\
     \Span(AY)&=\Span(e_{d_X+1},e_{d_X+2},\ldots, e_{d_X+d_Y}),
\end{align}
where $d_X:=\dim(\Span(X))$ and $d_Y:=\dim(\Span(Y))$, and $e_i$'s are unit vectors in $\Rb^d$.

Then, 
\begin{align}
     u^\TT(X X^\TT + Y Y^\TT)^\dagger v&= (Au)^\TT ((AX) (AX)^\TT + (AY) (AY)^\TT)^\dagger (Av).
\end{align}

Note that $ (AX) (AX)^\TT$ and $(AY) (AY)^\TT$ are now block diagonal matrices whose  non-zero blocks do not overlap with each other. To be more specific, they are in the following forms
\begin{align}
    (AX) (AX)^\TT&=\begin{bmatrix}
    \star & 0 & 0\\
    0 & 0 & 0\\
    0 & 0 & 0
    \end{bmatrix}, \qquad      (AY) (AY)^\TT=\begin{bmatrix}
    0 & 0 & 0\\
    0 & \star & 0\\
    0 & 0 & 0
    \end{bmatrix},
\end{align}
where $\star$ represents non-zero blocks.

Thus, if $u,v\in\Span(X)$, $Au$, $Av$ must be in the form of $\begin{bmatrix} \star\\ 0\\ 0\end{bmatrix}$, which verifies that $$ (Au)^\TT ((AX) (AX)^\TT + (AY) (AY)^\TT)^\dagger (Av)=(Au)^\TT ((AX) (AX)^\TT)^\dagger (Av)=u^\TT(X X^\TT)^\dagger v.$$

Similarly, if $u,v\in\Span(Y)$, we have 
$$ (Au)^\TT ((AX) (AX)^\TT + (AY) (AY)^\TT)^\dagger (Av)=(Au)^\TT ((AY) (AY)^\TT)^\dagger (Av)=u^\TT(Y Y^\TT)^\dagger v.$$

If  $u\in\Span(X)$ and $v\in\Span(Y)$, $Au$, $Av$ must be in the form of $\begin{bmatrix} \star\\ 0\\ 0\end{bmatrix}$ and $\begin{bmatrix} 0\\\star\\ 0\end{bmatrix}$, respectively. Thus, we must have 
\[ (Au)^\TT ((AX) (AX)^\TT + (AY) (AY)^\TT)^\dagger (Av)=0.\]
\end{proof}
\begin{remark}
We point out that Lemma~\ref{lemma: linIndpdnt} can be treated as a special instance of Lemma~\ref{lemma:collinear}.
\end{remark}

\begin{theorem}\label{thm:collinear}
For the federated linear contextual bandits considered in this work, partition clients into equivalence classes where two clients $i$ and $j$ are in the same class if there exists an arm $a\in[K]$ such that $x_{i,a}\sim x_{j,a}$. Let $S$ be the class that includes client $i$, and $X\in \Rb^{d\times m}$ be a matrix with $\Span(X) \subseteq \Span(\{x_{j,a}\}_{j\in S})$. Let $Y\in\Rb^{d\times n}$ be another matrix with $\Span(X)\cap \Span(Y)=\{0\}$. Let $u_X,v_X\in \Span(X)$ and $v_Y\in\Span(Y)$. 
Then, for a given policy $\pi$, if the local estimates for any arm $a\in[K]$ are in the following forms
\begin{align}\label{eqn:collinear_form1}
\hat{r}_{i,a}&=u_X^T(X X^\TT + Y Y^\TT)^\dagger (v_X+v_Y)\\
\hat{\sigma}_{i,a}&=u_X^T(X X^\TT + Y Y^\TT)^\dagger u_X \label{eqn:collinear_form2}
\end{align}
and the decisions are made based solely on those quantities besides other system parameters, then $\pi$ is a collinearly-dependent policy.
\end{theorem}
\begin{proof}
Based on Lemma~\ref{lemma:collinear}, we can show that
\begin{align}
\hat{r}_{i,a}&=u_X^T(X X^\TT + Y Y^\TT)^\dagger (v_X+v_Y)\\
& =u_X^T(X X^\TT + Y Y^\TT)^\dagger v_X + u_X^T(X X^\TT + Y Y^\TT)^\dagger v_Y\\
& = u_X^T(X X^\TT )^\dagger v_X, \label{eqn:collinear1}\\
{\sigma}_{i,a}&=u_X^T(X X^\TT + Y Y^\TT)^\dagger u_X =u_X^T(X X^\TT)^\dagger u_X.\label{eqn:collinear2}
\end{align}
Since $\hat{r}_{i,a}$ and $\hat{\sigma}_{i,a}$ only depend on the feature vectors associated with clients in the same class $S$, and the decisions under $\pi$ are based on those estimates only, $\pi$ is a collinearly-dependent policy. 
\end{proof}

\begin{corollary}\label{cor:fed-pe}
\alg and \texttt{Enhanced} \alg are both collinearly-dependent policies.
\end{corollary}
Corollary~\ref{cor:fed-pe} can be easily verified by checking the expressions of $\hat{r}_{i,a}^p$ and $\sigma^p_{i,a}$ under \alg and \texttt{Enhanced} \texttt{Fed-PE}, respectively.  
%\end{proof}

\begin{corollary}\label{cor:linucb}
LinUCB type of policies are collinearly-dependent policies if we include $d$ dummy clients with $d$ unit feature vectors in $\Rb^d$.
\end{corollary}
\begin{proof}
For single-agent sequential linear contextual bandits with disjoint parameters, 
the LinUCB algorithm \citep{Li:2010:LinUCB} works as follows: at time $t$, for any incoming context $c_t=i$, the learner obtains estimated reward and uncertainty for each arm $a\in[K] $ as follows:
\begin{align}
\hat{r}_{i,a}(t)&   = x^\TT_{i,a}V^{-1}_a(t) \left(\sum_{s=1}^{t-1} \lv\{a_s=a\} x_{c_s, a} y_s\right)\label{eqn:linucb1}\\
\hat{\sigma}_{i,a}(t)&=\|x_{i,a}\|_{V^{-1}_{a}(t)},\label{eqn:linucb2}
\end{align}
where 
\begin{align}\label{eqn:V}
V_a(t)= \sum_{s=1}^ {t-1} x_{c_s,a} x_{c_s,a}^\TT \lv\{a_s=a\} +\lambda I_d.
\end{align}
It then picks the arm with the highest upper confidence bound $\hat{r}_{i,a}(t)+\alpha \hat{\sigma}_{i,a}(t)$, where $\alpha$ is a real number.

If we view each context to be associated with a fixed client as in our federated linear contextual bandits setting, and treat the identity matrix $I_d$ as $\sum_{i=1}^d e_i e_i^\TT$ where $e_i$'s are unit feature vectors associated with $d$ dummy clients, then, the estimate at client $i$ for arm $a$ only depends on clients with $x_{j,a}\sim x_{i,a}$. Thus, the decision at client $i$ only depends on other clients in the same equivalent class, and such LinUCB type of policies are collinearly-dependent.
\end{proof}

\begin{corollary}\label{cor:thompson}
Thompson sampling based policies with Gaussian priors are collinearly-dependent policies.
\end{corollary}
\begin{proof}
To make the setting consistent, we first modify the Thompson sampling algorithm
proposed for the shared parameter linear contextual bandits \citep{Agrawal:2013:linear} to a disjoint parameter setting. 

Similar to LinUCB, the learner would calculate the mean and variance of the Gaussian distribution for the parameter $\theta_a$ associated with each arm $a\in[K]$ according to
\begin{align}
\hat{\theta}_a(t) & =  V^{-1}_a(t) \left(\sum_{s=1}^{t-1} \lv\{a_s=a\}  x_{c_s, a_s} y_s\right),\\
\hat{\sigma}_{i,a}(t)&= \alpha^2 V_a^{-1}(t),\label{eqn:ts2}
\end{align}
where $\alpha\in\Rb$ is a constant and $V_a^{-1}(t)$ is defined in the same way as in (\ref{eqn:V}). 
It then samples $\theta_a(t)$ from the distribution $\Nc(\hat{\theta}_a(t), \hat{\sigma}_{i,a}(t) )$, and then plays the arm $a_t := \arg \max_{a} x_{i,a}^\TT \theta_a(t)$ if the incoming context $c_t=i$.

We note that under the Thompson sampling procedure, the reward associated with arm $a$ and incoming context $i$ at time $t$ is actually a Gaussian random variable with distribution $\Nc\left(x_{i,a}^\TT V_a^{-1}(t) \left(\sum_{s=1}^{t-1}  \lv\{a_s=a\}  x_{c_s, a_s} y_s\right), \alpha^2x_{i,a}^\TT V_a^{-1}(t) x_{i,a}\right)$. According to Theorem~\ref{thm:collinear}, since the mean and variance are in the form specified in (\ref{eqn:collinear_form1})-(\ref{eqn:collinear_form2}), the distribution of the reward under the incoming context $i$ only depends on other contexts that in the same equivalence class.

If we associate each context with a fixed client, then, we can see that the decision at client $i$ only depends on other clients in the same equivalence class. Thus, such Thompson sampling type of policies are collinearly-dependent.
\end{proof}

\subsection{Proof of Theorem~\ref{thm:lower}}
We provide the detailed proof for the minimax regret bound in Theorem~\ref{thm:lower}. 
For simplicity, we assume $M/d$ is an integer. We use $\mathbf{1}$ and $\mathbf{0}$ to denote the all-one and all-zero vectors, respectively.

First, divide $[M]$ into $d$ groups, each with $M/d$ clients. Let the $s$-th group be $G_s = \{(s-1)M/d+1, \ldots, sM/d\}$, where $s\in\{1,2,\ldots,d\}$. We thus have $[M] = \cup_{s=1}^d G_s$. Consider the case where $x_{i,a} = e_s$ for all $i\in G_s$ and $a\in[K]$, where $e_s$ is the $s$-th canonical basis vector.

We first choose $\theta = \{\Delta\mathbf{1},\mathbf{0},\ldots,\mathbf{0}\}$, i.e. $\theta_1 = (\Delta,\Delta,\ldots,\Delta)^\TT$, and $\theta_{a} = (0,0,\ldots,0)^\TT$ for any $a\geq 2$. ($\Delta$ will be specified later.) Then, the reward for client $i$ pulling arm $a_t$ at time $t$ would be $y_{i,t}:=x_{i,a_t}^\TT\theta_{a_t}+\eta_{i,t}$, where $\eta_{i,t}$ is independently drawn from a standard Gaussian distribution.

Consider a collinearly-dependent policy $\pi$. Denote $\Pb_{\theta}$ as the probability measure induced by $\theta=\{\theta_a\}_{a\in[K]}$ under policy $\pi$, and $\Eb_{\theta}$ as the corresponding expectation. Let $T_{i,a}$ be the number of times that client $i$ pulls arm $a$ within time horizon $T$. Note that $\sum_{a\in[K]} T_{i,a} = T$.

Define $\bar{T}_{s,a} := \sum_{i\in G_s}T_{i,a}$, i.e., the total number of times that arm $a$ has been pulled by the clients in group $s$ over time horizon $T$. Then, we have
\[\sum_{a\in[K]} \Eb_{\theta} [\bar{T}_{s,a}] = \frac{MT}{d}.\]
Thus, there must exist an arm $a_s^*>1$ for each group $s$ such that $\Eb_{\theta}[\bar{T}_{s,a_s^*}]\leq \frac{MT}{d(K-1)}$.

Define a new instance $\tilde{\theta} = \{\tilde{\theta}_a\}_{a\in[K]}$, where
\[\tilde{\theta}_{a} = \theta_a + \sum_{s=1}^d 2\Delta e_{a_s^*} \lv{\{a_s^* = a\}}, \quad \forall a\in[K].\]
Therefore, while arm $1$ being optimal for all clients under instance $\theta$, the optimal arm of clients in group $G_s$ is arm $a_s^*$ under instance $\tilde{\theta}$.

Before we proceed, we introduce the following notations. 
Let $\mathcal{H}_t:=\{a_{i,\tau},y_{i,\tau}\}_{i\in[M],\tau \in[t]}$,
$\mathcal{H}_{s,t}:=\{a_{i,\tau},y_{i,\tau}\}_{i\in G_s,\tau \in[t]}$,
and
$\mathcal{G}_t:=\{a_{i,t},y_{i,t}\}_{i\in[M]}$,
$\mathcal{G}_{s,t}:=\{a_{i,t},y_{i,t}\}_{i\in G_s}$.

Then, we prove a useful lemma. 
\begin{lemma}\label{lemma:kl}
 Denote $D(P\|Q):=\int \log \frac{dP}{dQ}dP$ as the KL divergence between two distributions $P$ and $Q$, and $\Nc(\mu, \sigma^2)$ as a Gaussian distribution with mean $\mu$ and variance $\sigma^2$. Then,
\[D(\Pb_{\theta}(\Hc_{s,T}) \| \Pb_{\tilde{\theta}}(\Hc_{s,T}))= \Eb_{\theta}[\bar{T}_{s,a_s^*}]\cdot D(\Nc(0,1)\|\Nc(2\Delta,1)).\]
\end{lemma}
\begin{proof}
Based on the definition of KL divergence, we have
\begin{align}
    &D(\Pb_{\theta}(\Hc_{s,T}) \| \Pb_{\tilde{\theta}}(\Hc_{s,T}))\nonumber\\
    & =
    \Eb_{\theta}\left[\log\frac{\Pb_{\theta}(\Hc_{s,T}) }{ \Pb_{\tilde{\theta}}(\Hc_{s,T})}\right]\label{eqn:kl1}\\
    & =
    \Eb_{\theta}\left[\log\frac{\int_{i\notin G_s}\Pb_{\theta}(\Hc_{T}) }{ \int_{i\notin G_s}\Pb_{\tilde{\theta}}(\Hc_{T})}\right]\label{eqn:kl2}\\
     & =
    \Eb_{\theta}\left[\log\frac{\int_{i\notin G_s}\Pb_{\theta}(\Hc_{T-1})\Pb_{\theta}(\Gc_{T}|\Hc_{T-1}) }{ \int_{i\notin G_s}\Pb_{\tilde{\theta}}(\Hc_{T-1})\Pb_{\tilde{\theta}}(\Gc_{T}|\Hc_{T-1})}\right]\label{eqn:kl3}\\
    &= \Eb_{\theta}\left[\log\frac{\int_{i\notin G_s}\Pb_{\theta}(\Hc_{T-1}) \prod_{r}\Pb_{\theta}(\Gc_{r,T}|\Hc_{r,T-1}) }{ \int_{i\notin G_s}\Pb_{\tilde{\theta}}(\Hc_{T-1})\prod_{r}\Pb_{\tilde{\theta}}(\Gc_{s,T}|\Hc_{r,T-1})}\right]\label{eqn:kl4}\\
     &= \Eb_{\theta}\left[\log\frac{\int_{i\notin G_s}\Pb_{\theta}(\Hc_{T-1}) \prod_{r\neq s}\Pb_{\theta}(\Gc_{r,T}|\Hc_{r,T-1}) }{ \int_{i\notin G_s}\Pb_{\tilde{\theta}}(\Hc_{T-1})\prod_{r\neq s}\Pb_{\tilde{\theta}}(\Gc_{r,T}|\Hc_{r,T-1})} +\log \frac{\Pb_{\tilde{\theta}}(\Gc_{s,T}|\Hc_{s,T-1})}{\Pb_{\tilde{\theta}}(\Gc_{s,T}|\Hc_{s,T-1})} \right]\label{eqn:kl5}\\
 & =
    \Eb_{\theta}\left[\log\frac{\int_{i\notin G_s}\Pb_{\theta}(\Hc_{T-1}) }{ \int_{i\notin G_s}\Pb_{\tilde{\theta}}(\Hc_{T-1})}\right]+ \Eb_{\theta}\left[\sum_{i\in G_s}\lv{\{a_{i,T} = a_s^*\}}D(\Nc(0,1)\|\Nc(2\Delta,1))\right],\label{eqn:kl6}
\end{align}
where in (\ref{eqn:kl1}), we integrate over all variables that are associated with clients outside group $s$ to obtain the marginal distributions $\Pb_{\theta}(\Hc_{s,T})$ and $\Pb_{\tilde{\theta}}(\Hc_{s,T})$, respectively. Based on the assumption that $\pi$ is collinearly-independent, we can decompose the conditional distributions $\Pb_{\tilde{\theta}}(\Gc_{T}|\Hc_{T-1})$ and $\Pb_{{\theta}}(\Gc_{T}|\Hc_{T-1})$ into products of conditional distributions where only variables associated with clients in each group are dependent. Noticing that within group $s$, $\Pb_{\theta}$ and $\Pb_{\tilde{\theta}}$ only differs over arm $a_s^*$, based on which (\ref{eqn:kl6}) is obtained. 

Express the first term in (\ref{eqn:kl6}) recursively, we eventually have
\begin{align}
   D(\Pb_{\theta}(\Hc_{s,T}) \| \Pb_{\tilde{\theta}}(\Hc_{s,T}))&=\sum_{t\in[T]}\Eb_{\theta}\left[\sum_{i\in G_s}\lv{\{a_{i,t} = a_s^*\}}D(\Nc(0,1)\|\Nc(2\Delta,1))\right]\\
    &=\Eb_{\theta}[\bar{T}_{s,a_s^*}]\cdot D(\Nc(0,1)\|\Nc(2\Delta,1)).
\end{align}

\end{proof}

Now we are ready to bound the regrets under policy $\pi$ and two instances $\theta$ and $\tilde{\theta}$. Since the selection of any sub-optimal arm will incur a regret of $\Delta$, we have
\begin{align}
    R(T;\pi,\theta) + R(T;\pi,\tilde{\theta})&\geq
     \Delta\frac{MT}{2d}\sum_{s=1}^d\left(\Pb_{\theta}\left[\bar{T}_{s,1}\leq\frac{MT}{2d}\right] + \Pb_{\tilde{\theta}}\left[\bar{T}_{s,1}>\frac{MT}{2d}\right]\right)\\
     & \geq \Delta\frac{MT}{4d}\sum_{s=1}^d\exp(-D(\Pb_{\theta}(\Hc_{s,T}) \| \Pb_{\tilde{\theta}}(\Hc_{s,T})))\label{eqn: pinsker}\\
     & = \Delta\frac{MT}{4d}\sum_{s=1}^d\exp\left(-\Eb_{\theta}[\bar{T}_{s,a_s^*}]\cdot D(\Nc(0,1)||\Nc(2\Delta,1))\right)\label{eqn: D(P||Q)=D(N||N)}\\
     & \geq \Delta\frac{MT}{4d}\sum_{s=1}^d\exp\left(\frac{MT}{d(K-1)}\cdot 2\Delta^2\right)\label{eqn: ET<MT/dK}\\
     & = \frac{MT\Delta}{4}\exp\left(\frac{2MT\Delta^2}{d(K-1)}\right),
\end{align}
where (\ref{eqn: pinsker}) follows from the high probability Pinsker inequality, (\ref{eqn: D(P||Q)=D(N||N)}) is due to Lemma~\ref{lemma:kl}, and (\ref{eqn: ET<MT/dK}) follows from the definition of $\bar{T}_{s,a_s^*}$.

Let $\Delta := \sqrt{\frac{d(K-1)}{4MT}}$. Then, we have 
\[\max\left( R(T;\pi,\theta) + R(T;\pi,\tilde{\theta})\right)\geq \frac{\sqrt{d(K-1)MT}}{8\sqrt{e}} = \Omega(\sqrt{dKMT}).\]

\section{\alg for Shared Parameter Case}\label{appx:shared}
\subsection{Algorithm Details}
We slightly modify \alg for the shared parameter case where $\theta_a = \theta, \forall a\in[K]$. {Such minor tweaks, as opposed to a full-blown re-design, lead to a unified algorithmic framework that can be flexibly applied to both disjoint and shared parameter cases.} Similar to the disjoint parameter case, at the client side, each client $i$ treats $\{\theta_a\}_a$ separately, i.e., it does not aggregate rewards collected by pulling different arms to estimate the shared parameter $\theta$; Rather, at phase $p$, it will output $| \Ac_i^p |$ different estimates $\{\hat{\theta}_{i,a}\}_{a\in\Ac_i^p}$ and send them to the central server for aggregation. 

The major difference occurs in the server side subroutine. Once different estimates $\{\hat{\theta}_{i,a}\}_{i,a}$ are received, the central server would aggregate them to obtain a global estimate of the shared parameter $\theta$. Specifically,  
the global aggregation step at the server side in (\ref{eqn: theta_hat}) can be changed as follows: after obtaining
\begin{align}
V_a^{p+1}&\gets\bigg(\sum_{i\in \Rc_a^p}f_{i,a}^p\frac{\hat{\theta}_{i,a}^p(\hat{\theta}^p_{i,a})^\TT}{\|\hat{\theta}_{i,a}^p\|^2}\bigg)^{\dagger}, 
\end{align}
the server would calculate
\begin{align}
     V^{p+1} &\gets \bigg( \sum_{a\in\Ac^p}\left(V_a^{p+1}\right)^\dagger\bigg)^\dagger,\\
     \hat{\theta}^{p+1}&\gets V^{p+1}\bigg(\sum_{i\in[M]}\sum_{a\in\Ac_i^p}f_{i,a}^{p}\hat{\theta}_{i,a}^{p}\bigg).
\end{align}

Correspondingly, the multi-client G-optimal design at phase $p$ can be formulated as 
\begin{equation}\label{eqn: G_opt_shared}
\begin{aligned}
&\text{minimize\ \  } G(\pi) = \sum_{i = 1}^M\max_{a\in \Ac_i^p}  e_{i,a}^\TT\bigg(\sum_{j\in [M]}\sum_{a\in\Ac_j^p} \pi_{j,a}^pe_{j,a}e_{j,a}^\TT\bigg)^{\dagger}e_{i,a}\quad \text{s.t. } \pi^p\in\Cc^p.
\end{aligned}
\end{equation}
Compared with the multi-client G-optimal design for the disjoint parameter case in Eqn.~\eqref{eqn: G_opt}, the only difference is that (\ref{eqn: G_opt_shared}) now has a double summation over both arms and clients inside $(\cdot)^\dagger$ in the objective function, while (\ref{eqn: G_opt}) only contains a summation over the clients. This is consistent with the different constructions of the ``potential matrices'' for those two cases.

Once the estimated shared parameter and potential matrix $(\hat{\theta}^{p},V_a^{p} )$ is broadcast to the clients, they will use it to obtain estimates $\hat{r}_{i,a}^p$ and $u_{i,a}^p$, as in (\ref{eqn: confidence_interval}) under \texttt{Fed-PE}. The process then continues.

\subsection{Theoretical Analysis}
We present the regret upper bound and analysis for the shared parameter case in this subsection.

\begin{theorem}[Complete version of Theorem~\ref{thm: BoundSharedParameter}]\label{thm: complete_BoundSharedParameter}
Consider time horizon $T$ that consists of $H$ phases with $f^p= cn^p$, where $c$ and $n>1$ are fixed integers, and $n^p$ denotes the $p$th power of $n$.
Let 
\begin{equation}\label{eqn: alpha_shared}
\alpha=\min(\sqrt{2\log (H/\delta)+d\log(ke)},\sqrt{2\log(2MKH/\delta)}),
\end{equation}
where $k>1$ is a number satisfying $kd\geq 2\log(H/\delta)+d\log(ke)$.
Then, with probability at least $1-\delta$, the regret of the adapted \alg for the shared parameter case is upper bounded as $$R(T)\leq 4\alpha\frac{L}{\ell}\sqrt{dM}\left(\frac{\sqrt{n^2-n}}{\sqrt{n}-1}\sqrt{T} + \frac{K}{\sqrt{cn} - \sqrt{c}}\right).$$ 
In particular, if {$K=O(\sqrt{T})$}, the regret upper bound scales in
$$O\left(\sqrt{dMT(\log({(\log T)}/{\delta})+\min\{d,\log MK\})}\right),$$ and the communication cost scales in {$O((KdM + d^2M)\log T)$}.
\end{theorem}

Following similar analysis as in the proof of Lemma \ref{lemma: General_equivalence} by replacing $v_{j,a}$ with $e_{j,a}$, we can show that the multi-client G-optimal design in (\ref{eqn: G_opt_shared}) is equivalent to the following determinant optimization problem:
\begin{equation}\label{eqn: F_opt_shared}
\begin{aligned}
&\text{maximize\ \  } F(\pi) = \log \Det\bigg(\sum_{j\in [M]} \sum_{a\in\Ac_j^p}\pi_{j,a}^p e_{j,a}e_{j,a}^\TT\bigg)\quad \text{s.t. }  \pi^p\in\Cc^p.
\end{aligned}
\end{equation}
and the optimal solution $(\pi^p)^*$ satisfies $G((\pi^p)^*) = \rank(\{e_{i,a}\}_{i\in[M],a\in\Ac_i})\leq d$. 

Note that compared with the optimal solution $\pi^*$ in Lemma~\ref{lemma: General_equivalence} where $G(\pi^p) = \rank(\{v_{i,a}\}_{i\in[M],a\in\Ac_i})\leq Kd$, the upper bound on $G(\pi)$ has been reduced by a factor of $K$, due to the dimension difference between $e_{i,a}$ and $v_{i,a}$. Besides, we also note that the block diagonal structure inside $\Det(\cdot)$ in (\ref{eqn: general_F_opt}) does not exist in (\ref{eqn: F_opt_shared}) any more, which implies that the \bca algorithm in Algorithm~\ref{alg: BCM} cannot be applied for this case. Designing alternative efficient optimization algorithms to solve (\ref{eqn: F_opt_shared}) is one of our future steps.

With a little abuse of notation, in this section, we define
\begin{align}
\alpha_2&:=\sqrt{2\log (H/\delta)+d\log(ke)}\\
\sigma_{i,a}^p&:=\|x_{i,a}\|_{V^p}/\ell\\
    X_p&:=\bigg\|\sum_{i\in[M]}\sum_{a\in\Ac_i}\frac{e_{i,a}}{\|x_{i,a}\|}\ell\xi_{i,a}^{p-1}\bigg\|_{V^p}^2\\
    a_{(i,a),(j,b)}&:=\sqrt{f_{i,a}^{p-1}}e_{i,a}^\TT\bigg(\sum_{k\in[M]}\sum_{a\in\Ac_k^{p-1}}f_{k,a}^{p-1}e_{k,a}e_{k,a}^\TT\bigg)^\dagger e_{j,b}\sqrt{f_{j,b}^{p-1}}\\
    \xi &:= \left(\frac{\ell\xi_{i,a}^{p-1}}{\|x_{i,a}\|\sqrt{f_{i,a}^{p-1}}}\right)_{i\in[M],a\in\Ac_i^{p-1}}\in\Rb^{\sum_{i\in[M]}|\Ac_i^{p-1}|}\\
    A &:= (a_{(i,a),(j,b)})\in\Rb^{\left(\sum_{i\in[M]}|\Ac_i^{p-1}|\right)\times \left(\sum_{i\in[M]}|\Ac_i^{p-1}|\right)}.
\end{align}
It can be verified that $X_p = \|A\xi\|^2 $, $A^2=A$, and $\trace(A)\leq d$. 

Then, following the same definition of ``bad'' event $\Ec(\alpha)$ by using the new set of variables, we can show that Lemma~\ref{lemma: BndOneErr} still holds by slightly modifying the analysis in Appendix~\ref{sec: ProofAlgorithm1}. For the probability of the bad event, we can show that $\Pb[\Ec(\alpha_1)]\leq \delta$ similarly as in the proof of Lemma~\ref{lemma: BndAllErr}, while $\Pb[\Ec(\alpha_2)]\leq \delta$ can be derived using the new definitions of $X_p$, $A$, $\xi$. In particular, we can show the following lemma. 
\begin{lemma}
Under the adapted \alg algorithm for the shared parameter case, for any $p\in[H]$, we have $\Pb[X_p\geq \sqrt{2\log(H/\delta) + d\log(ke)}|\Fc_{p-1}]\leq \delta/H$.
\end{lemma}
$\Pb[\Ec(\alpha_2)]\leq \delta$ then follows by applying a union bound over all phases $p\in[H]$.

Then, in order to bound the regret when the good event happens, we follow the same steps as in the proof of Lemma~\ref{lemma: PhaseRegret}, except that we utilize the following lemma to remove the factor of $\sqrt{K}$ in the upper bound. 

\begin{lemma}\label{lemma:variance_bound_shared}
Under the adapted \alg algorithm for the shared parameter case, for any $p\in[H]$, we have
$\sum_{i\in[M]}\max_{a\in \Ac_i^p}\frac{x^\TT_{i,a}}{\|x_{i,a}\|_2} V^p \frac{x_{i,a}}{\|x_{i,a}\|_2}\leq   \frac{d}{f^{p-1}}$. 
\end{lemma}

Lemma~\ref{lemma:variance_bound_shared} is a direct consequence of the multi-agent G-optimal design for the shared parameter case, as elaborated earlier. Thus, when we put all pieces together, the regret is reduced by a factor of $\sqrt{K}$ compared with Theorem~\ref{thm: BoundAlgorithm1}, while the other factors remain the same.

\section{Experiment Details and Additional Results}\label{appx:experiment}

\textbf{Experiment Environment.} The experiments in both Section~\ref{sec:simulation} and this section are carried out using a MacBook Pro with a 2.3GHz Quad-Core Intel Core i5 CPU and 8GB 2133 MHz LPDDR3 memory. The programming language is Python3 with packages Numpy, Scipy, and Pandas.
All experiments are computationally light -- the running time ranges from less than a minute to slightly more than 1 hour.

\textbf{Regret versus System Parameters.}  In addition to the results reported in the main paper, we also plot the regret with varying $M,K$ and $d$ under \texttt{Enhanced Fed-PE} in Fig.~\ref{fig:syn_regret}. The experimental setup remains the same as the synthetic one in Section~\ref{sec:simulation}, {except that $T$ is set as $2^{16}$}. We see that the per-client regret is proportional to the number of arms and the context dimensions, and inversely proportional to the number of clients, which is reasonable.

\begin{figure}[htbp]	
	\centering  
	\subfigure[]{\includegraphics[width=0.32\textwidth]{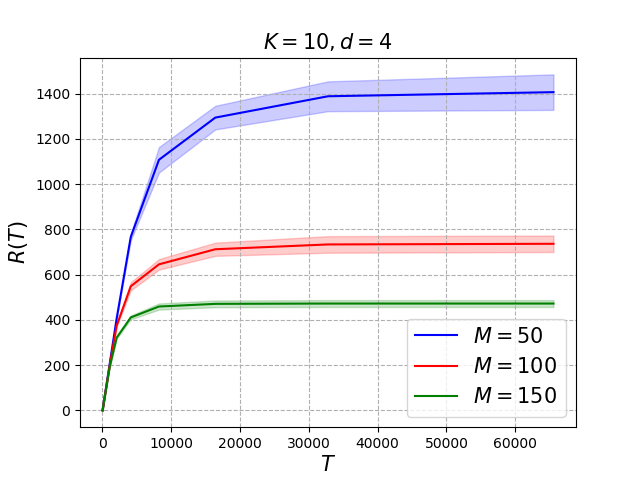}}
	\subfigure[]{\includegraphics[width=0.32\textwidth]{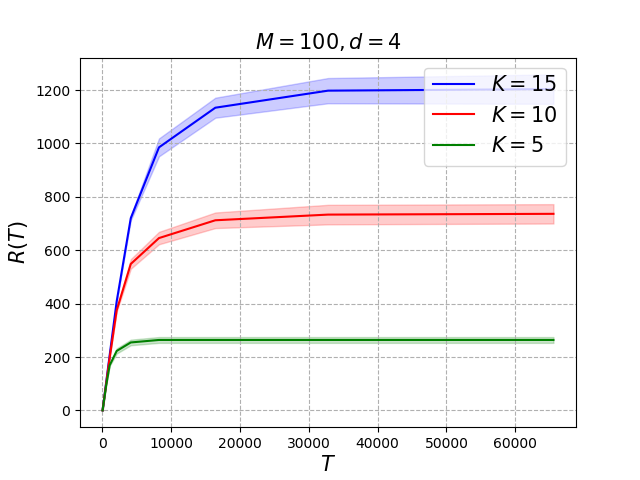}}
	\subfigure[]{\includegraphics[width=0.32\textwidth]{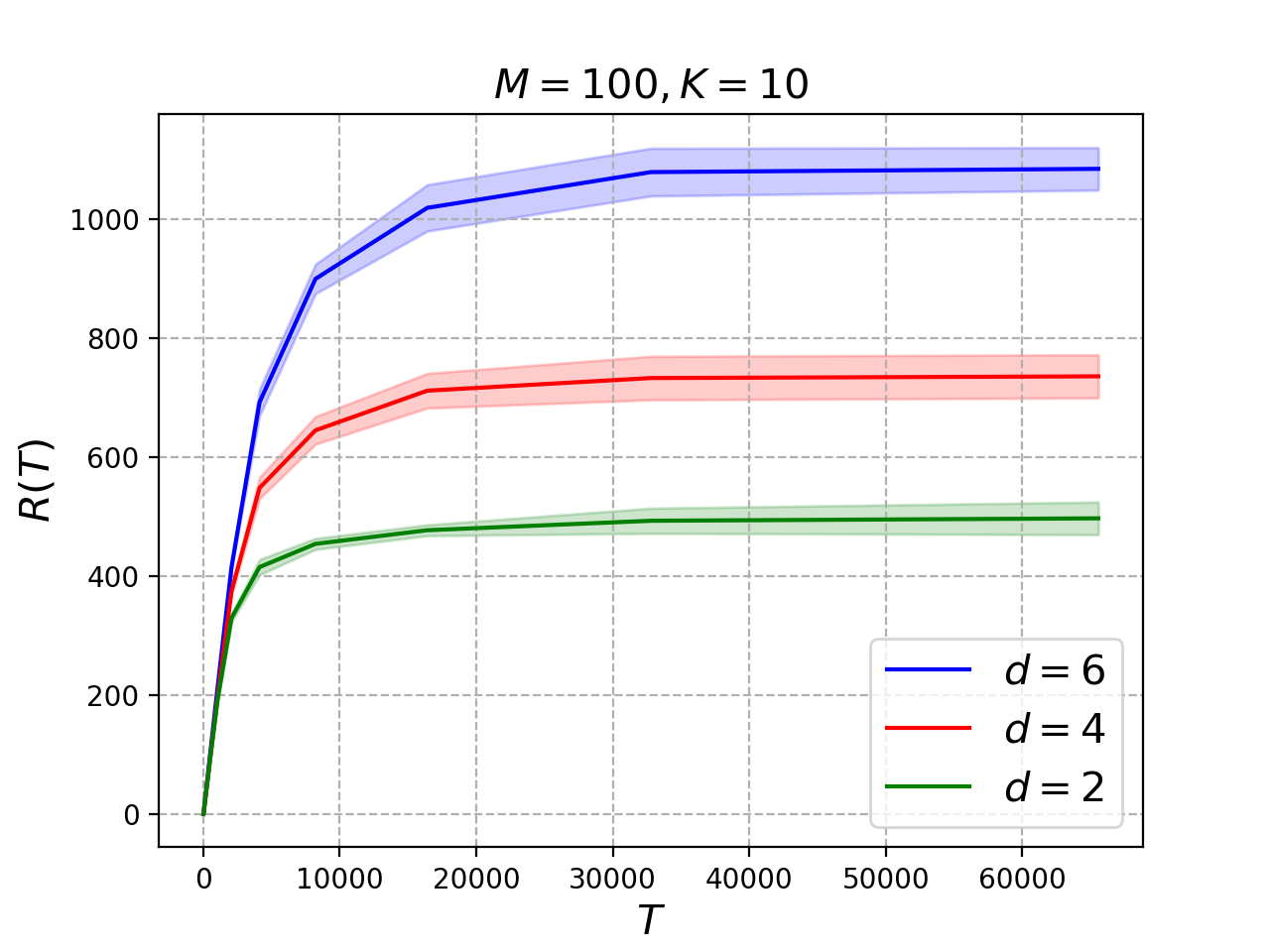}}
	\caption{\small Per-client pseudo-regret over $T$, with varying $M, K$, and $d$. Shaded area indicates the standard deviation.}\label{fig:syn_regret}
\end{figure}

\textbf{Sparsity Level.} A sparse solution $\pi$ of the optimization problem~(\ref{eqn: G_opt}) has many zero entries, which is desirable because it means each client only needs to explore a few arms in one phase even if the active arm set is very large.
Define the \texttt{Sparsity Level} of $\pi$ as:
\[\text{\texttt{Sparsity Level}} = \frac{\#\text{the number of non-zero entries of } \pi}{M}.\]
The smaller the \texttt{Sparsity Level} is, the sparser the solution is. Since there are $M$ linear equality constraints, the number of non-zero entries in $\pi$ should be at least $M$. Thus, the minimum value of \texttt{Sparsity Level} is 1. Moreover, if the elimination procedure works well, we expect to eventually observe only $M$ non-zero entries in $\pi$, and constant regret under a specific realization of the bandits problem. Therefore, \texttt{Sparsity Level} also characterizes how many arms that one client pulls on average within a phase. {Besides, a smaller \texttt{Sparsity Level} indicates} less communication cost per client.

\begin{figure}[htbp]	
	\centering  
	\subfigure[]{\includegraphics[width=0.32\textwidth]{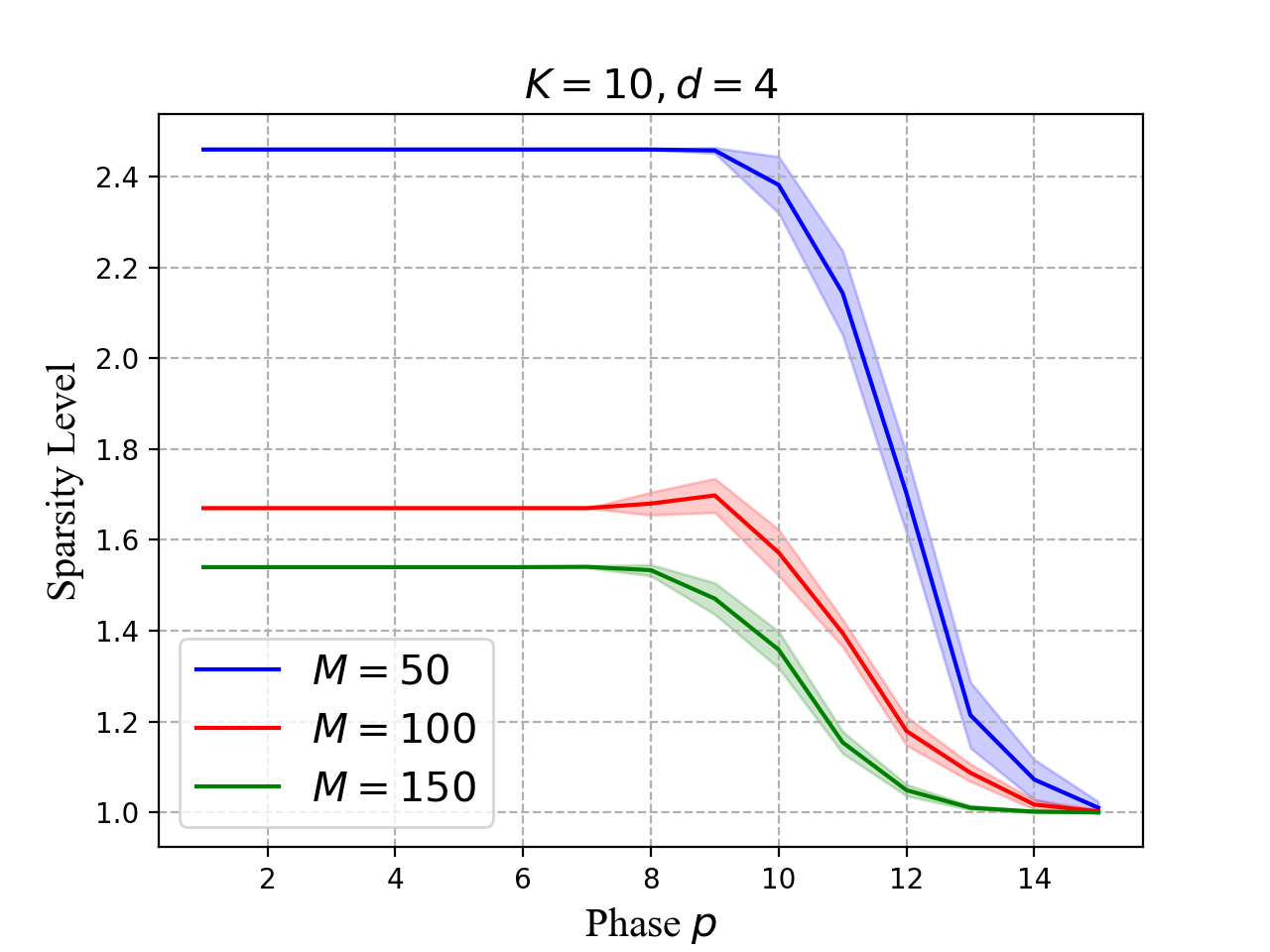}}
	\subfigure[]{\includegraphics[width=0.32\textwidth]{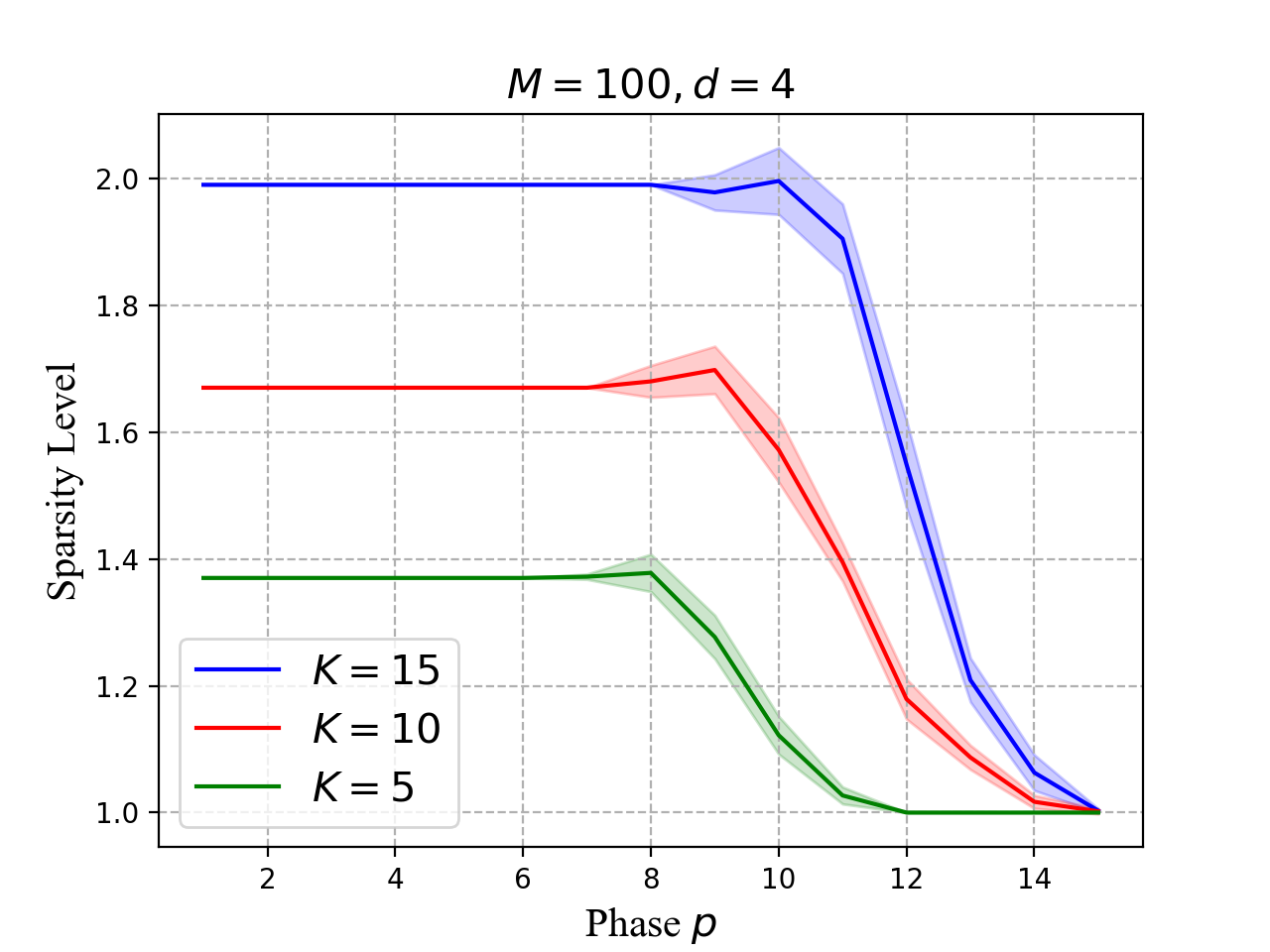}}
	\subfigure[]{\includegraphics[width=0.32\textwidth]{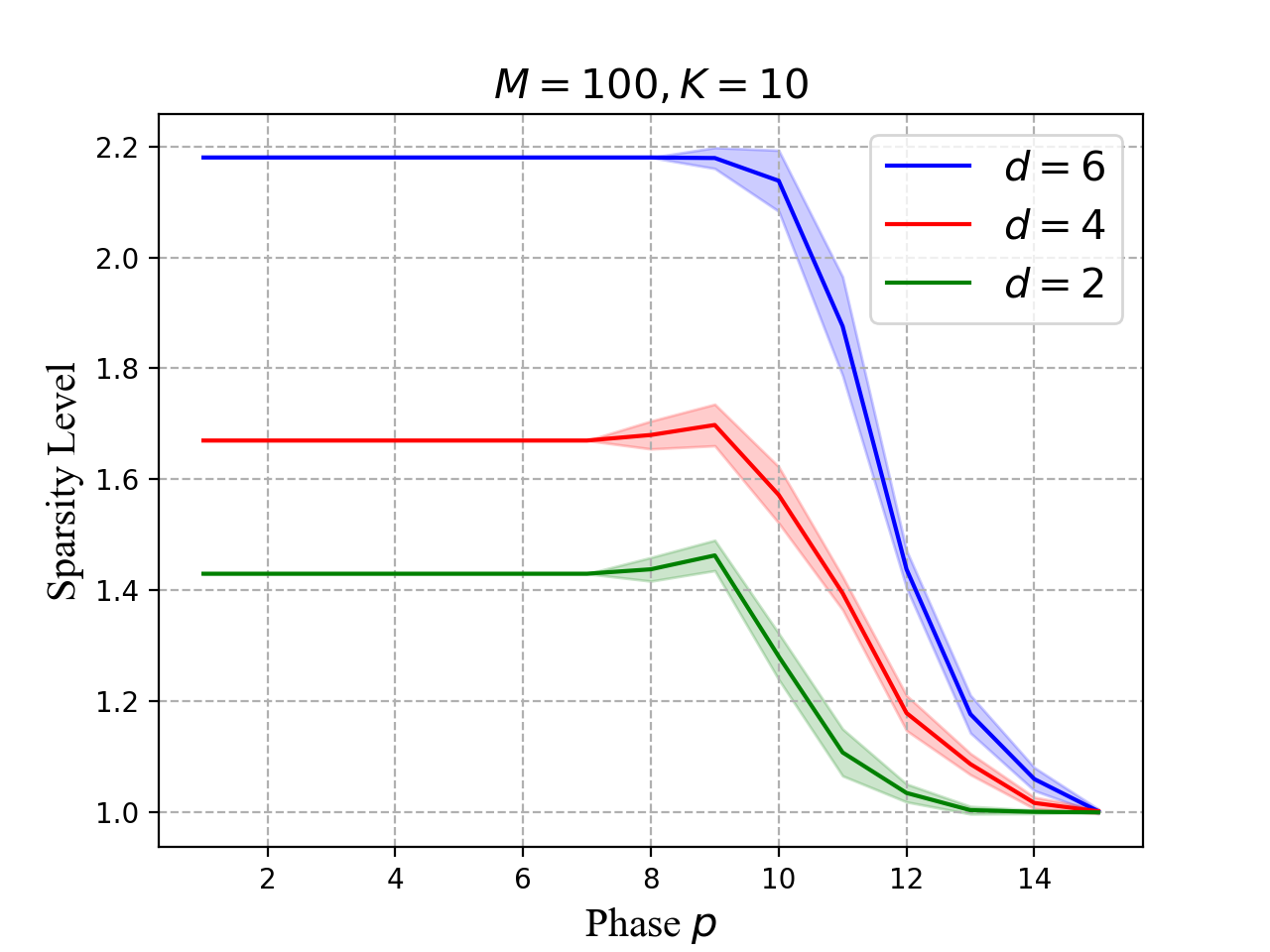}}
	\caption{\small Sparsity level over phase index $p$, with varying $M,K$ and $d$. Shaded area indicates the standard deviation.}\label{fig:syn_sparse}
\end{figure}

In Figure~\ref{fig:syn_sparse}, we plot the \texttt{Sparsity Level} as a function of the phase index $p$. The result suggests that our algorithm assigns each client only approximately 2 arms (on average) to explore. 
We also see that the \texttt{Sparsity Level} and the communication per client are proportional to $K$ and $d$, and inversely proportional to $M$.

\textbf{Number of Iterations.} To measure the efficiency of the \texttt{BCA} algorithm, we define the \texttt{Number of Iterations} as the minimum number $k$ such that $G(\pi^{(k)})\leq d_a^p + \epsilon$, where the notations remain the same as in the Appendix~\ref{appx:CoordinateAscent}. In the experiment, we choose $\epsilon = 0.1$.

In Figure~\ref{fig:syn_iteration}, we observe that the \texttt{Number of Iterations} is inversely proportional to $M$, and proportional to $K$ and $d$, which suggests that increasing the number of clients will not affect the efficiency significantly, although we should bear in mind that one iteration includes $M$ block updates.

\begin{figure}[htbp]	
	\centering  
	\subfigure[]{\includegraphics[width=0.32\textwidth]{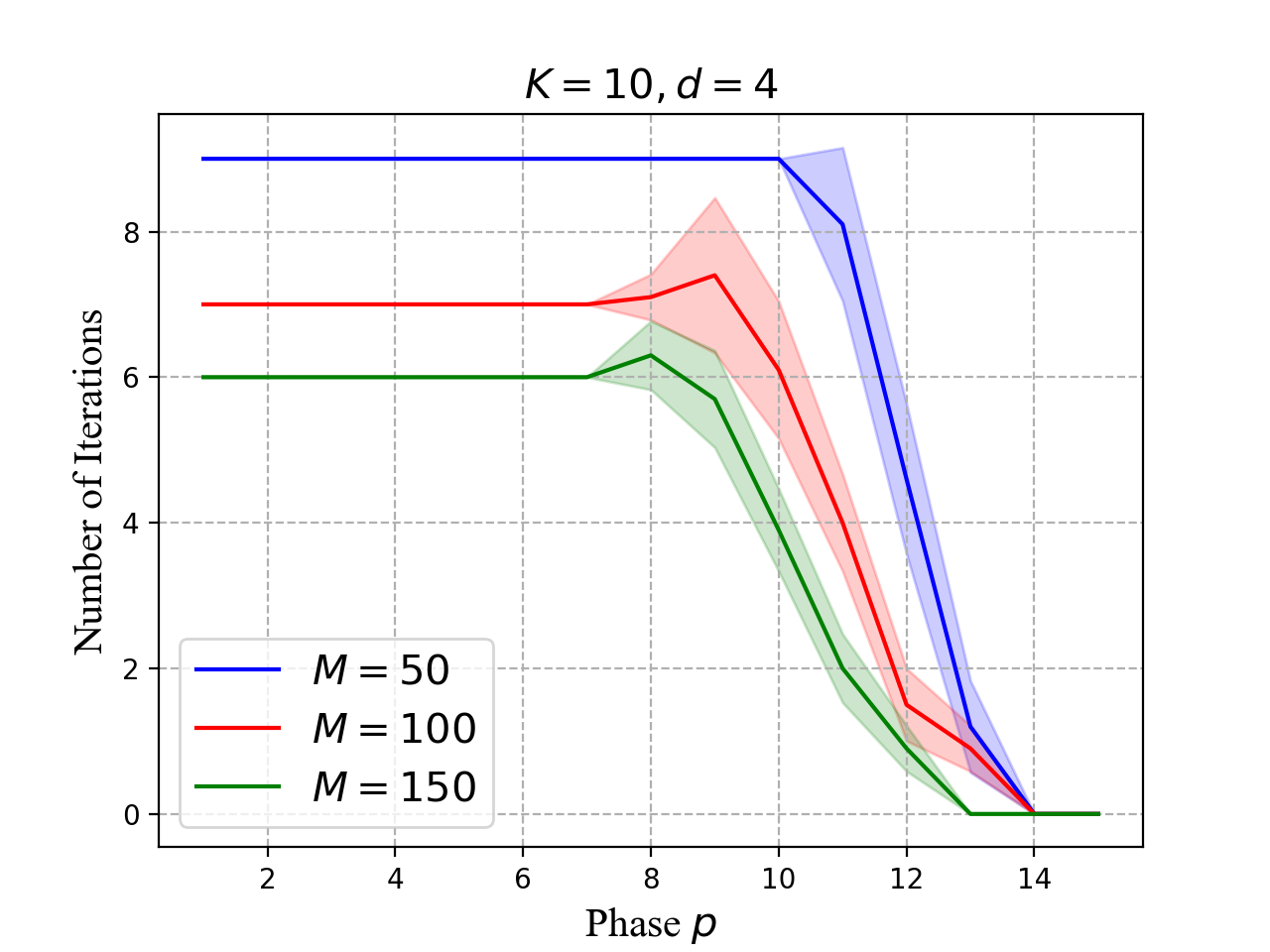}}
	\subfigure[]{\includegraphics[width=0.32\textwidth]{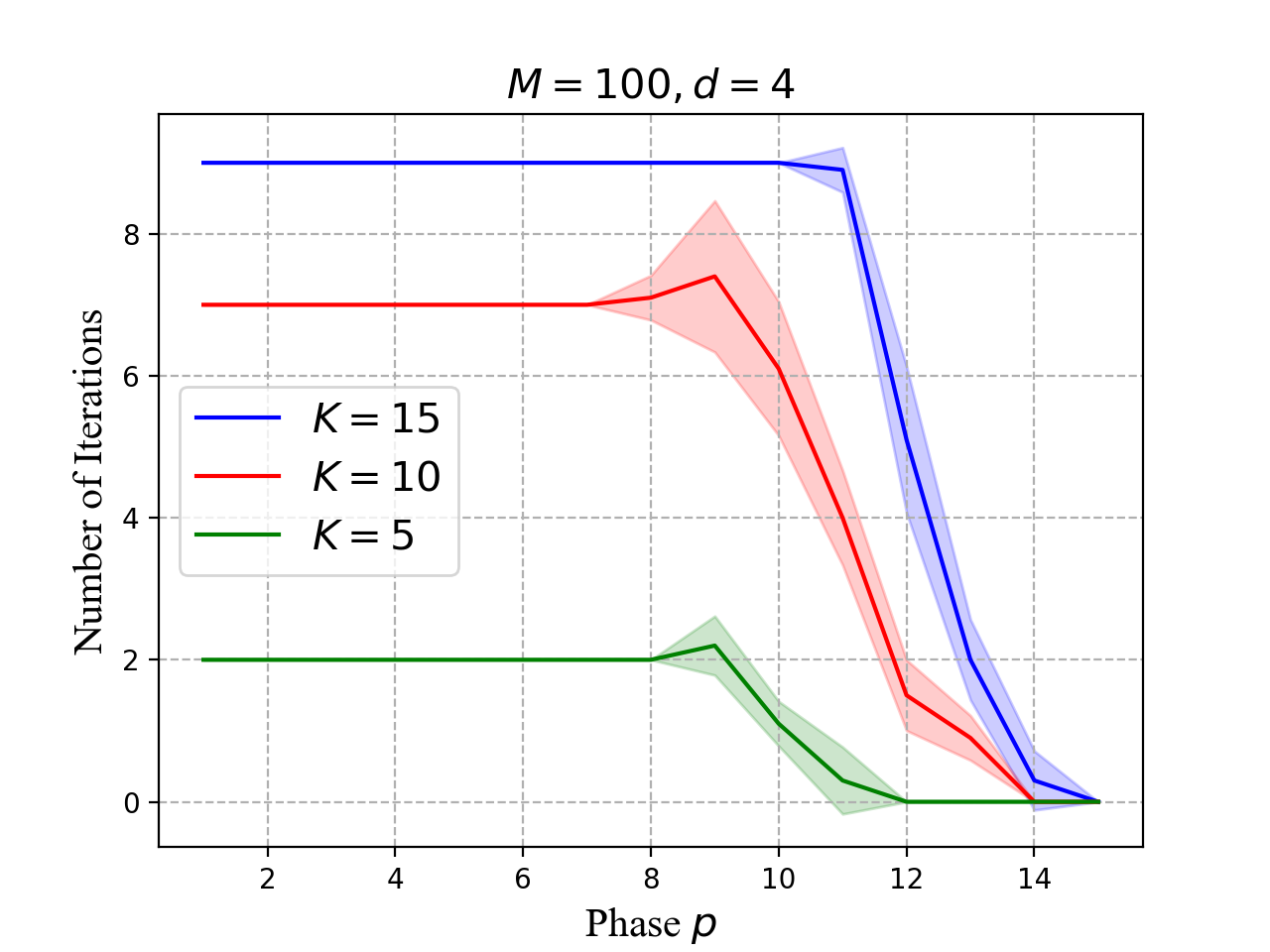}}
	\subfigure[]{\includegraphics[width=0.32\textwidth]{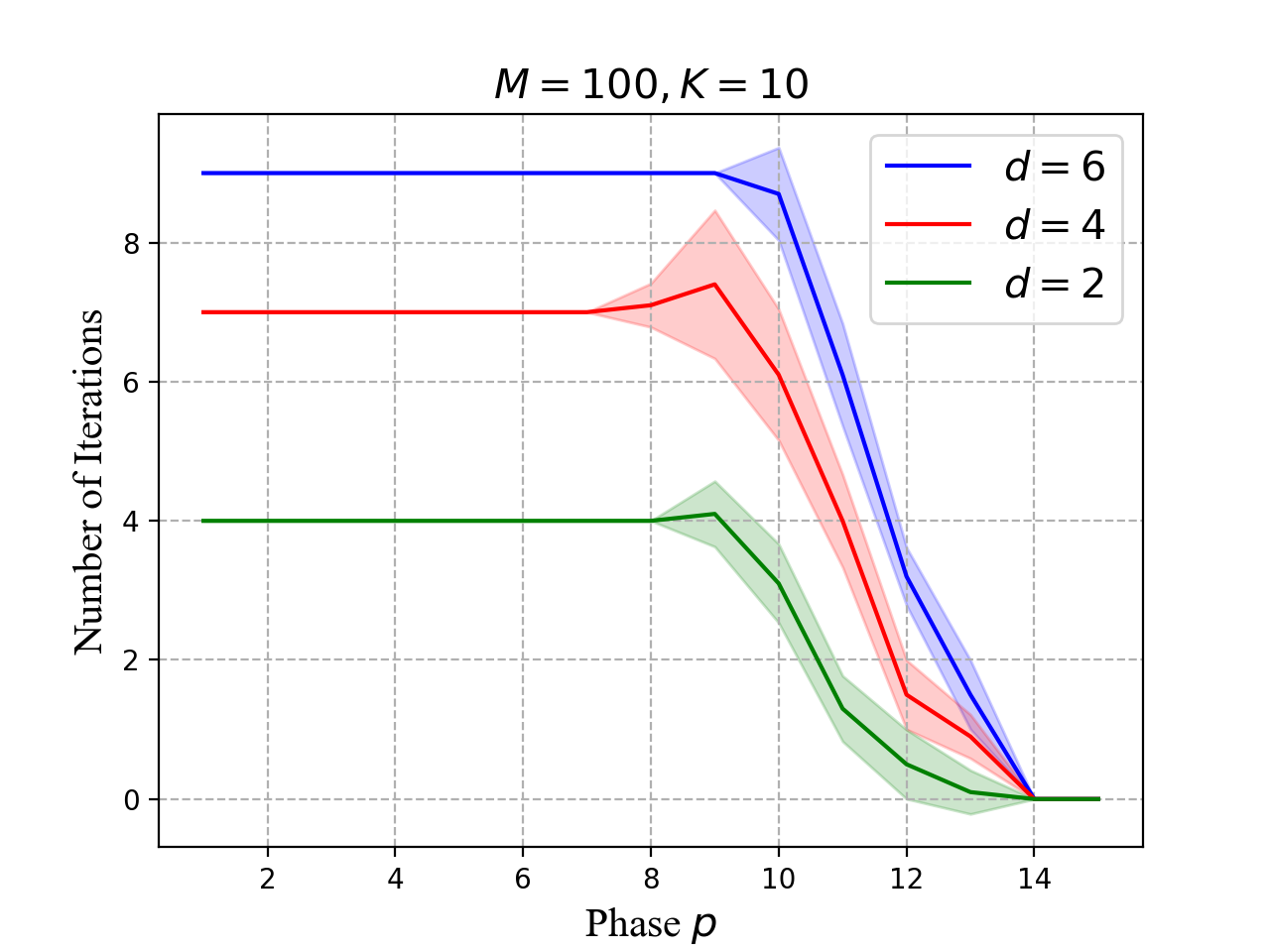}}
	\caption{\small Number of iterations over phase index $p$, with varying $M, K$, and $d$. Shaded area indicates the standard deviation.}\label{fig:syn_iteration}
\end{figure}

\textbf{Different Phase Length Selection.} In Figure~\ref{fig:compare_phase_length}, we plot the regret performances of the three selections of $f^p$ introduced in Appendix~\ref{subsec: different_phase_length}. Specifically, we construct the ``greedy'' phase length the same as in Appendix~\ref{subsec: different_phase_length}. The ``exponential'' selection is $f^p = 2^p$. For the ``uniform'' selection, we set $f^p=100$ rather than $10$ for time-saving. 
Note that the greedy selection leads to $\log_2\log_2 T -1 = 3$ phases, while the exponential selection consists of 15 phases, and uniform has 655 phases. As we observe, the uniform selection results in the lowest regret, while the exponential selection has slightly worse regret performance. The regret under the greedy selection is much worse than those under the other selections. The results suggest that exponential selection has the best tradeoff between regret and communication costs, corroborating the theoretical results in Table~\ref{table:enhanced}.

\begin{figure}[htbp]	
	\centering  
	{\includegraphics[width=0.32\textwidth]{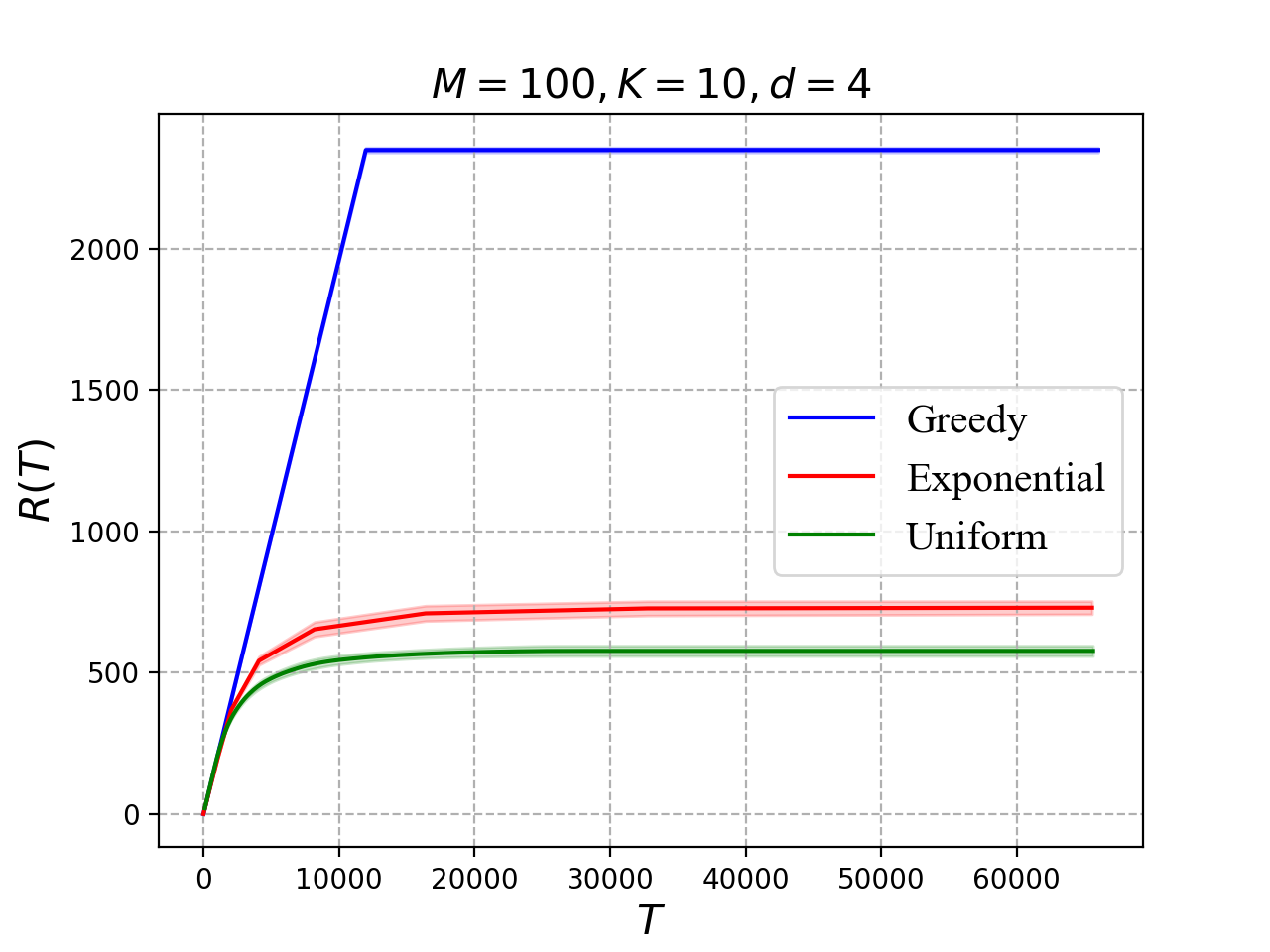}}
	\caption{\small Regret with different selection of $f^p$. Shaded area indicates the standard deviation.}\label{fig:compare_phase_length}
\end{figure}

\bibliography{BanditYang,MPMAB}
\bibliographystyle{apalike}

\end{document}